\documentclass{article}

\usepackage[main,final]{neurips_2026}

\usepackage[utf8]{inputenc}%
\usepackage[T1]{fontenc}%
\usepackage{hyperref}%
\usepackage{url}%
\usepackage{booktabs}%
\usepackage{array}%
\usepackage{etoc}%
\usepackage{amsfonts, amsmath, amssymb, amsthm}%
\usepackage{nicefrac}%
\usepackage{microtype}%
\usepackage{xcolor}%
\usepackage{caption}
\usepackage{subcaption}
\usepackage{graphicx}

\usepackage{algorithm}
\usepackage{algorithmic}

\newtheorem{theorem}{Theorem}[section]
\newtheorem{proposition}[theorem]{Proposition}
\newtheorem{corollary}[theorem]{Corollary}
\newtheorem{lemma}[theorem]{Lemma}

\newtheorem{remark}[theorem]{Remark}

\newcommand{\mist}{\textsc{Mist}}
\newcommand{\vfdt}{\textsc{Vfdt}}
\newcommand{\efdt}{\textsc{Efdt}}
\newcommand{\slda}{\textsc{Slda}}
\newcommand{\sqda}{\textsc{Sqda}}
\newcommand{\arf}{\textsc{Arf}}
\newcommand{\Prob}{\mathbb{P}}

\title{
    \mist{}: Reliable Streaming Decision Trees for Online Class-Incremental Learning via McDiarmid Bound
}

\author{
Phu-Hoa Pham$^{\star,1,2,4}$ \\
\texttt{23122030@student.hcmus.edu.vn}
\And
Chi-Nguyen Tran$^{\star,1,2,4}$ \\
\texttt{23122044@student.hcmus.edu.vn}
\AND
Phu-Quy Nguyen-Lam$^{1,2,4}$ \\
\texttt{23122048@student.hcmus.edu.vn}
\And
Dao Sy Duy Minh$^{1,2,4}$ \\
\texttt{23122041@student.hcmus.edu.vn}
\AND
Huynh Trung Kiet$^{1,2,4}$ \\
\texttt{23122039@student.hcmus.edu.vn}
\And
Long Tran-Thanh$^{\dagger,3}$ \\
\texttt{long.tran-thanh@warwick.ac.uk}
}

\begin{document}
\maketitle
\begin{center}
$^1$ Faculty of Information Technology, University of Science, Ho Chi Minh City, Vietnam \\
$^2$ Vietnam National University, Ho Chi Minh City, Vietnam \\
$^3$ Department of Computer Science, University of Warwick, Coventry, United Kingdom \\
$^4$ Compact AI
\end{center}
\begingroup
\renewcommand{\thefootnote}{}
\footnotetext{$^\star$\,Equal contribution.\quad
  $^\dagger$\,Corresponding author: \texttt{long.tran-thanh@warwick.ac.uk}}
\endgroup

\begin{abstract}
Streaming decision trees are natural candidates for open-world continual learning, as they perform local updates and maintain explicit decision boundaries. Despite these, they still fail in online class-incremental learning due to two coupled miscalibrations: (i) their split criterion grows unreliable as the class count $K$ expands, and (ii) the absence of knowledge transfer at split time. The first issue is linked to the range of Information Gain, which intrinsically scales with $\log_2 K$. Consequently, any Hoeffding-style confidence radius derived from it must inevitably grow with the class count, preventing a $K$-independent radius under this Information-Gain-based calibration, taking away the potential benefits of applying streaming decision trees to continual learning. To fix this issue, we present \mist{} (McDiarmid Incremental Streaming Tree), which addresses these two issues through three integrated components: (i) a McDiarmid confidence radius for Gini splitting based on a tight, $K$-independent sensitivity bound that acts as a structural regulariser; (ii) a Bayesian inheritance protocol that projects parent statistics to child nodes via truncated-Gaussian moments, with variance reduction under the stated prior assumptions; and (iii) per-leaf KLL quantile sketches that support both continuous threshold evaluation and geometry-adaptive leaf prediction from a single data structure. On standard and stress-test tabular streams, \mist{} is competitive with global parametric methods on near-Gaussian benchmarks, although the latter achieve higher mean accuracy on the standard benchmarks. Its sketch-based variant, \mist{}-K, shows benefits on specific non-Gaussian geometries.
Code is available at \href{https://github.com/PhamPhuHoa-23/MIST-Official}{GitHub repository}.
\end{abstract}

\etocdepthtag.toc{mtmain}
\section{Introduction}

Continual learning systems that operate in an open world must absorb new
classes as they arrive, without access to historical data and without
degrading what was previously learned. In tabular settings, where large
pretrained representations are unavailable, this challenge must be resolved
directly from the raw feature stream under strict memory constraints.

Streaming decision trees are a natural candidate for this regime. Their updates
are local, as only the leaf reached by an arriving sample is modified. As such, their memory
depends on the number of leaves and the statistics retained at each leaf, and they share no global representation across classes.
For example, in the foundational VFDT framework, which laid down the basis for streaming decision trees, a decision
boundary, once committed via the Hoeffding bound, is not revised. Instead,
splits are permanent, preserving committed partitions. Leaf predictions can
still change as statistics are updated, so this alone does not prevent forgetting. This structural potential, however,
has never been fully realised. In practice, classical streaming trees collapse in the online
class-incremental learning (Online CIL) setting. The reason is more fundamental than it may
first appear, and tracing it to its root reveals two coupled failures that must be
resolved simultaneously.

\noindent
\textbf{The root cause: Premature split.}
Existing streaming trees typically
trigger splits when the empirical gain gap
exceeds a Hoeffding confidence radius derived from Information Gain~\cite{DomingosHulten2000,ManapragadaWebbSalehi2018,korycki2021,BifetGavalda2009HAT,chaouki2024}. This
radius grows as $\log_2 K / \sqrt{n}$, which comes from the structural implication of
Information Gain's range scaling with $\log_2 K$, and becomes unreliable
as $K$ grows without bound, either delaying necessary splits or triggering
prematurely, making streaming decision trees unsuitable for online CIL settings.

\noindent
\textbf{A potential solution and its cold-start split issue.}
A natural way to fix this premature split issue is to switch to Gini
impurity, whose range is bounded in $[0,1)$ regardless of $K$, making a
class-count-independent radius achievable in principle. However, identifying a tight decision radius for the Gini split gain is non-trivial.
To address this, we derive a tight bounded-difference constant for Gini split gain under
single-sample replacement and prove it is unimprovable for arbitrary $K$,
yielding a confidence radius without a multiplicative $\log_2 K$ range factor.

This calibration is conservative
by design, acting as a structural regulariser. It does, however, expose
a second point of failure: when a leaf splits, its children inherit nothing from
their parent and must rebuild predictive distributions from scratch,
producing unreliable predictions precisely when the tree adapts most
aggressively; further splits from sparse children compound the
instability.

\noindent
\textbf{The coupled failures: premature and cold-start splits.}
The two failures reinforce each other, as premature splits produce sparse
children, whose cold start inflates impurity estimates, triggering
further premature splits. To fix this, we break this cycle by addressing both
simultaneously: Conservative splitting ensures children are born from
data-rich parents, improving the use of inherited statistics.
To do so, we use Bayesian inheritance and translate those statistics into tight posterior
distributions from the first sample onward. Crucially, variance reduction
from inheritance scales with the parent sample count at split time, so
the more conservative the criterion, the stronger the inheritance
guarantee.

\noindent
\textbf{Sketching for both splitting and prediction.}
The two components above handle \emph{when} to split and \emph{how} to
initialise after splitting.
However, both would normally rely on past raw data, which is typically unavailable in the online CIL setting.
To address this, we propose a new, sketching-based approach to continuous threshold evaluation
without retaining raw data. In particular, per-leaf, label-aware quantile sketches generate
candidate thresholds from empirical class-conditional distributions, and the
same structure powers geometry-adaptive leaf prediction---parametric Gaussian
on near-Gaussian data, distribution-free otherwise. Both configurations are
formalised in Section~\ref{sec:framework}.

\noindent
\textbf{Contributions.} In this paper, we present \mist{} (McDiarmid Incremental Streaming Tree) to implement the three ideas above as follows:
\begin{itemize}
  \item \textbf{Tight McDiarmid calibration for Gini splitting.}
  We derive the tight bounded-difference constant for multi-class Gini split
  gain and prove that it is sharp for arbitrary $K$, yielding a
  confidence radius without the Information-Gain range factor
  (Corollary~\ref{cor:radius}). Candidate multiplicity still enters
  logarithmically. This calibration acts as a structural regulariser
  alongside the inheritance mechanism.
  \item \textbf{Bayesian post-split stabilisation.}
  We introduce parent-to-child projection via exact truncated-Gaussian
  conditional moments for the split feature and Naive Bayes mass rescaling for
  remaining features. Variance reduction guarantees cover both
  class-probability estimates
  and feature-parameter estimates
  (Proposition~\ref{prop:variance})
  with both
  posteriors tightening in proportion to the parent sample count at split time.
  \item \textbf{Unified label-aware sketching for structure and prediction.}
  Per-leaf KLL quantile sketches support continuous split threshold evaluation
  without raw-data retention; the same sketches power an optional
  non-parametric leaf predictor under a Naive Bayes factorisation at no
  additional storage cost. Theorem~\ref{thm:composed} provides a composed
  error bound jointly controlling statistical variance and sketch rank error.
  \item \textbf{Regime-aware empirical characterisation.}
  We identify and explain conditions under which the parametric and
  non-parametric leaf predictors are respectively preferable, supported by
  prequential accuracy curves, ablation studies, and stress-test experiments
  on synthetic streams with non-Gaussian class geometry.
\end{itemize}

\section{Related Work}
Due to space limits, this section covers only the work most relevant to ours and defers a more detailed summary of the online continual learning literature to Appendix~\ref{app:related_work}.

\noindent
\textbf{Streaming decision trees in class-incremental settings.}
\vfdt{}~\cite{DomingosHulten2000} established Hoeffding-style confidence tests for incremental splitting, with EFDT~\cite{ManapragadaWebbSalehi2018} adding continuous re-evaluation and HAT~\cite{BifetGavalda2009HAT} incorporating adaptive drift detection via ADWIN~\cite{BifetGavalda2007ADWIN}. These methods treat distributional shifts as concept drift to be discarded, directly conflicting with Online CIL where class arrivals are permanent structural events; disabling their drift detectors recovers substantial accuracy on our benchmarks, confirming that the failure stems from protocol misalignment rather than from the algorithms themselves. \vfdt{} zero-initialises child leaves; \textsc{VFDTc}~\citep{Gama2003} introduced Naive Bayes leaves but retained zero-initialisation at split time. Korycki and Krawczyk~\cite{korycki2021} motivate knowledge transfer at structural boundaries under replay; our work addresses the analogous problem under the strict exemplar-free constraint with formal variance reduction guarantees. Thompson-sampling trees~\cite{chaouki2024} are restricted to categorical attributes and a closed-world class set, making direct comparison inappropriate.

\noindent
\textbf{Concentration bounds for Gini-based splitting.}
Rutkowski et al.~\cite{rutkowski2013} established that McDiarmid's inequality applies to Gini split gain with constant $c_i = 8/n$---the first $K$-independent sensitivity constant. De Rosa and Cesa-Bianchi~\cite{derosa2015splitting} tightened this to $4/n$ for $K=2$ only. Our work closes the remaining gap: $4/n$ is tight for arbitrary $K$, while candidate-count dependence remains logarithmic in the operational radius. C-Tree~\cite{derosa2015splitting} derives refined large-deviation intervals using problem-dependent parameters, but its analysis is restricted to $K=2$; an empirical comparison is omitted as extending its implementation to $K>2$ is non-trivial.

\noindent
\textbf{Global methods for exemplar-free tabular CIL.}
\slda{}~\cite{hayes2020lifelong} uses a shared-covariance Gaussian discriminant, \sqda{}~\cite{hayes2020lifelong} relaxes
this to per-class Gaussians, NCM~\cite{mensink2013} tracks class centroids, and
RLS~\cite{zhuang2022acil} maintains a shared linear classifier via
recursive least squares. These are strong references on the standard benchmarks. Shared-covariance
or shared-representation updates can couple class decisions, whereas SQDA
and NCM retain class-specific summaries. These structural distinctions alone
do not establish a forgetting or accuracy ordering. MIST offers local
partitioning and flexible marginal densities, with limitations from univariate
splitting and feature factorisation.

\noindent
\textbf{Non-parametric and sketch-based prediction in streaming settings.}
KDE Naive Bayes~\cite{john1995,Moghaddam2002} improves accuracy under non-Gaussian
class-conditionals, but raw-sample KDE is infeasible in an exemplar-free
stream. Prior work~\cite{benhaim2010} uses additive-error quantile sketches for split threshold
evaluation without connecting the sketch to the leaf
predictor. Our work unifies both roles: the same per-leaf KLL sketches~\cite{Karnin2016} that
support continuous split evaluation also power a distribution-free leaf
predictor, yielding a streaming analog of KDE Naive Bayes with formal error
guarantees. KLL’s efficient update mechanism bounds the size of each sketch, and its high-probability rank guarantees feed directly into Theorem~\ref{thm:composed}, with $O(k)$ memory and
$O(\log k)$ amortised update cost.

\section{Preliminaries and Problem Setting}
\label{sec:prelim}
\subsection{Strict Online CIL}
We consider learning from an infinite stream $\mathcal{S} = \{(x_t, y_t)\}_{t=1}^\infty$
where the label space $\mathcal{Y}_t$ expands monotonically and
$K = \lim_{t\to\infty}|\mathcal{Y}_t|$ is unbounded and unknown a priori.
An algorithm operates in the \emph{strict Online CIL} regime if it satisfies:
(i)~single-pass processing; (ii)~exemplar-free memory; (iii)~anytime classification.
New-class arrivals are permanent structural events, not distributional shifts
to be discarded. Each split decision is evaluated locally on the $n$ samples
accumulated at a leaf; class-conditional distributions $p(x \mid y=c)$ are
stationary by protocol, and samples at a leaf are treated as an independent
batch~\cite{DomingosHulten2000}, satisfying McDiarmid's independence requirement.
\enlargethispage{2\baselineskip}

\subsection{Gini Split Gain and Bounded Differences}
Let $S$ be $n$ samples at a leaf with empirical class proportions $p_c = n_c/n$.
The Gini impurity and split gain for a binary partition $(S_L, S_R)$ are:
\begin{equation}
  G(S) = 1 - \sum_{c=1}^K p_c^2, \qquad
  F(S) = G(S) - \tfrac{n_L}{n}G(S_L) - \tfrac{n_R}{n}G(S_R).
  \label{eq:gain}
\end{equation}
A function $f : (\mathcal{X} \times \mathcal{Y})^n \to \mathbb{R}$ satisfies the \emph{bounded difference
property} with constants $c_1,\ldots,c_n$ if replacing any single input changes
$f$ by at most $c_i$. McDiarmid's inequality then gives:
\begin{equation}
\Prob\!\left[\,|f(S) - \mathbb{E}[f]| \geq \varepsilon\right]
\leq 2\exp\!\left(-\frac{2\varepsilon^2}{\sum_i c_i^2}\right).
\label{eq:mcdiarmid}
\end{equation}
The central theoretical question of this paper is: what is the tight value
of $c_i$ for the Gini split gain $F$? The answer, derived in
Section~\ref{sec:theory}, determines the sensitivity term in the confidence radius
used by \mist{}; candidate multiplicity still enters logarithmically.

\section{The \mist{} Framework}
\label{sec:framework}

Online CIL exposes two coupled failures in streaming trees: (i) the Information-Gain split test becomes miscalibrated as classes arrive ($\log_2 K$ range); and (ii) children created by a split start without local evidence, and their cold start feeds further unreliable splits. \mist{} answers them with a McDiarmid-calibrated Gini split test (Section~\ref{sec:split}) and Bayesian inheritance (Section~\ref{sec:inheritance}). Both must work without replaying data, which per-leaf KLL sketches make possible (Section~\ref{sec:sketch}). Each subsection opens with the problem it addresses (Figure~\ref{fig:mist-pipeline}). \mist{} comes in two variants that share the split test, inheritance, and sketches and differ only in the leaf predictor: \mist{}-G predicts from Gaussian per-feature running statistics, \mist{}-K from sketch-derived marginal densities (Section~\ref{sec:sketch}).

\begin{figure}[H]
\centering
\includegraphics[width=\textwidth]{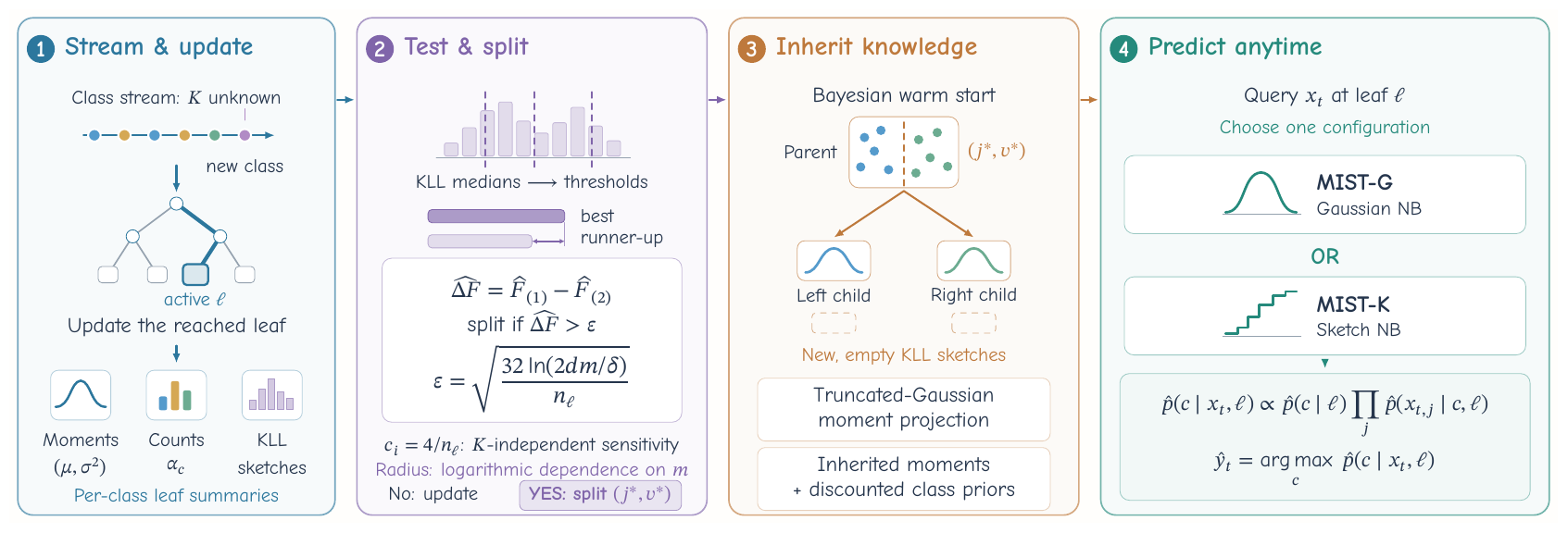}
\caption{Overview of \mist{}. Stream observations update only the reached leaf.
KLL sketches propose candidates for the McDiarmid-calibrated Gini split test.
Bayesian inheritance transfers moments and discounted class priors to new
children, whose KLL sketches start empty. Prediction uses either Gaussian
(\mist{}-G) or sketch-based (\mist{}-K) Naive Bayes likelihoods.}
\label{fig:mist-pipeline}
\end{figure}

\subsection{Sensitivity-Calibrated Split Criterion}
\label{sec:split}
\mist{} triggers a split when the empirical Gini gain gap $\widehat{\Delta F}$
between the best and second-best candidate feature exceeds a confidence
radius (the \emph{GapTest}). The radius is derived from the tight, class-count-independent McDiarmid calibration in
Section~\ref{sec:theory}; the candidate count still enters the operational radius logarithmically.

\begin{corollary}[Operational McDiarmid Radius]
\label{cor:radius}
Consider one split decision under the sample-independence assumptions
stated in Section~\ref{sec:prelim}. Let $d$ be the number of candidate
features and $m$ the number of candidate thresholds per feature. Applying
the McDiarmid bound of Theorem~\ref{thm:tight} and a union bound over the
$dm$ candidates, with probability at least $1-\delta$, the empirically
best candidate also has the largest expected Gini gain whenever its gain
gap over the second-best candidate satisfies:

\begin{equation}
  \widehat{\Delta F} > \varepsilon, \qquad
  \varepsilon = \sqrt{\frac{32\ln(2dm/\delta)}{n}}.
  \label{eq:radius}
\end{equation}
\end{corollary}

\begin{remark}
The factor of 32 arises from the GapTest structure: concluding
$\mathbb{E}[F_1] > \mathbb{E}[F_2]$ requires controlling both tails
simultaneously, so $\varepsilon = 2\eta$ where $\eta =
\sqrt{8\ln(2dm/\delta)/n}$ is the per-tail McDiarmid radius, giving
$\varepsilon = \sqrt{32\ln(2dm/\delta)/n}$. The union bound over $d$
features and $m$ thresholds enters only additively inside the logarithm
as $\ln(2dm/\delta)$, not multiplicatively into the constant 32.
\end{remark}

The radius has no multiplicative $\log_2 K$ range factor of the compared
Information-Gain-based Hoeffding criterion. It can still depend on the class
count through $m \leq K_{\mathrm{local}}-1$, where $K_{\mathrm{local}}$ is the
number of classes present at the leaf: the threshold count from
Section~\ref{sec:sketch} contributes an $O(\sqrt{\ln m})$ factor.

\noindent
\textbf{Role as structural regulariser.}
The conservatism of this criterion is a feature:
replacing the McDiarmid GapTest with a more aggressive criterion
increases tree size substantially at the cost of only marginal accuracy gains
(Table~\ref{tab:ablation-core}), consistent with the coupled failure
mode analysis (Figure~\ref{fig:tree-size}).

\subsection{Bayesian Knowledge Inheritance}
\label{sec:inheritance}
\noindent
\textbf{Motivation.}
A leaf that has accumulated sufficient evidence to split has rich
class-conditional statistics that should not be discarded. Standard streaming
trees initialise child leaves without prior information, forcing them to reconstruct predictive
distributions from scratch. In a strict exemplar-free setting, this
cold-start phase can be severe: the children must accumulate enough new
samples to form reliable predictions before they are useful, and any further
splits triggered from sparse children compound the problem. The conservative
split criterion of Section~\ref{sec:split} mitigates this by ensuring that
splits occur only from data-rich parents. We use Bayesian knowledge inheritance
to complete the picture by ensuring that the resulting children begin with
statistically grounded priors.

\noindent
\textbf{Truncated-Gaussian projection at split time.}
\mist{} uses Gaussian Naive Bayes (GNB) leaves, parameterised by
per-class, per-feature Welford accumulators $(\mu_j^c, \sigma_j^c)$ and
Dirichlet class-probability pseudo-counts. When a parent splits on
feature $j$ at threshold $v$, let $\zeta = (v - \mu_j^c)/\sigma_j^c$.
Exact truncated-Gaussian moment projection (Figure~\ref{fig:key-components}b) initialises the split-feature
statistics as:
\begin{align}
\label{eq:mu_left_right}
  \mu_j^{\mathrm{left},c}  = \mu_j^c
    - \sigma_j^c\frac{\phi(\zeta)}{\Phi(\zeta)}, \qquad
  \mu_j^{\mathrm{right},c} = \mu_j^c
    + \sigma_j^c\frac{\phi(\zeta)}{1 - \Phi(\zeta)},
\end{align}
with corresponding variance updates given in
Appendix~\ref{app:truncgauss}. For every non-split feature $k \neq j$,
under the Naive Bayes conditional independence assumption
$p(X_k \mid c, X_j \leq v) = p(X_k \mid c)$, parameters are inherited
unchanged.
Class mass is reapportioned as
$n_c^{\mathrm{left}} \leftarrow \alpha\,\Phi(\zeta)\,n_c^{\mathrm{parent}}$
and
$n_c^{\mathrm{right}} \leftarrow \alpha\,(1{-}\Phi(\zeta))\,n_c^{\mathrm{parent}}$,
with Dirichlet pseudo-counts reweighted identically by factor $\alpha$.

For categorical features encoded as one-hot indicators, the split is
binary and no moment projection is required: each child $s \in \{0,1\}$
inherits class mass $n_c^s \leftarrow \alpha\, n_c(s)$, where
$n_c(s) = \#\{i : y_i = c,\, x_i^{(\mathrm{attr})} = s\}$
is the empirical co-occurrence count at the parent leaf,
$\alpha \in (0,1]$ is the inheritance discount
(Section~\ref{sec:theory-var}), and Laplace smoothing at prediction
time ensures non-zero posteriors for all classes.

The inherited prior decays at rate $O(1/n_{\mathrm{child}})$ as new
samples accumulate, allowing each child to specialise to its local
geometry while retaining the benefit of the parent's experience during
the critical early phase.
This projection applies to both configurations (Section~\ref{sec:sketch}):
Gaussian moments initialise class-probability priors and apportion class
mass after every split. \mist{}-G inherits the full Gaussian leaf;
\mist{}-K's KLL sketches start empty (Algorithm~\ref{alg:split}), so
feature likelihoods require new samples, though Dirichlet priors are
immediately available.

\noindent
\textbf{Joint design with the split criterion.}
The variance reduction from inheritance scales inversely with the parent
sample count at split time
(Proposition~\ref{prop:variance}): a
split triggered too early yields a weaker prior and larger cold-start
variance for the child. The McDiarmid GapTest ensures splits occur only
from sufficiently data-rich parents, improving the use
of the inherited statistics. The two mechanisms reinforce each other:
the more conservative the split criterion, the richer the parent
statistics at split time, and therefore the stronger the variance
reduction guarantee.

\subsection{Label-Aware Sketching and the Unified Dual Role of KLL}
\label{sec:sketch}
\looseness=-1
Threshold search and prediction must work without replaying raw data, and one Gaussian per class can misfit multimodal or skewed marginals. \mist{} maintains per-leaf, per-class KLL quantile sketches of capacity
$k$. A sketch for class $c$ at leaf $\ell$ is created only when samples
of class $c$ are routed there, giving per-leaf memory
$O(d \cdot K_{\mathrm{local}}^{(\ell)} \cdot k)$. When useful partitions form, local class sparsity can reduce storage
relative to the worst-case bound $O(dkL_tK_t)$
(Appendix~\ref{app:memory-scaling}). Neither sparsity nor decreasing total
memory is guaranteed; Appendix~\ref{app:rebuttal-memory} reports measured
storage and allocation peaks.

\noindent
\textbf{Split threshold evaluation.}
Candidate thresholds are midpoints between adjacent class-conditional
medians from active sketches (Figure~\ref{fig:key-components}a), yielding at most $m \leq K_{\mathrm{local}}-1$
candidates per feature, exactly the $m$ in Eq.~\eqref{eq:radius}, so the
union bound covers only candidates actually considered. For categorical
attributes encoded as one-hot indicators, $m \leq 1$ and
Eq.~\eqref{eq:radius} remains valid without modification.

\noindent
\textbf{Leaf prediction: \mist{}-G and \mist{}-K.}
The same per-class sketches that generate split thresholds power an
optional non-parametric leaf predictor. In the default configuration
(\mist{}-G), leaf likelihoods are computed from the Gaussian parameters
maintained by the Welford accumulators; sketches contribute only to
split threshold evaluation. In the non-parametric configuration
(\mist{}-K), leaf likelihoods are instead estimated from sketch-derived
per-feature CDFs under a Naive Bayes factorisation:
\begin{equation}
  \hat{p}(x_j \mid c)
  = \frac{\widehat{F}_{j,c}(x_j + h_j)
          - \widehat{F}_{j,c}(x_j - h_j) + \epsilon_s}
         {2h_j + \epsilon_s},
  \label{eq:mistk-likelihood}
\end{equation}
where $\widehat{F}_{j,c}$ is the sketch-derived CDF, $h_j$ is the
inter-quantile spacing at $x_j$ (local bandwidth), and $\epsilon_s > 0$
is a dimensionless smoothing regulariser that prevents division by zero
when $h_j \to 0$; its contribution is negligible at interior points
where $h_j \gg \epsilon_s$.

\noindent
\textbf{Regime-dependent behaviour.}
\mist{}-G is preferable when class-conditional distributions are
approximately unimodal Gaussian, a condition empirically indicated by
strong \slda{}/\sqda{} performance on the same benchmark. \mist{}-K, the
empirical default when the marginal shape is unknown
(Appendix~\ref{app:predictor-choice}), avoids the systematic bias that the
Gaussian leaf assumption introduces under strongly non-Gaussian geometry
(multimodal, angular, or heavily skewed; Figure~\ref{fig:key-components}c).
The dual role of the KLL sketch is not an engineering convenience:
the same data structure that makes continuous split evaluation
memory-feasible also captures the distributional information needed for
geometry-adaptive prediction.
\begin{algorithm}[t]
\caption{\mist{}: Update and Predict}
\label{alg:mist}
\begin{algorithmic}[1]
\REQUIRE Sample $(x_t, y_t)$, tree $\mathcal{T}$, discount $\alpha$,
         failure probability $\delta$, grace period $g$
\STATE Route $x_t$ to leaf $\ell$
\STATE Update Welford accumulators $(\mu_j^c, \sigma_j^c)$ and Dirichlet
       pseudo-counts at $\ell$ with $(x_t, y_t)$
\STATE Update per-class KLL sketch at $\ell$ with $x_{t,j}$ for each
       feature $j$
\IF{$n_\ell \geq g$ \AND \textsc{GapTest}$(\ell, \delta)$}
  \STATE $j^*, v^* \leftarrow$ best feature and threshold from sketch medians
  \STATE \textsc{Split}$(\ell, j^*, v^*, \alpha)$
         \COMMENT{Algorithm~\ref{alg:split}}
\ENDIF
\STATE \textbf{return} $\arg\max_c \hat{p}(c \mid x_t, \ell)$
\STATE
\STATE \textbf{function} \textsc{GapTest}$(\ell, \delta)$
       \COMMENT{rule of Corollary~\ref{cor:radius}; guarantee holds per decision}
\STATE \quad $m \leftarrow K_{\mathrm{local}}^{(\ell)} - 1$
       \COMMENT{at most $K_{\mathrm{local}}^{(\ell)}-1$ candidate
       thresholds per feature; see Section~\ref{sec:sketch}}
\STATE \quad $\varepsilon \leftarrow \sqrt{32\ln(2dm/\delta)/n_\ell}$
\STATE \quad \textbf{return} $\widehat{F}_{(1)} - \widehat{F}_{(2)} > \varepsilon$
\end{algorithmic}
\end{algorithm}

\section{Theoretical Analysis}
\label{sec:theory}
The central result of this section is Theorem~\ref{thm:tight}: a tight,
$K$-independent bounded-difference constant for Gini split gain.
This is what makes Corollary~\ref{cor:radius} valid in an open world,
and Information-Gain-based Hoeffding criteria do not provide it.
Theorem~\ref{thm:composed} extends it to the sketched setting, and
Proposition~\ref{prop:variance} quantifies the variance reduction from
Bayesian inheritance.
\subsection{Tightness of the Gini Sensitivity}
\label{sec:theory-tight}
\begin{theorem}[Tightness of the Gini Sensitivity]
\label{thm:tight}
Let $S$ be a sequence of $n \geq 2$ samples and $F(S)$ the Gini split
gain defined in Eq.~\eqref{eq:gain}. For any single-sample replacement
$z_i \to z_i'$, the bounded-difference constant satisfies $c_i \leq 4/n$.
Moreover, this rate is tight: for any $\gamma > 0$, there exists a
configuration $S$ and a replacement such that
$|F(S) - F(S^{(i)})| > 4/n - \gamma$.
\end{theorem}

\begin{proof}[Proof sketch]
Case~A (same-child routing): replacing $z_i$ with $z_i'$ in child $L$
leaves $n_L$, $n_R$, and $G(S_R)$ unchanged, so only $G(S)$ and
$G(S_L)$ vary; each changes by at most $2/n$ and $2/n_L$, giving
$|\Delta F| \leq 2/n + (n_L/n)(2/n_L) = 4/n$.
Case~B (cross-routing): direct expansion over disjoint class proportions
gives $n|\Delta F| \leq 2 < 4$ (Appendix~\ref{app:case-b}).
Tightness: binary labels with $\lfloor\sqrt{n/2}\rfloor$ left-child
samples yield $|\Delta F| \to 4/n$ as $n \to \infty$.
\end{proof}

\begin{remark}
\label{rem:k-independence}
Rutkowski et al.~\cite{rutkowski2013} showed $c_i = 8/n$ for arbitrary $K$;
De Rosa and Cesa-Bianchi~\cite{derosa2015splitting} tightened this to $4/n$
for $K=2$ only, without establishing tightness.
Theorem~\ref{thm:tight} closes the gap: $4/n$ is tight for all $K$,
so its sensitivity constant is independent of $K$. The operational radius
of Corollary~\ref{cor:radius} retains logarithmic dependence on candidate
multiplicity, which can grow with the local class count.
\end{remark}
\subsection{Composed Split Error with Sketched Features}
\label{sec:theory-composed}

\begin{theorem}[Composed Split Error at Fixed Threshold]
\label{thm:composed}
Let $F^*$ be the expected true Gini gain and $\widehat{F}$ the estimate
from a quantile sketch with rank error $\varepsilon_{\mathrm{sketch}}$.
For any $\delta, \delta_{\mathrm{sketch}} \in (0,1)$, with probability
at least $1 - \delta - \delta_{\mathrm{sketch}}$:
\begin{equation}
  |\widehat{F} - F^*| \leq
  \underbrace{\sqrt{\frac{8\ln(2dm/\delta)}{n}}}_{\text{statistical}}
  + \underbrace{4\varepsilon_{\mathrm{sketch}}}_{\text{sketch}}.
  \label{eq:composed-error}
\end{equation}
The statistical term uses factor 8 rather than 32 because bounding
$|\widehat{F} - F^*|$ at a fixed threshold requires only a two-sided
single-quantity concentration argument, not the simultaneous two-tail
control needed by the GapTest in Eq.~\eqref{eq:radius}.
Full proof in Appendix~\ref{app:proof-composed}.
The factor~$4\varepsilon_{\mathrm{sketch}}$ follows from Theorem~\ref{thm:tight}:
rank error shifts routing counts by $n\varepsilon_{\mathrm{sketch}}$, costing
$(4/n)\cdot n\varepsilon_{\mathrm{sketch}}$ in Gini gain.
\end{theorem}

\begin{corollary}[Full Composed Error with Optimization Displacement]
\label{cor:opt-displacement}
Under Gaussianity of $X_j \mid c$ and finite condition number
$\kappa = \sigma_{\max}/\sigma_{\min}$, with probability at least
$1 - \delta - \delta_{\mathrm{sketch}}$:
\begin{equation}
  F(S;\,v^*) - F(S;\,\hat{v}^*)
  \leq
  \sqrt{\frac{8\ln(2dm/\delta)}{n}}
  + 4(1+\kappa)\,\varepsilon_{\mathrm{sketch}},
  \label{eq:composed-full}
\end{equation}
where $\hat{v}^*$ is the sketch-derived threshold. The additional
$4\kappa\,\varepsilon_{\mathrm{sketch}}$ over Theorem~\ref{thm:composed}
accounts for threshold displacement: KLL rank error $\varepsilon_{\mathrm{sketch}}$
implies $|\hat{v}^* - v^*| \leq \varepsilon_{\mathrm{sketch}}\,\sigma_{\max}\sqrt{2\pi}$
under Gaussianity, and by the Lipschitz constant of Lemma~\ref{lem:gini-lipschitz}
($L_v = O(1/\sigma_{\min})$), this propagates to a Gini loss of
$4\kappa\,\varepsilon_{\mathrm{sketch}}$.
Accuracy saturates at $k \geq 32$ and the gap to exact-threshold
evaluation falls below $0.1\%$ (Table~\ref{tab:ablation-k-alpha}),
confirming the displacement term is negligible in practice.
Full proof in Appendix~\ref{app:proof-displacement}.
\end{corollary}
\subsection{Variance Reduction from Bayesian Inheritance}
\label{sec:theory-var}
\begin{proposition}[Post-Split Variance Reduction]
\label{prop:variance}
Let $\alpha \in (0,1)$ be the inheritance discount and
$S_0 = \alpha n_{\mathrm{parent}}$ the inherited pseudo-count total.
\emph{(i) Class probabilities.}
Under the Dirichlet inheritance prior, the posterior variance satisfies
for all $t \geq 0$:
\begin{equation}
  \mathrm{Var}(\theta_c \mid \mathcal{D}_t)
  \leq \tfrac{1}{4(S_0 + 1)}
  = O\!\left(\tfrac{1}{\alpha n_{\mathrm{parent}}}\right).
\end{equation}
\emph{(ii) Feature parameters.}
Under the NIG conjugate prior with $n_0 = \alpha n_c^{\mathrm{parent}}$
pseudo-observations, the marginal posterior of $\mu_j^c$ is
Student-$t$ with $\nu = n_0$ degrees of freedom and scale
$(\sigma_j^c)^2/n_0$. Its variance exists when $n_0 > 2$ and equals:
\begin{equation}
  \mathrm{Var}(\mu_j^c \mid \mathcal{D}_0)
  = \frac{(\sigma_j^c)^2}{n_0} \cdot \frac{\nu}{\nu - 2}
  = \frac{(\sigma_j^c)^2}{\alpha n_c^{\mathrm{parent}} - 2}
  = O\!\left(\tfrac{1}{\alpha n_c^{\mathrm{parent}}}\right).
\end{equation}
For the split feature, $(\sigma_j^c)^2$ is replaced by
$\sigma_{\mathrm{trunc}}^2 \leq (\sigma_j^c)^2$.
Full proof in Appendix~\ref{app:proof-variance}.
\end{proposition}

\begin{remark}
Under the stated prior model, posterior variance decreases with inherited
mass. Conservative splitting can retain more parent observations before a
split, but does not maximise their number or guarantee representative child
priors. The tested values $\alpha\geq0.2$ recover near-full performance in
Table~\ref{tab:ablation-k-alpha}; the mismatch diagnostic in
Appendix~\ref{app:rebuttal-inheritance} shows why this is not universal.
\end{remark}

\section{Experiments}
\label{sec:experiments}

\subsection{Experimental Setup}
\label{sec:setup}

\noindent
\textbf{Datasets.}
We evaluate on two benchmark categories. \emph{Standard benchmarks} include
two synthetic Gaussian mixtures (Synth-10, Synth-50) and eight real-world
tabular datasets: Covertype, Split-MNIST, Pendigits, Shuttle, Letter, Wine,
Iris, and HAR. Categorical attributes are one-hot encoded before streaming.
\emph{Stress-test synthetics} (Section~\ref{sec:stress}) isolate specific
failure modes of parametric leaf predictors: Multimodal, Skewed,
AngularSectors, Antipodal, Ring, HeavyTail, NoisyFeature, and ConceptDrift.
For Split-MNIST, raw pixels ($d{=}784$) are reduced to $d{=}50$ via PCA
fitted on the first $1{,}000$ task-0 samples only. The frozen projection is
then applied uniformly to all methods.

\noindent
\textbf{Baselines.}
We compare two groups of methods. \emph{Tree-based:} \vfdt{}, Rutkowski,
EFDT, HAT, and \arf{}-NR (10 trees, drift detection disabled). EFDT and HAT
under default drift-reset collapse under strict Online CIL
(Appendix~\ref{app:noreset}); we report their no-reset variants in the main
comparison only. \emph{Non-tree:} \slda{}, \sqda{}, NCM, RLS. We also compare against PEC~\cite{zajkac2024}, a per-class exemplar-free baseline that avoids global coupling via independent class models. All
methods follow the strict Online CIL protocol: single-pass, no raw data
retention, sequentially expanding $K$\footnote{Note that these are the SOTA methods for our setting: online CIL with raw tabular features. More recent, technically heavily involved, CIL methods require pretrained backbones so they do not apply directly to our setting~\cite{goswami2023fecam, zhuang2024gacl, petit2023fetril}.}.

\noindent
\textbf{Hyperparameters.}
Sketch capacity $k$, inheritance discount $\alpha$, failure probability
$\delta$, grace period, and threshold generation policy are fixed globally
without per-dataset tuning; defaults are $k{=}64$, $\alpha{=}0.6$,
$\delta{=}0.10$, grace${=}200$ (full configurations in
Table~\ref{tab:hyperparams}). Sensitivity to $k$, $\alpha$, and
$\delta$ is in Table~\ref{tab:ablation-k-alpha}; sensitivity of
\mist{}-K to $\epsilon_s$ and $\beta$ is in
Appendix~\ref{app:mistk-sensitivity}.

\begin{table}[h]
\caption{Final mean accuracy among \textbf{tree-based methods} (mean $\pm$
std over 5 seeds). \textbf{Bold}: best within group. All methods follow the
no-reset CIL-adapted protocol; default drift-reset results for EFDT and HAT
are reported in Appendix~\ref{app:noreset}.
\arf{}-NR: Adaptive Random Forest with drift detection disabled (10 trees).}
\label{tab:tree-results}
\centering
\resizebox{\textwidth}{!}{%
\setlength{\tabcolsep}{3pt}
\begin{tabular}{lccccccccccc}
\toprule
\textbf{Method}
  & Synth-10 & Synth-50 & Wine & Iris
  & Covertype & Split-MNIST & Pendigits & Shuttle & Letter & HAR & \textbf{Mean} \\
\midrule
\mist{}-G (ours)
  & 0.994$\pm$0.004 & 0.999$\pm$0.002 & \textbf{0.970}$\pm$0.020
  & \textbf{0.933}$\pm$0.030 & 0.559$\pm$0.013 & 0.863$\pm$0.003
  & \textbf{0.889}$\pm$0.023 & 0.733$\pm$0.033
  & 0.678$\pm$0.005 & 0.607$\pm$0.013 & 0.8225 \\
\mist{}-K (ours)
  & 0.996$\pm$0.002 & \textbf{1.000}$\pm$0.001 & 0.957$\pm$0.038
  & 0.927$\pm$0.013 & \textbf{0.694}$\pm$0.006 & 0.772$\pm$0.007
  & 0.877$\pm$0.020 & \textbf{0.783}$\pm$0.067
  & \textbf{0.692}$\pm$0.009 & \textbf{0.625}$\pm$0.009 & \textbf{0.8323} \\
\midrule
\vfdt{}
  & 0.998$\pm$0.002 & 0.396$\pm$0.003 & \textbf{0.970}$\pm$0.020
  & \textbf{0.933}$\pm$0.030 & 0.605$\pm$0.010 & \textbf{0.874}$\pm$0.004
  & 0.508$\pm$0.039 & 0.417$\pm$0.145
  & 0.372$\pm$0.056 & 0.324$\pm$0.008 & 0.6397 \\
Rutkowski
  & \textbf{1.000}$\pm$0.000 & \textbf{1.000}$\pm$0.000 & \textbf{0.970}$\pm$0.020
  & \textbf{0.933}$\pm$0.030 & 0.559$\pm$0.013 & 0.851$\pm$0.003
  & 0.863$\pm$0.010 & 0.527$\pm$0.099
  & 0.678$\pm$0.005 & 0.614$\pm$0.018 & 0.7995 \\
EFDT (NR)
  & 0.219$\pm$0.038 & 0.040$\pm$0.000
  & \textbf{0.970}$\pm$0.020 & \textbf{0.933}$\pm$0.030
  & 0.277$\pm$0.002 & 0.182$\pm$0.009 & 0.272$\pm$0.057 & 0.143$\pm$0.000
  & 0.168$\pm$0.026 & 0.202$\pm$0.066 & 0.3406 \\
HAT (NR)
  & 0.423$\pm$0.049 & 0.084$\pm$0.017
  & \textbf{0.970}$\pm$0.020 & \textbf{0.933}$\pm$0.030
  & 0.262$\pm$0.032 & 0.324$\pm$0.037 & 0.614$\pm$0.037 & 0.220$\pm$0.063
  & 0.237$\pm$0.032 & 0.307$\pm$0.054 & 0.4374 \\
\arf{}-NR
  & \textbf{1.000}$\pm$0.000 & \textbf{1.000}$\pm$0.000
  & 0.952$\pm$0.010 & 0.927$\pm$0.025
  & 0.458$\pm$0.013 & 0.782$\pm$0.009 & 0.846$\pm$0.009 & 0.444$\pm$0.029
  & 0.568$\pm$0.012 & 0.601$\pm$0.010 & 0.7578 \\
\bottomrule
\end{tabular}}
\end{table}

\begin{table}[h]
\caption{Final mean accuracy among \textbf{non-tree baselines} (mean $\pm$
std over 5 seeds), included as reference. \textbf{Bold}: best within group.
All methods use the same metric and protocol.}
\label{tab:nontree-results}
\centering
\resizebox{\textwidth}{!}{%
\setlength{\tabcolsep}{3pt}
\begin{tabular}{lccccccccccc}
\toprule
\textbf{Method}
  & Synth-10 & Synth-50 & Wine & Iris
  & Covertype & Split-MNIST & Pendigits & Shuttle & Letter & HAR & \textbf{Mean}\\
\midrule
\mist{}-G (ours)  & 0.994$\pm$0.004 & 0.999$\pm$0.002 & 0.970$\pm$0.020
  & 0.933$\pm$0.030 & 0.559$\pm$0.013 & 0.863$\pm$0.003
  & 0.889$\pm$0.023 & 0.733$\pm$0.033
  & 0.678$\pm$0.005 & 0.607$\pm$0.013 & 0.8225 \\
\mist{}-K (ours)
  & 0.996$\pm$0.002 & \textbf{1.000}$\pm$0.001 & 0.957$\pm$0.038
  & 0.927$\pm$0.013 & \textbf{0.694}$\pm$0.006 & 0.772$\pm$0.007
  & 0.877$\pm$0.020 & \textbf{0.783}$\pm$0.067
  & 0.692$\pm$0.009 & 0.625$\pm$0.009 & 0.8323 \\
\midrule
\slda{}
  & \textbf{1.000}$\pm$0.000 & \textbf{1.000}$\pm$0.000
  & 0.979$\pm$0.021 & \textbf{0.973}$\pm$0.039
  & 0.673$\pm$0.007 & 0.854$\pm$0.005 & 0.879$\pm$0.007 & 0.529$\pm$0.000
  & 0.718$\pm$0.004 & \textbf{0.815}$\pm$0.001 & 0.8420 \\
\sqda{}
  & \textbf{1.000}$\pm$0.000 & 0.999$\pm$0.002
  & \textbf{0.987}$\pm$0.016 & 0.967$\pm$0.030
  & 0.528$\pm$0.004 & \textbf{0.940}$\pm$0.003
  & \textbf{0.982}$\pm$0.002 & 0.774$\pm$0.029
  & \textbf{0.895}$\pm$0.006 & 0.791$\pm$0.005 & \textbf{0.8863} \\
NCM
  & \textbf{1.000}$\pm$0.000 & \textbf{1.000}$\pm$0.000
  & 0.981$\pm$0.010 & 0.860$\pm$0.039
  & 0.603$\pm$0.012 & 0.794$\pm$0.003 & 0.817$\pm$0.009
  & 0.549$\pm$0.030 & 0.602$\pm$0.004 & 0.639$\pm$0.007 & 0.7845 \\
RLS
  & 0.999$\pm$0.002 & \textbf{1.000}$\pm$0.000
  & 0.981$\pm$0.028 & 0.827$\pm$0.049
  & 0.659$\pm$0.007 & 0.818$\pm$0.006 & 0.857$\pm$0.008 & 0.350$\pm$0.057
  & 0.596$\pm$0.004 & \textbf{0.815}$\pm$0.001 & 0.7902 \\
PEC
  & \textbf{1.000}$\pm$0.001 & 0.952$\pm$0.018 & 0.782$\pm$0.057 & 0.820$\pm$0.040 & 0.650$\pm$0.007 & 0.898$\pm$0.009 & 0.961$\pm$0.002 & 0.421$\pm$0.005 & 0.760$\pm$0.005 & 0.614$\pm$0.011 & 0.7857 \\
\bottomrule
\end{tabular}}
\end{table}

\subsection{Main Results}
\label{sec:results}
Among tree-based methods (Table~\ref{tab:tree-results}), \mist{}-G and
\mist{}-K are the strongest single-tree configurations across the benchmark
suite. \efdt{}~(NR) collapses on most benchmarks, consistent with
cold-start instability compounding under more aggressive split triggers
without inheritance. \arf{}-NR, with 10 base trees, falls below both
\mist{} configurations on most real-stream benchmarks,
at a memory overhead that depends on the dataset regime
(Table~\ref{tab:memory}, Appendix~\ref{app:memory-full}).

On near-Gaussian, low-$K$ streams (Synth-10, Wine, Iris), \vfdt{} matches
\mist{}-G on final accuracy, reflecting an honest trade-off: its more
aggressive Hoeffding criterion yields faster structural adaptation when class
geometry is simple. The forgetting results (Appendix~\ref{app:forgetting-full})
show that \mist{}-G consistently forgets less, so this parity comes at the
cost of greater instability under class arrival.
On Letter ($K{=}26$, $d{=}16$), zero splits under default settings
reflect the GapTest withholding commitment: the Gini gain gap between
candidate splits does not exceed the McDiarmid radius, so the root partition
is not informative enough to warrant a split.
Table~\ref{tab:nontree-results} compares accuracy with non-tree baselines.
\sqda{} has higher mean accuracy (0.8863) than \mist{}-K (0.8323) and
\mist{}-G (0.8225), including stronger results on Pendigits, Letter, and
Split-MNIST. \sqda{} also has lower mean forgetting (0.0397 versus 0.0530 for
\mist{}-G and 0.0714 for \mist{}-K; Table~\ref{tab:main-forgetting}) and higher
throughput (Table~\ref{tab:memory}). \mist{}-K performs well on Covertype and
Shuttle, but these results do not establish a general accuracy advantage over
Gaussian models. Appendices~\ref{app:prequential}--\ref{app:forgetting-full} report
prequential and forgetting diagnostics. Appendix~\ref{app:covonly} adds
a controlled dependence-only case where both MIST variants remain near
chance while SQDA benefits from class-specific covariance.

\textbf{When to prefer \mist{}.}
Prefer global Gaussian methods when class-conditionals are close to Gaussian or
the signal lies mainly in feature dependence; on 19 further real streams,
\mist{}-K beats the better of \sqda{} and \slda{} by more than 0.02 on 4, is
within $\pm0.02$ on 5, and trails on 10 (Appendix~\ref{app:real-sweep}). Prefer
\mist{} when a single-pass, local tree is required and the marginals are
non-Gaussian or multimodal enough to misspecify a global Gaussian
(AngularSectors, Table~\ref{tab:stress}). Neither family dominates; locating
the boundary online is open.

\subsection{Ablation Study}

We focus on two main ablation tests. Removing inheritance causes a relative drop of
$9.55\%$ ($7.86$ pp absolute), showing the contribution of post-split inheritance on these datasets
(Figure~\ref{fig:split-accuracy}, Appendix~\ref{app:prequential}).
Replacing GNB leaves with majority-vote causes a $71.97\%$ relative
collapse ($59.19$ pp absolute): without class-conditional likelihood
estimates, the tree cannot distinguish between classes sharing the same
partition cell.
The McDiarmid versus Hoeffding comparison shows negligible accuracy difference
($+1.02\%$) on near-Gaussian benchmarks. The structurally meaningful difference is
shown in Figure~\ref{fig:tree-size}
(Appendix~\ref{app:tree-dynamics}): Hoeffding triggers substantially
more splits as $K_{\mathrm{local}}$ grows, while McDiarmid maintains compact
trees.
GainTest yields marginal accuracy gains ($+1.20\%$) at the cost of
substantially larger trees (Figure~\ref{fig:tree-size}), confirming the
McDiarmid radius as a structural regulariser.
Accuracy is robust to $k$, $\alpha$, and $\delta$ over broad ranges
(Appendix~\ref{app:sensitivity}): performance saturates at $k \geq 32$ and
$\alpha \geq 0.2$, indicating low sensitivity within the tested ranges. Setting $\alpha = 0$ collapses to
$0.7439$, exactly matching the \emph{w/o Inheritance} ablation.

\begin{table}[h]
\caption{Core and split-criterion ablations (average accuracy over
ten standard datasets, mean over 5 seeds).}
\label{tab:ablation-core}
\centering
\small
\setlength{\tabcolsep}{3pt}
\begin{tabular}{lclc}
\toprule
Variant & Avg Acc. & Description & $\Delta$ \\
\midrule
\mist{}-G (default) & 0.8225 & McDiarmid + sketch + inheritance + GNB leaf & --- \\
\midrule
w/o Inheritance   & 0.7439 & No parent-to-child transfer         & $-$\textbf{9.55}\%  \\
w/o Sketching      & 0.8204 & Gaussian threshold approximation    & $-$0.26\%  \\
w/ Hoeffding       & 0.8309 & Hoeffding split rule                & $+$1.02\%  \\
Majority-vote leaf & 0.2306 & No density model                    & $-$\textbf{71.97}\% \\
w/ SampleSplit     & 0.8227 & Sample-based threshold generation   & $+$0.02\%  \\
w/ DenseThreshold  & 0.8218 & Dense threshold grid                & $-$0.09\%  \\
\midrule
GapTest (default)    & 0.8225 & Gap between best and second-best   & ---        \\
GainTest             & 0.8324 & Best gain vs.\ radius              & $+$1.20\%  \\
GainTestNoUnion      & 0.8405 & GainTest without union-bound term  & $+$2.19\%  \\
GaussianInheritance  & 0.8228 & GapTest + Gaussian CDF projection  & $+$0.04\%  \\
\bottomrule
\end{tabular}
\end{table}

\subsection{Stress Test on Non-Gaussian Geometry}
\label{sec:stress}

\begin{table}[h]
\caption{Stress-test results (mean $\pm$ std over 5 seeds) on non-Gaussian
synthetic streams.}
\label{tab:stress}
\centering
\resizebox{\textwidth}{!}{%
\setlength{\tabcolsep}{3pt}
\begin{tabular}{lccccccccc}
\toprule
Method & Antipodal & Multimodal & Ring & Skewed & HeavyTail
       & NoisyFeature & ConceptDrift & AngularSectors & \textbf{Mean} \\
\midrule
\mist{}-G & 0.834$\pm$0.033 & 0.919$\pm$0.032 & 0.975$\pm$0.019
          & 0.899$\pm$0.012 & 0.881$\pm$0.014 & 0.808$\pm$0.106
          & \textbf{0.993}$\pm$0.012 & 0.241$\pm$0.010 & 0.8188 \\
\mist{}-K & 0.829$\pm$0.021 & 0.983$\pm$0.005 & 0.968$\pm$0.026
          & 0.978$\pm$0.009 & 0.897$\pm$0.022 & 0.768$\pm$0.121
          & 0.988$\pm$0.011 & \textbf{0.436}$\pm$0.017 & \textbf{0.8559} \\
\midrule
\slda{}   & 0.171$\pm$0.042 & 0.756$\pm$0.053 & 0.978$\pm$0.020
          & 0.897$\pm$0.012 & \textbf{0.898}$\pm$0.025 & 0.811$\pm$0.104
          & \textbf{0.993}$\pm$0.012 & 0.240$\pm$0.041 & 0.7180 \\
\sqda{}   & \textbf{0.984}$\pm$0.007 & \textbf{0.999}$\pm$0.001
          & 0.975$\pm$0.021 & \textbf{1.000}$\pm$0.000
          & 0.858$\pm$0.012 & 0.798$\pm$0.109
          & \textbf{0.993}$\pm$0.012 & 0.235$\pm$0.015 & 0.8553 \\
\arf{}-NR & 0.862$\pm$0.025 & 0.947$\pm$0.018 & \textbf{0.979}$\pm$0.012
          & 0.921$\pm$0.016 & 0.893$\pm$0.019 & \textbf{0.832}$\pm$0.091
          & 0.991$\pm$0.009 & 0.255$\pm$0.021 & 0.8350 \\
PEC       & 0.918$\pm$0.010 & 0.964$\pm$0.017 & 0.926$\pm$0.035 & 0.972$\pm$0.009 & 0.801$\pm$0.012 & 0.233$\pm$0.025 & 0.984$\pm$0.018 & 0.302$\pm$0.046 & 0.7625 \\
\bottomrule
\end{tabular}}
\end{table}
Table~\ref{tab:stress} shows complementary strengths. On AngularSectors,
\mist{}-K reaches 0.436 versus 0.235 for \sqda{} and 0.255 for \arf{}-NR.
On Multimodal and Skewed, \mist{}-K improves over \mist{}-G, while
\sqda{} remains strongest. The stress-test means of \mist{}-K (0.8559)
and \sqda{} (0.8553) are close, and the \sqda{} lead over \slda{} on
Antipodal shows the value of class-specific covariance.
Figure~\ref{fig:stress-vis} provides PCA visualisations of the stress streams.
These results support benefits on particular marginal geometries, rather
than robustness to arbitrary non-Gaussian distributions.

For unknown marginal shape, the 25-dataset study
(Appendix~\ref{app:predictor-choice}) supports \mist{}-K as the default, with
\mist{}-G useful when Gaussian marginals fit; hard likelihood selection does not
beat always using K, and blending is optional.

\mist{} retains per-leaf, per-class summaries, with worst-case storage
$O(dkL_tK_t)$; local class sparsity can reduce this cost but is not guaranteed.
Appendix~\ref{app:rebuttal-memory} compares 19 streams, including regimes where
\mist{} costs more than \sqda{}, and separates retained storage from peak
traced allocations. Sketch-quantile initialisation costs $1.27\times$ Gaussian
projection (Appendix~\ref{app:sketch-vs-gaussian}). Inheritance depends on how well parent
summaries represent each child: on ShiftAtSplit, accuracy falls from 0.935 at
$\alpha=0.2$ to 0.864 at $\alpha=1$ (Appendix~\ref{app:rebuttal-inheritance},
Tables~\ref{tab:inheritance-mismatch} and~\ref{tab:inheritance-transient}).

\vspace{-0.2cm}
\section{Conclusion}
\label{sec:conclusion}
Streaming decision tree failure in online CIL traces to two coupled sources: a split criterion that degrades as $K$ grows, and the absence of parent-to-child knowledge transfer. \mist{} addresses them with a McDiarmid radius built on a tight, $K$-independent sensitivity bound ($c_i{=}4/n$) that mitigates premature splits, Bayesian inheritance for post-split stabilisation, and unified KLL sketches for split evaluation and prediction.
Empirically, \mist{}-G is competitive with global parametric methods on near-Gaussian streams, although \sqda{} has the higher mean accuracy. On AngularSectors, \mist{}-K outperforms the per-class discriminants and 10-tree ensemble. This is specific to the tested geometry; SQDA remains stronger on several other benchmarks, motivating both predictors.

Regarding limitations, both predictors use a Naive Bayes factorisation;
\mist{}-K replaces Gaussian leaf likelihoods, not the Gaussian inheritance
projection. Both can fail on dependence-only signals
(Appendix~\ref{app:covonly}); a covariance-aware leaf could help but needs
new theory, since per-class covariances estimated on the sparse leaves that
the tree structure creates degenerate; we leave it to future work. Inheritance can introduce
bias when parent and child distributions differ
(Appendix~\ref{app:rebuttal-inheritance}). Memory grows with leaves, local classes,
dimension, and sketch capacity. The analysis assumes stationary
class-conditionals and independent samples for each concentration argument.
Reliable online predictor selection remains open, and time-uniform confidence
sequences \citep{Howard2021} are future work: the per-decision union bound gives
no infinite-horizon guarantee.

Ultimately, \mist{} decouples local predictive modules and transfers statistically grounded knowledge at split time, giving streaming decision trees a modular route to open-world continual learning that mitigates forgetting as the tree grows.

\section*{Acknowledgments and Disclosure of Funding}
We thank the Faculty of Information Technology, University of Science, Vietnam
National University, Ho Chi Minh City, for providing the computing resources
and research environment that supported this work. The authors declare no
competing interests.

\bibliographystyle{plain}
\bibliography{references}

\newpage
\appendix
\etocdepthtag.toc{mtappendix}
\etocsettagdepth{mtmain}{none}
\etocsettagdepth{mtappendix}{subsection}
\renewcommand{\contentsname}{Appendix Contents}
{\footnotesize\linespread{0.9}\selectfont\makeatletter
\def\l@section#1#2{\addvspace{2pt}\@dottedtocline{1}{0em}{1.5em}{\bfseries #1}{\bfseries #2}}
\makeatother\tableofcontents}
\clearpage
\section{Broader Context: Online Continual Learning}
\label{app:related_work}

\subsection{Catastrophic forgetting and the continual-learning taxonomy}
\label{app:related_work:taxonomy}

The tendency of connectionist learners to overwrite previously acquired knowledge was
formalised as \emph{catastrophic interference} by~\cite{mccloskey1989catastrophic}
and~\cite{ratcliff1990connectionist}, and revisited for modern networks
by~\cite{french1999catastrophic} and~\cite{goodfellow2013empirical}.
Modern continual learning (CL) is commonly organised along three scenarios, including \textbf{task-incremental, domain-incremental, and class-incremental}, formalised
by~\cite{vandeven2019three} and later refined in~\cite{vandeven2022three};
the categorisation was independently validated by~\cite{hsu2018reevaluating}.
Class-incremental learning (CIL) is widely regarded as the hardest setting because
the task identifier is not supplied at inference and the classifier must discriminate
between all classes ever seen, requiring calibrated cross-task decision boundaries.
Broader surveys include~\cite{parisi2019continual,delange2022survey}, and,
for the specifically \emph{online} (single-pass, small-batch) regime,~\cite{mai2022online}.
Streaming settings aggravate forgetting relative to the batch case because each sample
is seen only once and stationarity cannot be assumed, motivating drift-aware
evaluation~\cite{ghunaim2023realtime}.

\subsection{Regularisation-based methods}
\label{app:related_work:regularisation}

Regularisation approaches anchor parameters deemed important for past tasks.
Elastic Weight Consolidation (EWC) uses a Fisher-information
penalty~\cite{kirkpatrick2017overcoming}; Synaptic Intelligence accumulates path
integrals of parameter sensitivity~\cite{zenke2017continual}; Memory-Aware Synapses
uses gradient magnitudes of the learned function~\cite{aljundi2018memory};
Learning without Forgetting employs knowledge distillation from the previous model
as a functional prior~\cite{li2018learning}.
Online variants such as Progress \& Compress~\cite{schwarz2018progress} and
RWalk~\cite{chaudhry2018riemannian} were developed to amortise Fisher estimation
over streams.
These methods are effective for short task sequences but degrade sharply in unbounded
class-incremental streams: importance estimates drift, the quadratic penalty cannot
simultaneously protect arbitrarily many past classes, and they do not address the
representation--classifier mismatch at the core of CIL.

\subsection{Replay and rehearsal}
\label{app:related_work:replay}

Replay stores or synthesises past data to counteract forgetting.
Experience Replay~\cite{rolnick2019experience}, iCaRL~\cite{rebuffi2017icarl},
GEM~\cite{lopezpaz2017gem} and its efficient variant A-GEM~\cite{chaudhry2019agem}
constrain updates with respect to a memory buffer; tiny-memory baselines are analysed
in~\cite{chaudhry2019tiny}.
DER++~\cite{buzzega2020der} augments replay with logit distillation,
while MIR~\cite{aljundi2019mir} and GSS~\cite{aljundi2019gss} target the
\emph{online} regime with interference-aware retrieval and gradient-based sample selection.
ER-ACE~\cite{caccia2022eracer} mitigates representation drift caused by
na\"{i}ve cross-entropy replay.
Generative replay replaces the buffer with a learned generator~\cite{shin2017dgr},
and GDumb~\cite{prabhu2020gdumb} shows that a buffer-trained model is a strong baseline.
Selective replay for lifelong reinforcement learning was proposed
by~\cite{isele2018selective}.
All of these methods \textbf{violate the exemplar-free constraint} relevant to
privacy-sensitive tabular applications such as finance and healthcare, and generative
replay additionally compounds forgetting via generator drift.

\subsection{Architecture-expansion and parameter isolation}
\label{app:related_work:architecture}

A complementary family grows or partitions the network per task:
Progressive Networks~\cite{rusu2016progressive},
PackNet~\cite{mallya2018packnet},
Dynamically Expandable Networks~\cite{yoon2018lifelong},
HAT~\cite{BifetGavalda2009HAT},
Piggyback~\cite{mallya2018piggyback}, and
Expert Gate~\cite{aljundi2017expert}.
These methods largely avoid forgetting by construction but require a task
identifier at inference or introduce unbounded parameter growth, which is
incompatible with an open-ended class stream.

\subsection{Exemplar-free deep continual learning}
\label{app:related_work:exemplar_free}

Exemplar-free CIL methods retain only compact statistics.
PASS~\cite{zhu2021prototype} uses prototype augmentation and self-supervision;
IL2M~\cite{belouadah2019il2m} rebalances logits with class statistics;
SSRE~\cite{zhu2022self} expands representations without exemplars;
PRAKA~\cite{shi2023prototype} introduces prototype reminiscence and asymmetric
knowledge aggregation;
FeTrIL~\cite{petit2023fetril} translates old-class features from new-class geometry;
FeCAM~\cite{goswami2023fecam} exploits heterogeneous class covariances.
Deep SLDA~\cite{hayes2020lifelong} freezes a backbone and maintains a
streaming Gaussian classifier.
Prompt-based methods such as L2P~\cite{wang2022l2p},
DualPrompt~\cite{wang2022dualprompt}, and CODA-Prompt~\cite{smith2023coda}
exploit large pre-trained vision transformers with learnable prompt pools; they
depend on the quality and transferability of the frozen backbone, a resource
unavailable in most tabular domains.

\subsection{Online and streaming-specific methods}
\label{app:related_work:online}

REMIND~\cite{hayes2020remind} and Latent Replay~\cite{pellegrini2020latent}
compress activations to enable single-pass rehearsal;
\cite{lesort2020continual} surveys the online robotics setting;
\cite{ghunaim2023realtime} argues for throughput-aware evaluation.
In contrast to multi-epoch offline CIL, these works highlight the stringent
per-sample compute budget that governs truly online learners.

\subsection{Tabular continual learning and streaming decision trees}
\label{app:related_work:tabular}

Continual learning on \textbf{tabular} streams is markedly under-explored.
Recent work considers tabular CL in energy-constrained~\cite{wang2025imlp}
and contrastive or out-of-distribution settings~\cite{tccl2025}, as well as
financial-audit anomaly detection~\cite{hemati2022continual}; a pseudo-rehearsal
approach is reported in~\cite{tril3_2025}.
These efforts remain either replay-based or confined to narrow domains, leaving
class-incremental tabular learning largely open.

The complementary data-stream-mining literature offers decision-tree learners with
provable split guarantees: the Hoeffding-bound VFDT~\cite{DomingosHulten2000} and
its concept-drift extension CVFDT~\cite{HultenSpencerDomingos2001}; the Extremely
Fast Decision Tree (EFDT)~\cite{ManapragadaWebbSalehi2018}; the Hoeffding Adaptive
Tree~\cite{BifetGavalda2009HAT}; streaming
ensembles~\cite{BifetGavalda2009HAT,gomes2017};
and change-detection mechanisms ADWIN~\cite{BifetGavalda2007ADWIN} and
DDM~\cite{GamaMedasCastilloRodrigues2004}.
Crucially, \cite{rutkowski2013} show that VFDT's use of
Hoeffding's inequality is statistically inappropriate for information-gain-type split
criteria and replace it with the tighter \textbf{McDiarmid bound}.
However, none of these streaming-tree methods are designed or analysed for
class-incremental learning: new classes invalidate cached sufficient statistics at
existing splits, pure leaf counts drift as the label space grows, and classical drift
detectors conflate concept drift with class emergence.

\subsection{Positioning of MIST}
\label{app:related_work:positioning}

MIST sits at the intersection of the gaps identified above.
It is \textbf{exemplar-free} (no raw sample buffers or generative models,
unlike~\S\ref{app:related_work:replay}),
\textbf{single-pass and online} (unlike most methods
in~\S\ref{app:related_work:regularisation},
\S\ref{app:related_work:architecture}, and \S\ref{app:related_work:exemplar_free}),
\textbf{class-set-unbounded} (unlike parameter-isolation methods
in~\S\ref{app:related_work:architecture} and short-horizon regularisers
in~\S\ref{app:related_work:regularisation}), and \textbf{designed for tabular streams}
(unlike methods in~\S\ref{app:related_work:online}--\S\ref{app:related_work:tabular},
which presuppose visual pre-training or fixed feature spaces).

Methodologically, MIST replaces the statistically loose Hoeffding split test of
VFDT and EFDT with a \textbf{McDiarmid-calibrated criterion} in the spirit
of~\cite{rutkowski2013}, and extends it with \textbf{Bayesian inheritance of
class-conditional leaf statistics} so that newly emerging classes borrow strength
from ancestor nodes rather than restarting estimation.
This yields a streaming decision-tree learner that tolerates open class sets
without replay, capacity expansion, or a pre-trained backbone.
A detailed analysis of the GapTest threshold and the tightness of the
McDiarmid constant $c_i = 4/n$ is provided in Section~\ref{sec:theory}
and Appendix~\ref{app:proofs}.

\section{Algorithm Pseudocode}
\label{app:algorithms}

Algorithm~\ref{alg:mist} in the main text gives the high-level MIST update and
predict loop. We provide here the two subroutines that it invokes:
Algorithm~\ref{alg:split} details the Bayesian inheritance protocol executed at
every split event, and Algorithm~\ref{alg:mistk} details the sketch-based leaf
likelihood used by \mist{}-K. Both pseudocode listings use the notation
established in Section~4 of the main paper; $\phi$ and $\Phi$ denote the
standard normal PDF and CDF respectively.

\begin{algorithm}[h]
\caption{\mist{}: Split with Bayesian Inheritance}
\label{alg:split}
\begin{algorithmic}[1]
\REQUIRE Leaf $\ell$, split feature $j^*$, threshold $v^*$,
         discount $\alpha \in (0,1]$
\FOR{each child $s \in \{\mathrm{left}, \mathrm{right}\}$}
  \FOR{each class $c$}
    \IF{$j^*$ is continuous}
      \STATE $\zeta \leftarrow (v^* - \mu_{j^*}^c)/\sigma_{j^*}^c$
      \STATE $\tilde\Phi \leftarrow \Phi(\zeta)$ if $s{=}\mathrm{left}$,
             else $1{-}\Phi(\zeta)$
      \STATE $\mu_{j^*}^{s,c} \leftarrow \mu_{j^*}^c -
             \sigma_{j^*}^c\,\phi(\zeta)/\tilde\Phi$ if $s{=}\mathrm{left}$,
             else $\mu_{j^*}^c + \sigma_{j^*}^c\,\phi(\zeta)/\tilde\Phi$
             \COMMENT{Eq.~\eqref{eq:mu_left_right}}
      \IF{$s = \mathrm{left}$}
        \STATE $(\sigma_{j^*}^{s,c})^2 \leftarrow \sigma_{j^*}^{c\,2}
               \!\left[1 - \dfrac{\zeta\phi(\zeta)}{\tilde\Phi}
               - \left(\dfrac{\phi(\zeta)}{\tilde\Phi}\right)^{\!2}\right]$
               \COMMENT{upper-truncated; Appendix~\ref{app:truncgauss}}
      \ELSE
        \STATE $(\sigma_{j^*}^{s,c})^2 \leftarrow \sigma_{j^*}^{c\,2}
               \!\left[1 + \dfrac{\zeta\phi(\zeta)}{\tilde\Phi}
               - \left(\dfrac{\phi(\zeta)}{\tilde\Phi}\right)^{\!2}\right]$
               \COMMENT{lower-truncated; sign of linear term is $+$;
               Appendix~\ref{app:truncgauss}}
      \ENDIF
      \STATE $n_c^s \leftarrow \tilde\Phi \cdot \alpha \cdot n_c^\ell$
    \ELSE
      \STATE \COMMENT{categorical one-hot: split feature is deterministic in each child}
      \STATE $v_{\mathrm{cat}} \leftarrow 0$ if $s{=}\mathrm{left}$, else $1$
      \STATE $\mu_{j^*}^{s,c} \leftarrow v_{\mathrm{cat}}$;\quad
             $(\sigma_{j^*}^{s,c})^2 \leftarrow 0$
             \COMMENT{$x_{j^*}{=}v_{\mathrm{cat}}$ w.p.\ 1 in child $s$; inherit no uncertainty}
      \STATE $n_c^s \leftarrow \alpha \cdot n_c(v_{\mathrm{cat}})$,\;
             $n_c(v_{\mathrm{cat}}){=}\#\{i:y_i{=}c,\,x_i^{(\mathrm{attr})}{=}v_{\mathrm{cat}}\}$
    \ENDIF
    \FOR{$k \neq j^*$}
      \STATE $(\mu_k^{s,c},\,\sigma_k^{s,c}) \leftarrow (\mu_k^c,\,\sigma_k^c)$
    \ENDFOR
  \ENDFOR
  \STATE Initialise KLL sketches at $s$ as empty
  \COMMENT{\mist{}-G: unaffected (Gaussian accumulators inherited above).
           \mist{}-K: sketch-based feature likelihoods unavailable until
           new samples arrive; class-probability priors remain valid.}
\ENDFOR
\STATE Remove $\ell$ from active leaves
\end{algorithmic}
\end{algorithm}

\begin{algorithm}[h]
\caption{\mist{}-K: Sketch-Based Leaf Likelihood}
\label{alg:mistk}
\begin{algorithmic}[1]
\REQUIRE Observation $x$, leaf $\ell$, smoothing $\epsilon_s$ (default $1.0$),
         bandwidth multiplier $\beta$,
         Dirichlet pseudo-counts $\{\alpha_c^{(\ell)}\}$ where
         $\alpha_c^{(\ell)} = n_c^{(\ell)} + \epsilon_s$,\\
         $\mathrm{IQS}_{\ell,c}(x_j) = \widehat{F}_{j,c}^{-1}(\min(r+h_0,1))
         - \widehat{F}_{j,c}^{-1}(\max(r-h_0,0))$
         where $r = \widehat{F}_{j,c}(x_j)$ is the sketch-estimated rank
         of $x_j$ and $h_0=0.25$ is the fixed half-window;
         returns $\epsilon_s$ if fewer than two distinct quantiles exist
\FOR{each class $c$}
  \STATE $\log\hat{p}(c) \leftarrow \log\alpha_c^{(\ell)} -
         \log\sum_{c'}\alpha_{c'}^{(\ell)}$
         \COMMENT{Dirichlet prior}
  \FOR{each feature $j$}
    \IF{$j$ is categorical}
      \STATE $\hat{p}_j \leftarrow (n_c^{(\ell)}(j{=}x_j) + \epsilon_s)\;/\;
             (n_c^{(\ell)} + 2\epsilon_s)$
             \COMMENT{Bernoulli likelihood from Dirichlet pseudo-counts;
             $n_c^{(\ell)}(j{=}x_j)$: count of class-$c$ samples with
             $x_j$ matching the observed category}
      \STATE $\log\hat{p}(x_j \mid c) \leftarrow \log\hat{p}_j$
    \ELSE
      \STATE $h_j \leftarrow \beta \cdot \mathrm{IQS}_{\ell,c}(x_j)$
             \COMMENT{local inter-quantile spacing}
      \STATE $\hat{F}^+ \leftarrow \widehat{F}_{j,c}(x_j + h_j)$;\quad
             $\hat{F}^- \leftarrow \widehat{F}_{j,c}(x_j - h_j)$
      \STATE $\log\hat{p}(x_j \mid c) \leftarrow
             \log(\hat{F}^+ - \hat{F}^- + \epsilon_s) -
             \log(2h_j + \epsilon_s)$
    \ENDIF
  \ENDFOR
  \STATE $\log\hat{p}(c \mid x) \leftarrow \log\hat{p}(c) +
         \sum_j \log\hat{p}(x_j \mid c)$
\ENDFOR
\STATE \textbf{return} $\arg\max_c \log\hat{p}(c \mid x)$
\end{algorithmic}
\end{algorithm}

\section{Illustrations of the Key Components}
\label{app:illustrations}

Figure~\ref{fig:key-components} shows, on data, the mechanisms of
Section~\ref{sec:framework}: how class medians propose a split threshold
(Section~\ref{sec:sketch}), how a split changes the distribution each child
inherits (Section~\ref{sec:inheritance}), and why a single Gaussian can
misfit a skewed class (Section~\ref{sec:sketch}).

\begin{figure}[h]
\centering
\begin{subfigure}[t]{0.32\textwidth}
  \includegraphics[width=\textwidth]{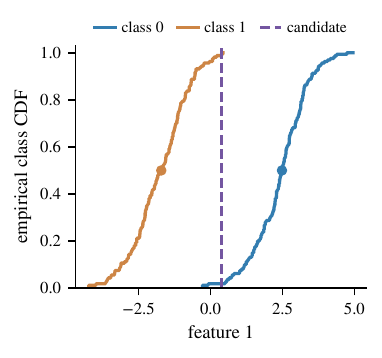}
  \caption{Candidate threshold}
\end{subfigure}\hfill
\begin{subfigure}[t]{0.32\textwidth}
  \includegraphics[width=\textwidth]{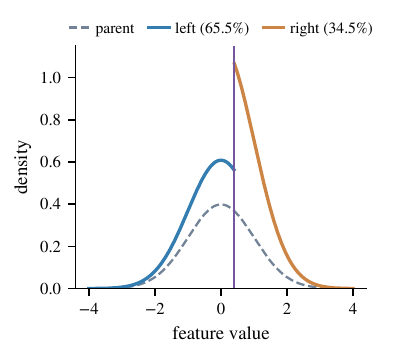}
  \caption{Inherited distributions}
\end{subfigure}\hfill
\begin{subfigure}[t]{0.32\textwidth}
  \includegraphics[width=\textwidth]{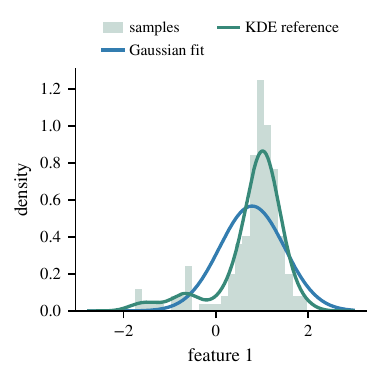}
  \caption{Leaf likelihood}
\end{subfigure}
\caption{Illustrations of the key components. (a) Exact empirical class CDFs of
feature~1 on Synth-10 (seed~42, classes 0 and 1); the midpoint of the class
medians is the candidate threshold, and CDF queries score the mass on each side.
(b) A Gaussian parent split at $0.4$ and the conditional densities that
truncated-Gaussian projection initialises for the two children (analytical;
percentages are each child's share of the parent mass).
(c) Class-0 samples of the Skewed stream (seed~42, feature~1) with a
moment-matched Gaussian (\mist{}-G) and a raw-sample KDE as a visual reference
for \mist{}-K. Panels (a) and (c) use raw data for illustration only; neither
curve is a measured KLL-sketch output.}
\label{fig:key-components}
\end{figure}

\section{Extended Experimental Setup}
\label{app:setup}

\subsection{Hyperparameter Configurations}
\label{app:hyperparams}

Table~\ref{tab:hyperparams} lists the fixed default hyperparameters used for
every method in the main benchmark comparison. No method-specific held-out
tuning was performed; the same configuration was applied uniformly across all
datasets. For \mist{}, the defaults were chosen to reflect the theoretically
motivated operating points: $k{=}64$ (sketch capacity at which the safety
margin of the $4\varepsilon_{\mathrm{sketch}}$ bound saturates;
Appendix~\ref{app:sketch-bound}), $\alpha{=}0.6$ (well within the plateau
$\alpha \geq 0.2$ identified in Table~\ref{tab:ablation-k-alpha}),
and $\delta{=}0.10$ (consistent with the near-flat $\sqrt{\ln(1/\delta)}$
sensitivity; Table~\ref{tab:ablation-k-alpha}).

\begin{table}[h]
\caption{Fixed default configurations used in the main benchmark comparison.}
\label{tab:hyperparams}
\centering
\small
\begin{tabular}{lll}
\toprule
Method & Key hyperparameters & Policy \\
\midrule
\mist{}-G / \mist{}-K & grace=200, $\delta$=0.10, tie=0.05, $k$=64 & fixed default \\
\vfdt{} (Gini)         & grace=200, confidence=0.95, tie=0.05        & fixed default \\
EFDT / HAT (Gini)             & grace=50, split conf=$10^{-7}$, tie=0.05    & fixed default \\
Rutkowski (Gini)             & grace=200,  $\delta$=0.10, tie=0.05    & fixed default \\
\slda{} / \sqda{}      & shrinkage=$10^{-4}$                          & fixed default \\
NCM / RLS / PEC        & implementation defaults                       & fixed default \\
\bottomrule
\end{tabular}
\end{table}

\subsection{Protocol Mismatch in Drift-Reset Trees}
\label{app:noreset}

Table~\ref{tab:noreset} reports a focused diagnostic comparing EFDT and HAT
under their default drift-triggered reset configuration against the no-reset
variant used throughout the main experiments. Under strict Online CIL,
class arrivals are permanent structural events; drift detectors treat them
as transient shifts and discard the associated leaf statistics upon reset,
causing near-random performance.

When resets are disabled, HAT recovers substantial accuracy on several
datasets (e.g.\ Pendigits: $0.250 \to 0.614$, Split-MNIST: $0.170 \to 0.324$),
confirming that its failure under the default configuration is a consequence
of protocol misalignment. EFDT (no-reset) shows modest recovery on some
datasets (e.g.\ Covertype: $0.167 \to 0.277$, Pendigits: $0.196 \to 0.272$)
but remains below competitive accuracy, consistent with cold-start instability
without inheritance.

\begin{table}[h]
\caption{Default vs.\ no-reset diagnostics on selected datasets (final mean
accuracy, mean over seeds).}
\label{tab:noreset}
\centering\small
\begin{tabular}{lcccccc}
\toprule
Method & Pendigits & Shuttle & Letter & HAR & Covertype & Split-MNIST \\
\midrule
EFDT+ADWIN (reset)  & 0.196 & 0.253 & 0.121 & 0.167 & 0.167 & 0.167 \\
EFDT (no-reset)     & 0.272 & 0.143 & 0.168 & 0.202 & 0.277 & 0.182 \\
HAT (default)       & 0.250 & 0.207 & 0.107 & 0.231 & 0.168 & 0.170 \\
HAT (no-reset)      & 0.614 & 0.220 & 0.237 & 0.307 & 0.262 & 0.324 \\
\bottomrule
\end{tabular}
\end{table}

\subsection{Baseline Fairness and Tuning Policy}
\label{app:baselines-fairness}

All main benchmark results (Tables~\ref{tab:tree-results}
and~\ref{tab:main-forgetting}) are produced without method-specific held-out
tuning; the policy in Table~\ref{tab:hyperparams} is applied uniformly to all
methods. For non-tree baselines, we use published implementation defaults
without further adjustment. \arf{}-NR uses 10 base trees with drift detection
disabled, matching the no-reset convention of the tree-based group.

\section{Main Results: Extended Analysis}
\label{app:results}

\subsection{Prequential Accuracy and Knowledge Inheritance}
\label{app:prequential}

Figure~\ref{fig:split-accuracy} measures what happens when a node splits. We hook
the split routine and evaluate the held-out test examples of all classes seen
so far immediately before and after each actual structural split, using the
same examples on both sides (at most 3{,}000, fixed across the run). MIST-G and
MIST-K change by at most $0.6$ points across the six streams, whereas without
inheritance accuracy falls from $84$--$100\%$ to $0$--$29\%$: the new children
carry no class statistics until they see data. Only seeds in which the tree
split enter each mean (counts are in the panel titles). This is a diagnostic run
with the default configuration and the standard, unweighted Gini criterion; it
is not part of the main benchmark.

\begin{figure}[t]
\centering
\includegraphics[width=0.95\textwidth]{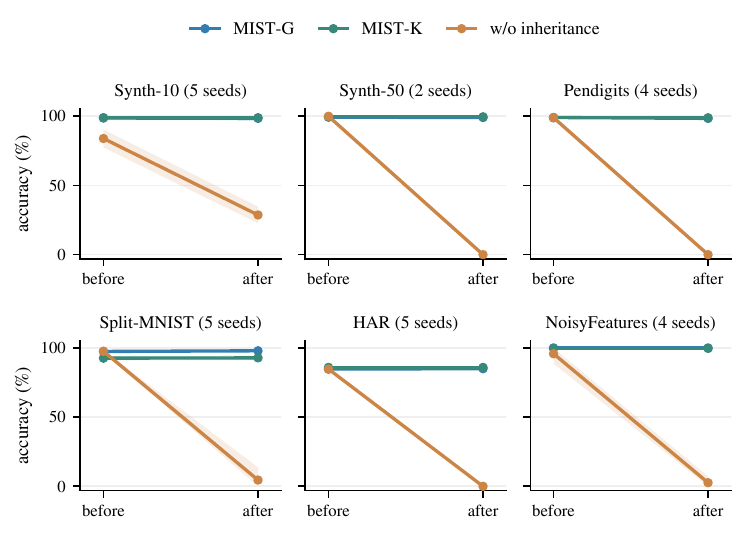}
\caption{Held-out accuracy on all classes seen so far immediately before and
after each structural split (mean $\pm$ std over the seeds in which the tree
split). With inheritance the split leaves accuracy essentially unchanged;
without it, accuracy collapses to $0$--$29\%$.}
\label{fig:split-accuracy}
\end{figure}

\subsection{Forgetting Results: Full Tables and Per-Task Curves}
\label{app:forgetting-full}

Table~\ref{tab:main-forgetting} reports average forgetting across all standard
benchmarks. EFDT and HAT (both NR) show elevated forgetting on Shuttle ($0.400$)
and HAR ($0.314$), consistent with Figure~\ref{fig:tree-size}: these methods
trigger more splits, creating smaller and less stable leaf predictors whose class
estimates degrade faster under new class arrivals. The negative forgetting entry
for \mist{}-G on Shuttle ($-0.035$) reflects a genuine accuracy improvement on
Task-0 as the GNB accumulator refines its per-class likelihood estimates over
the stream; no split occurs on Shuttle, so this is not attributable to structural
repartitioning.

Per-task accuracy decay curves for Synth-10, Covertype, and Split-MNIST are
provided in the supplementary materials. Across all three datasets, \mist{}-G
with inheritance retains earlier tasks substantially better than the
no-inheritance ablation: per-task accuracy curves decay more slowly and
stabilise at higher plateaus, with the gap proportional to $n_{\mathrm{parent}}$
at the triggering split event, consistent with Proposition~\ref{prop:variance}.

\begin{table*}[t]
\caption{Average forgetting (mean $\pm$ std over 5 seeds) on the standard
Online CIL benchmarks (lower is better). \textbf{Bold}: best mean within group.}
\label{tab:main-forgetting}
\centering
\resizebox{\textwidth}{!}{%
\setlength{\tabcolsep}{3pt}
\begin{tabular}{lccccccccccc}
\toprule
\textbf{Method}
  & Synth-10 & Synth-50 & Wine & Iris
  & Covertype & Split-MNIST & Pendigits & Shuttle & Letter & HAR & \textbf{Mean} \\
\midrule
\multicolumn{12}{l}{\textit{Tree-based (CIL-adapted)}} \\
\mist{}-G (ours)
  & 0.001$\pm$0.003 & 0.001$\pm$0.002 & 0.038$\pm$0.039 & 0.080$\pm$0.045
  & 0.143$\pm$0.017 & 0.051$\pm$0.003 & 0.059$\pm$0.016 & $-$0.035$\pm$0.040
  & 0.089$\pm$0.005 & 0.103$\pm$0.006 & \textbf{0.0530} \\
\mist{}-K (ours)
  & 0.001$\pm$0.001 & 0.000$\pm$0.000 & 0.055$\pm$0.045 & 0.100$\pm$0.035
  & 0.200$\pm$0.027 & 0.064$\pm$0.004 & 0.067$\pm$0.019 & 0.000$\pm$0.007
  & 0.066$\pm$0.003 & 0.161$\pm$0.007 & 0.0714 \\
\midrule
\vfdt{}
  & 0.001$\pm$0.003 & 0.750$\pm$0.001 & 0.038$\pm$0.035 & 0.080$\pm$0.040
  & 0.340$\pm$0.013 & 0.066$\pm$0.004 & 0.581$\pm$0.059 & 0.486$\pm$0.140
  & 0.612$\pm$0.072 & 0.600$\pm$0.084 & 0.3554 \\
Rutkowski
  & 0.000$\pm$0.000 & 0.000$\pm$0.000 & 0.038$\pm$0.035 & 0.080$\pm$0.040
  & 0.143$\pm$0.015 & 0.060$\pm$0.002 & 0.066$\pm$0.008 & 0.367$\pm$0.099
  & 0.089$\pm$0.004 & 0.102$\pm$0.006 & \textbf{0.0945} \\
EFDT (NR)
  & 0.024$\pm$0.050 & 0.250$\pm$0.092 & 0.038$\pm$0.035 & 0.080$\pm$0.040
  & 0.490$\pm$0.096 & 0.044$\pm$0.049 & 0.720$\pm$0.083 & 0.151$\pm$0.099
  & 0.300$\pm$0.187 & 0.804$\pm$0.122 & 0.2901 \\
HAT (NR)
  & $-$0.035$\pm$0.106 & 0.227$\pm$0.105 & 0.038$\pm$0.035 & 0.080$\pm$0.040
  & 0.740$\pm$0.128 & $-$0.038$\pm$0.105 & 0.423$\pm$0.035 & 0.443$\pm$0.113
  & 0.094$\pm$0.059 & 0.516$\pm$0.051 & 0.2488 \\
\midrule
\multicolumn{12}{l}{\textit{Non-tree baselines (reference)}} \\
\slda{}
  & 0.000$\pm$0.000 & 0.000$\pm$0.000 & 0.014$\pm$0.020 & 0.020$\pm$0.027
  & 0.202$\pm$0.009 & 0.074$\pm$0.005 & 0.084$\pm$0.009 & 0.210$\pm$0.041
  & 0.106$\pm$0.007 & 0.014$\pm$0.002 & 0.0724 \\
\sqda{}
  & 0.000$\pm$0.000 & 0.002$\pm$0.003 & 0.000$\pm$0.000 & 0.050$\pm$0.050
  & 0.226$\pm$0.002 & 0.026$\pm$0.005 & 0.011$\pm$0.004 & 0.020$\pm$0.010
  & 0.041$\pm$0.004 & 0.021$\pm$0.002 & \textbf{0.0397} \\
NCM
  & 0.000$\pm$0.000 & 0.000$\pm$0.000 & 0.007$\pm$0.016 & 0.150$\pm$0.061
  & 0.273$\pm$0.011 & 0.084$\pm$0.004 & 0.073$\pm$0.011 & 0.079$\pm$0.033
  & 0.103$\pm$0.004 & 0.143$\pm$0.007 & 0.0912 \\
RLS
  & $-$0.000$\pm$0.004 & 0.000$\pm$0.000 & 0.029$\pm$0.047 & 0.180$\pm$0.057
  & 0.219$\pm$0.007 & 0.109$\pm$0.006 & 0.066$\pm$0.012 & 0.428$\pm$0.025
  & 0.147$\pm$0.009 & 0.013$\pm$0.001 & 0.1191 \\
PEC
  & 0.001$\pm$0.001 & 0.017$\pm$0.012 & 0.152$\pm$0.067 & 0.090$\pm$0.049 & 0.231$\pm$0.007 & 0.042$\pm$0.005 & 0.026$\pm$0.003 & 0.004$\pm$0.008 & 0.086$\pm$0.004 & 0.099$\pm$0.041 & 0.0748 \\
\bottomrule
\end{tabular}}
\end{table*}

\subsection{Tree-Dynamics Diagnostics}
\label{app:tree-dynamics}

Figure~\ref{fig:tree-size} plots the evolution of four structural quantities
over the training stream: number of leaves, cumulative splits, weighted leaf
Gini impurity, and mean local class cardinality $K_{\mathrm{local}}$. The
comparison between the McDiarmid GapTest (default) and the Hoeffding-based
variant reveals the structural regularisation effect described in Section~4.1:
the Hoeffding criterion triggers substantially more splits as $K_{\mathrm{local}}$
grows,
yet Table~\ref{tab:ablation-core} shows only marginal aggregate accuracy
gain ($+1.02\%$); the structurally meaningful difference lies in tree
compactness (Figure~\ref{fig:tree-size}), not raw accuracy.
The lower-right panel shows reduced local class cardinality in these runs.
This can reduce storage relative to the worst-case bound in
Appendix~\ref{app:memory-full}; it does not imply that total allocated memory
decreases after each split.

\begin{figure}[t]
\centering
\includegraphics[width=0.48\textwidth]{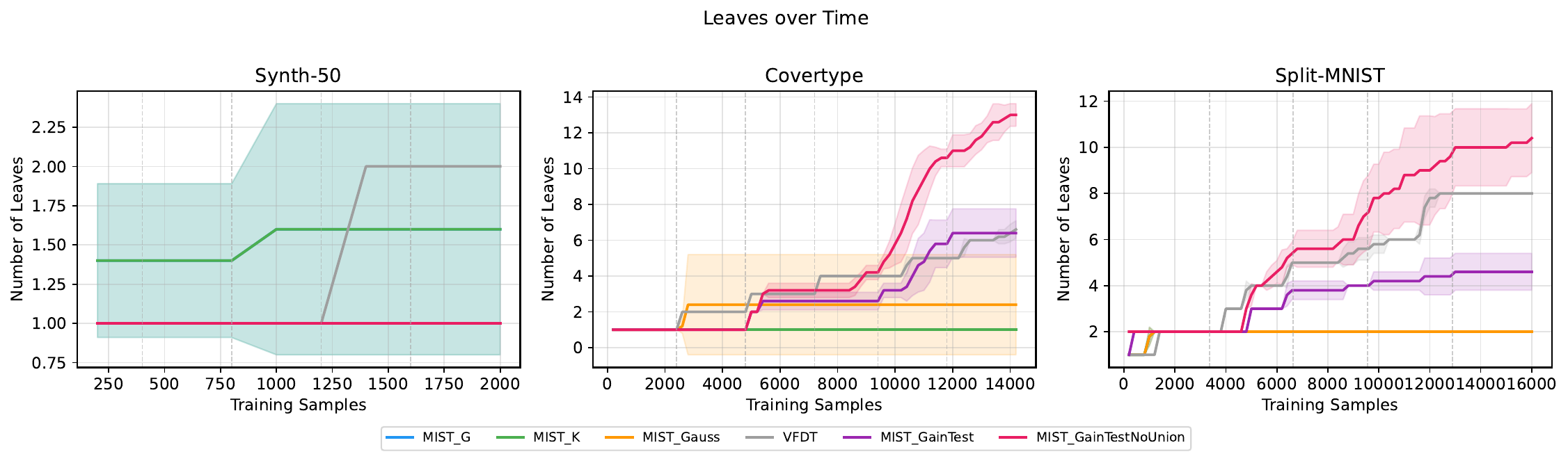}
\hfill
\includegraphics[width=0.48\textwidth]{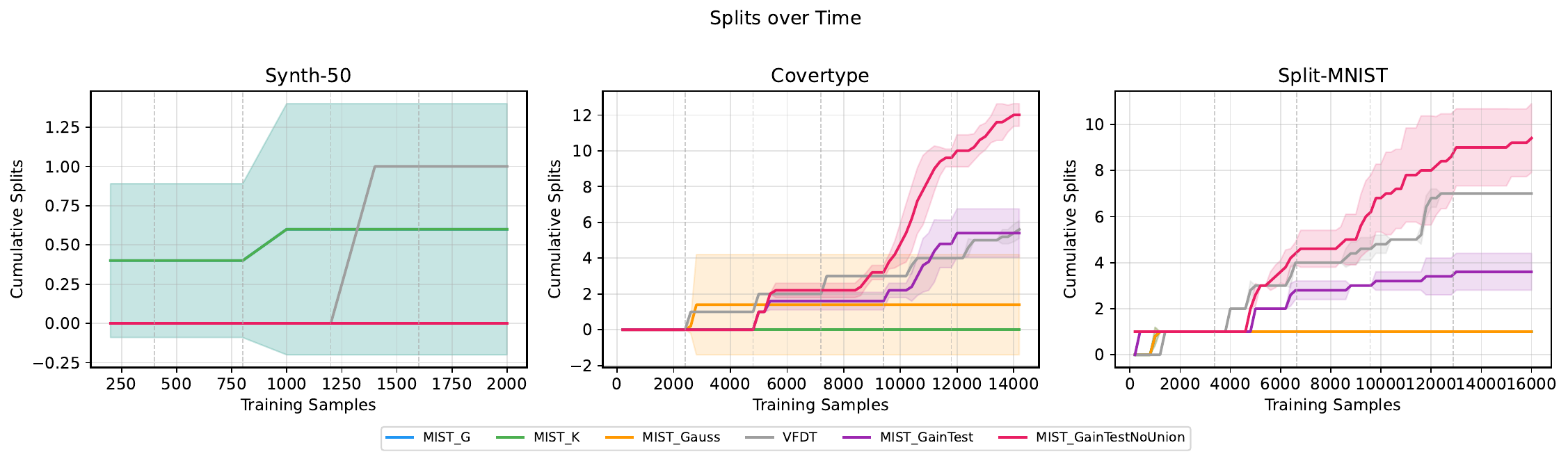}\\[4pt]
\includegraphics[width=0.48\textwidth]{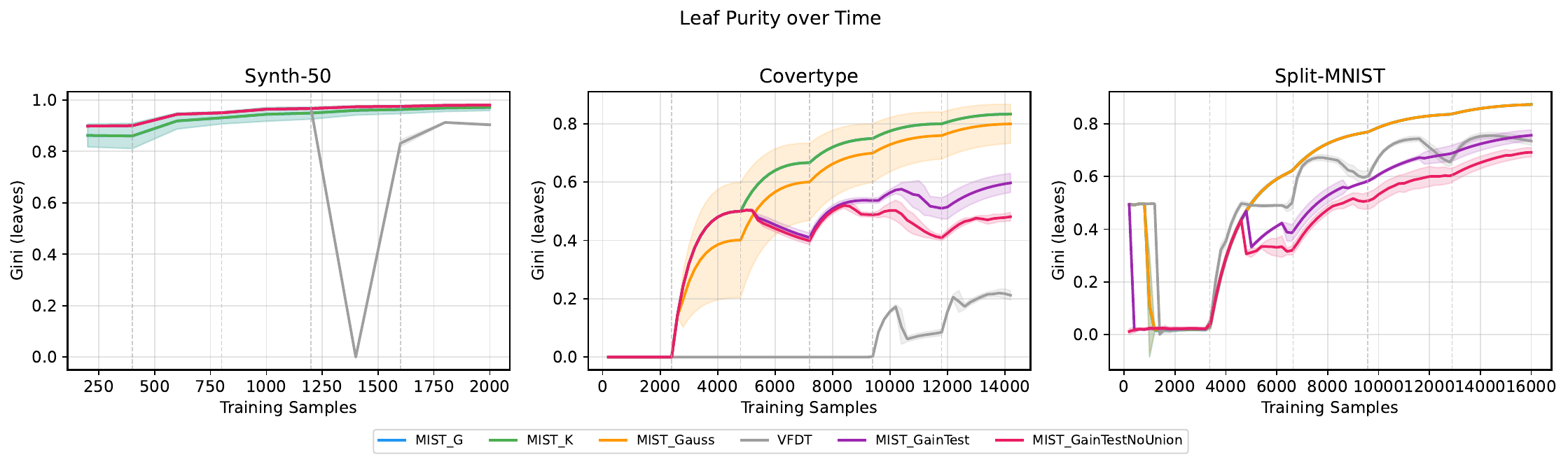}
\hfill
\includegraphics[width=0.48\textwidth]{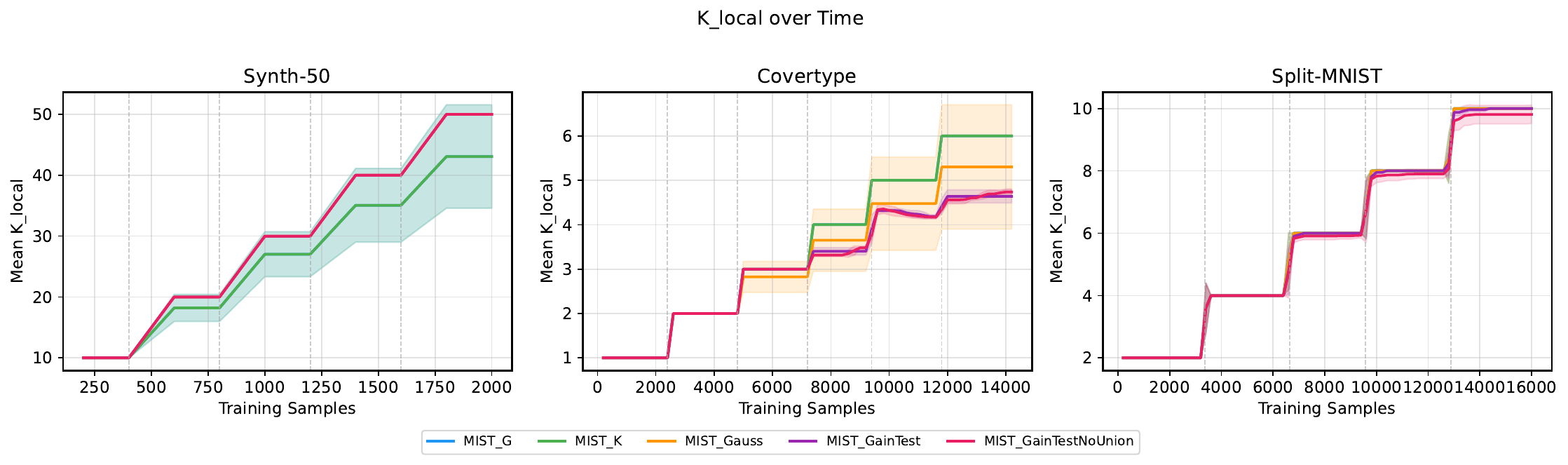}
\caption{Tree-dynamics diagnostics on Synth-50, Covertype, and Split-MNIST
         (standard, unweighted Gini criterion for all models).
\textbf{Top:} number of leaves and cumulative splits over the stream.
\textbf{Bottom:} weighted leaf Gini impurity and mean local class cardinality
$K_{\mathrm{local}}$.
Removing the union bound over attributes (\textsc{GainTestNoUnion}) causes
systematic over-splitting---up to $6\times$ more leaves than \mist{}---without
a commensurate gain in leaf purity (bottom-left), confirming that the union
bound is essential for controlling split proliferation as $K_{\mathrm{local}}$
grows.
\textsc{VFDT} (Hoeffding bound, no inheritance) shows a comparable split rate
to \textsc{GainTest} but higher residual impurity, reflecting the absence of
knowledge inheritance.
The staircase pattern in $K_{\mathrm{local}}$ (bottom-right) marks task
boundaries; \mist{}'s tight McDiarmid radius delays new splits until sufficient
evidence accumulates across all $K_{\mathrm{local}}$ classes.}
\label{fig:tree-size}
\end{figure}

Table~\ref{tab:inheritance-ablation} compares the default truncated-Gaussian
projection (Eq.~\eqref{eq:mu_left_right}) against a
sketch-quantile projection (\textsc{SketchInherit}) that initialises child
means directly from parent KLL quantiles. Both variants produce identical
accuracy and forgetting on Split-MNIST,
confirming that the post-split cold-start benefit derives from the warm-start
mechanism itself---i.e., providing statistically grounded priors to
children---rather than from the specific formula used to project those priors.
The sketch-quantile variant is $1.27\times$ slower on average; the overhead
is not amortised over the stream since splits are infrequent. The large
per-seed variance in runtime (Gaussian: ${\pm}27.4$ s; Sketch-Quantile:
${\pm}17.5$ s) reflects differences in tree depth reached across seeds rather
than algorithmic instability. The default truncated-Gaussian projection is
therefore preferred on near-Gaussian benchmarks where both variants are
equivalent in accuracy.
\begin{table}[h]
\caption{Split-time inheritance ablation on Split-MNIST (5 seeds). The non-parametric initialisation matches
accuracy but is $1.27\times$ slower on average.}
\label{tab:inheritance-ablation}
\centering
\small
\begin{tabular}{lcccc}
\toprule
Variant & Accuracy & Forgetting & Runtime (s) & Relative time \\
\midrule
Gaussian projection (default) & $0.7717\pm0.007$ & $0.0739\pm0.005$ & $69.35\pm27.39$ & $1.00\times$ \\
Sketch-quantile projection    & $0.7717\pm0.007$ & $0.0739\pm0.005$ & $88.35\pm17.48$ & $1.27\times$ \\
\bottomrule
\end{tabular}
\end{table}

\section{Ablation and Sensitivity Analysis}
\label{app:ablation}

\subsection{Sensitivity to $k$, $\alpha$, and $\delta$}
\label{app:sensitivity}

Table~\ref{tab:ablation-k-alpha} reports accuracy averaged over Synth-10,
Iris, Covertype, and Split-MNIST as each hyperparameter is swept in
isolation with all others held at default. Two plateaus are clearly
visible. For sketch capacity, performance saturates at $k \geq 32$, with negligible practical variation for $k > 64$; this is consistent with the $4/k$
decay of the sketch error bound (Section~\ref{app:sketch-bound}) entering
the negligible regime. For inheritance discount, the transition is sharp:
$\alpha{=}0.0$ collapses to $0.7439$, matching the \emph{w/o Inheritance}
ablation exactly (Table~\ref{tab:ablation-core}), while any $\alpha \geq 0.2$
recovers near-full performance. For failure probability $\delta$, the
near-constant accuracy across nearly two orders of magnitude confirms the
theoretically predicted $\sqrt{\ln(1/\delta)}$ plateau in the operational
radius (Eq.~\eqref{eq:radius}): the logarithmic dependence makes
$\delta$ a low-leverage parameter in practice.

\begin{table}[h]
\caption{Sensitivity of \mist{}-G to sketch capacity $k$, inheritance
discount $\alpha$, and split-confidence level $\delta$ (average ten standard datasets). The near-identical
values across $\delta$ reflect the theoretically predicted
$\sqrt{\ln(1/\delta)}$ plateau and are a feature, not a reporting artifact
(confirmed experimentally across $\delta \in [0.01, 0.30]$).}
\label{tab:ablation-k-alpha}
\centering
\small
\begin{tabular}{lcccccc}
\toprule
$k$ & 16 & 32 & 64 & 128 & 256 & 512 \\
Avg Acc. & 0.8219 & 0.8221 & 0.8225 & 0.8225 & 0.8260 & 0.8252 \\
\midrule
$\alpha$ & 0.0 & 0.2 & 0.4 & 0.6 & 0.8 & 1.0 \\
Avg Acc. & 0.7439 & 0.8226 & 0.8226 & 0.8226 & 0.8226 & 0.8226 \\
\midrule
$\delta$ & 0.01 & 0.05 & 0.10 & 0.15 & 0.20 & 0.30 \\
Avg Acc. & 0.8226 & 0.8226 & 0.8225 & 0.8226 & 0.8226 & 0.8226 \\
\bottomrule
\end{tabular}
\end{table}

\subsection{Threshold-Density Ablation}
\label{app:threshold}
MIST generates split-threshold candidates from inter-class quantile midpoints,
yielding at most $K_{\mathrm{local}}-1$ candidates per feature (sparse policy).
Table~\ref{tab:threshold-density} evaluates denser grids of 10 and 20
candidates.
Table~\ref{tab:threshold-density} evaluates denser grids of 10 and 20
candidates. The observed non-monotone runtime (dense grids are modestly
faster on average) therefore cannot be attributed to a reduced per-candidate
cost; it more likely reflects reduced tree-growth overhead when a larger
candidate set occasionally identifies a stronger split earlier, slightly
shortening the overall training trajectory. Runtime variance increases
substantially with grid density, however (std up to 28\,s on Covertype
vs.\ ${<}2$\,s for sparse), making this benefit unreliable.
Crucially, accuracy does not change at any density level across both
datasets. Crucially, accuracy does not change at any density level across both
datasets. This confirms that (i)~the sparse quantile-midpoint policy already
identifies the same best split as a much denser grid, because KLL quantiles
concentrate around class-conditional modes; and (ii)~the union bound over
candidate thresholds in Eq.~\eqref{eq:radius} does not conservatively expand
the radius enough to suppress valid splits.
The sparse policy therefore remains preferable for its stability
(std $<$2 s vs.\ up to 28 s for dense grids on Covertype).
\begin{table}[h]
\caption{Threshold-density ablation (mean runtime; 5 seeds). Accuracy is unchanged at all densities.
Dense grids show lower mean runtime but substantially higher variance,
confirming the sparse policy as the stable, efficient default.}
\label{tab:threshold-density}
\centering
\small
\resizebox{\textwidth}{!}{%
\begin{tabular}{lccccc}
\toprule
Dataset & Sparse (1) & Dense (10) & Dense (20) & Accuracy & Acc.\ change \\
\midrule
Covertype   & $65.40\pm1.26$ s & $53.19\pm19.12$ s ($-19\%$) & $53.09\pm28.40$ s ($-19\%$) & $0.694\pm0.006$ & none \\
Split-MNIST & $126.10\pm18.00$ s & $102.24\pm9.54$ s ($-19\%$) & $99.93\pm1.63$ s ($-21\%$) & $0.772\pm0.007$ & none \\
\bottomrule
\end{tabular}%
}
\end{table}

\subsection{\mist{}-K Parameter Sensitivity}
\label{app:mistk-sensitivity}
Table~\ref{tab:mistk-sensitivity} sweeps the two \mist{}-K-specific
hyperparameters: smoothing constant $\epsilon_s$ and bandwidth multiplier
$\beta$. Across all ten benchmarks,
accuracy is nearly invariant to $\epsilon_s$: every value in
$[10^{-5}, 10^{-1}]$ yields the same average of $0.8322$.
Note that $\epsilon_s$ appears in Eq.~\eqref{eq:mistk-likelihood} for
\emph{all} features, including numerical ones; however, on purely numerical
benchmarks the sketch-CDF differences $\hat{F}^+ - \hat{F}^-$ dominate and
the additive $\epsilon_s$ contribution is negligible, making accuracy
insensitive to its value. For nominal (categorical) features, $\epsilon_s$
acts as Laplace smoothing; since all standard benchmarks are purely
numerical, this branch is inactive and the invariance to $\epsilon_s$ is
expected.
The bandwidth multiplier $\beta$ does influence accuracy: performance peaks
at $\beta \in [1.0, 2.0]$ (avg.\ acc.\ $0.8336$--$0.8338$) and
degrades at both extremes—under-smoothing at $\beta{=}0.25$ ($0.7963$,
$-4.5\%$) collapses adjacent class-conditional modes, and over-smoothing
at $\beta{=}4.0$ ($0.8141$, $-2.3\%$) blurs interclass boundaries.
We recommend $\beta{=}1.0$ as the safe default; $\epsilon_s$ is irrelevant
for numerical-only data and can be set to any value in the standard range.

\begin{table}[h]
\caption{Sensitivity of \mist{}-K average accuracy (over 10 datasets,
5 seeds) to smoothing constant $\epsilon_s$
and bandwidth multiplier $\beta$. $\epsilon_s$ has no effect because all
benchmarks are purely numerical (KLL sketch-CDF is used for likelihood;
Laplace smoothing is inactive).}
\label{tab:mistk-sensitivity}
\centering
\small
\begin{tabular}{lccccc}
\toprule
$\epsilon_s$ & $10^{-5}$ & $10^{-4}$ & $10^{-3}$ & $10^{-2}$ & $10^{-1}$ \\
Avg Acc.     & 0.8322    & 0.8322    & 0.8322    & 0.8322    & 0.8322    \\
\midrule
$\beta$  & 0.25   & 0.50   & 1.00   & 2.00   & 4.00   \\
Avg Acc. & 0.7963 & 0.8162 & 0.8336 & 0.8338 & 0.8141 \\
\bottomrule
\end{tabular}
\end{table}

\section{Memory and Computational Analysis}
\label{app:memory-all}

\subsection{Memory Footprint, Throughput, and Trade-offs}
\label{app:memory-full}

Let $L$ be the number of leaves, $d$ the number of features, and
$K_{\mathrm{local}}^{(\ell)}$ the number of active classes at leaf $\ell$.
With sketch capacity $k$, total sketch memory is:
\[
  M_{\mathrm{sketch}}
  = O\!\left(dk\sum_{\ell=1}^{L} K_{\mathrm{local}}^{(\ell)}\right),
\]
which can be smaller than the worst-case $O(dkLK)$ when
$K_{\mathrm{local}}^{(\ell)} \ll K$ once leaves become purer
(Appendix~\ref{app:memory-scaling}).

\paragraph{Structural storage versus allocation peaks.}
Before splitting, the root can represent all classes, costing $O(dkK_t)$.
A split can increase total storage; it does not guarantee that the sum of
local class counts decreases. Appendix~\ref{app:rebuttal-memory} reports
retained statistics and Python-traced allocation peaks on Letter and HAR.
These measurements do not support a universal factor-two peak-memory bound.
Table~\ref{tab:dim-scaling} isolates the effect of dimension: raising $d$ from
20 to 100 grows \mist{}-K storage by about $5\times$ and \sqda{} storage by
$24.3\times$.

\paragraph{Throughput and Pareto frontier.}
Table~\ref{tab:memory} characterises memory and throughput on two
representative datasets that bracket the practical operating range. Among
Pareto-optimal configurations on Covertype, NCM provides the fastest and
lightest option, \slda{} provides a fast and accurate parametric baseline,
and \mist{}-K achieves the highest single-method accuracy at the cost of
greater memory and lower throughput (due to per-prediction quantile lookups
in the sketch). \mist{}-G offers a middle ground: sketch memory identical to
\mist{}-K but significantly higher throughput ($674$ vs $43$ samples/s on
Covertype), since split evaluation uses sketches but prediction uses the
faster Gaussian accumulators. EFDT (D) and HAT (D) are included for
completeness; their near-random accuracy reflects protocol misalignment
(Appendix~\ref{app:noreset}) and their throughput numbers are not meaningful
in this context.

\begin{table}[t]
\caption{Memory footprint (KB) and throughput (samples/s)
on two representative datasets: Covertype ($K=6$, $d=54$, low-$K$ regime)
and Split-MNIST ($K=10$, $d=50$ after PCA, higher-$K$ regime).
\textbf{Pareto-optimal} configurations on Covertype are marked with $\star$:
NCM (fastest/lightest), \slda{} (fast with high accuracy), and \mist{}-K
(highest accuracy). EFDT (D) and HAT (D) --- where (D) denotes the default
drift-reset configuration --- are included for completeness;
their near-random accuracy reflects protocol misalignment
(Appendix~\ref{app:noreset}), not resource efficiency.
Memory is the mean over 3 seeds, except the \mist{} rows ($\S$), which use the
common structural estimator and the five seeds of
Table~\ref{tab:memory-common} (Appendix~\ref{app:rebuttal-memory}); the original
\mist{} entries in this table could not be reproduced and were replaced.
Throughput is from the original run for all rows.}
\label{tab:memory}
\centering
\small
\setlength{\tabcolsep}{3pt}
\begin{tabular}{lccccl}
\toprule
Method
  & \multicolumn{2}{c}{Memory (KB)}
  & \multicolumn{2}{c}{Throughput (samp./s)}
  & Note \\
\cmidrule(lr){2-3}\cmidrule(lr){4-5}
  & Covertype & Split-MNIST & Covertype & Split-MNIST & \\
\midrule
\mist{}-G$^{\S}$         &   587.0 &  1171.5 &    674 &    401 & \\
\mist{}-K$^{\S}$ $\star$ &   587.0 &  1171.5 &     43 &     37 & highest acc \\
\vfdt{}           &   305.5 &   591.1 &   1179 &    682 & \\
\arf{}-NR         &  3044.5 &  5912.3 &   1104 &    634 & \\
EFDT (D)          &  4335.4 &      -- &     11 &     -- & near-random \\
HAT (D)           &   703.4 &      -- &    244 &     -- & near-random \\
\slda{} $\star$   &   188.5 &   243.8 &   2543 &   2123 & fast+accurate \\
\sqda{}           &   416.8 &   596.1 &   1560 &   1161 & \\
NCM $\star$       &     5.8 &     8.8 & 46778  & 31815  & ultra-fast \\
\bottomrule
\end{tabular}
\end{table}

\subsection{Local Class Cardinality on Random-Means Streams}
\label{app:memory-scaling}

To examine local class sparsity empirically, we evaluate \mist{}-G
on multi-dimensional random-means streams ($d{=}10$, $\sigma{=}0.3$,
$K \in \{10\ldots150\}$, 2 seeds). Class means are drawn uniformly from
$[0,5]^{10}$, so splits distribute classes across multiple leaves
naturally without artificial pre-specification of the partition structure.

Table~\ref{tab:klocal} shows that $K_{\mathrm{local}}$ falls 64--82\% below
$K_{\mathrm{global}}$ across all tested cardinalities (18--36\% of
$K_{\mathrm{global}}$ on streams whose trees grow to 17.5--33 leaves), consistent
with the lower-right panel of Figure~\ref{fig:tree-size} on real benchmarks.
At $K{=}150$, fewer than 30 classes remain active per leaf on average despite
the global set containing 150 classes, showing that local class sparsity reduces the storage estimate relative to
the $O(dkLK)$ worst case on these random-means streams.

\begin{table}[h]
\centering
\caption{$K_{\mathrm{local}}$ mean vs.\ $K_{\mathrm{global}}$ on
random-means streams ($d{=}10$, $\sigma{=}0.3$, 2 seeds). The
reported local class counts are below the global count in these runs.}
\label{tab:klocal}
\begin{tabular}{rrrr}
\toprule
$K$ & Leaves & $K_{\mathrm{local}}$ mean & Reduction \\
\midrule
25  & 20.0 &  9.0 & 64\% \\
50  & 17.5 & 13.7 & 73\% \\
75  & 33.0 & 15.5 & 79\% \\
100 & 20.5 & 29.2 & 71\% \\
150 & 21.0 & 27.1 & 82\% \\
\bottomrule
\end{tabular}
\end{table}

\section{Empirical Validation of the Theory}
\label{app:theory-validation}

\subsection{Empirical Validation of the $4\varepsilon_{\mathrm{sketch}}$ Bound}
\label{app:sketch-bound}

We measured $|\hat{F} - F|$ (sketched vs.\ exact Gini gain) at every split
decision under the strict Online CIL protocol across 7 datasets, 4 sketch
capacities ($k \in \{16, 32, 64, 128\}$), and 5 seeds. All results use the
corrected CIL class-incremental presentation and standard (unweighted) Gini
gain, consistent with Theorem~\ref{thm:tight}.

\paragraph{Per-dataset results at $k{=}64$.}

Table~\ref{tab:sketch-bound-k64} reports empirical sketch error across all
validated datasets at $k{=}64$, giving $4\varepsilon = 4/64 = 0.0625$.
The bound holds at every recorded split event. Safety margins range from
$2.3\times$ (Pendigits, the most marginal case) to $27.9\times$ (HAR, where
class-conditional distributions are approximately isotropic and the sketch
is highly accurate). Covertype exceeds the KLL rank-error threshold
($\varepsilon_{\mathrm{sketch}} = 1/k$) as a CDF \emph{value} error, but
this does not violate the bound: Theorem~\ref{thm:composed} controls Gini
gain error via rank error, not CDF value error directly, and
Gini err max $= 0.015 \ll 0.0625$ confirms the bound holds.

\begin{table}[h]
\centering
\caption{Empirical sketch error at $k{=}64$ ($4\varepsilon = 0.0625$).
         \emph{Events}: split decisions recorded (5 seeds, $\times$5).
         \emph{Ratio}: $4\varepsilon \,/\,$ Gini err max (worst-case
         safety margin).
         Letter is excluded; see note below.}
\label{tab:sketch-bound-k64}
\resizebox{\linewidth}{!}{%
\begin{tabular}{lrrrrrrr}
\toprule
Dataset & Events
        & CDF err mean & CDF err max
        & Gini err mean & Gini err max
        & $4\varepsilon$ & Ratio \\
\midrule
Pareto               & 20 & 0.0031 & 0.0067 & 0.00498 & 0.01126 & 0.0625 & $5.5\times$  \\
LogNormal            & 20 & 0.0080 & 0.0145 & 0.01056 & 0.02955 & 0.0625 & $2.1\times$  \\
Exponential          & 20 & 0.0047 & 0.0101 & 0.00523 & 0.01226 & 0.0625 & $5.1\times$  \\
Outlier 5\%          & 20 & 0.0068 & 0.0122 & 0.00914 & 0.02318 & 0.0625 & $2.7\times$  \\
Pendigits            &  4 & 0.0073 & 0.0145 & 0.01331 & 0.02715 & 0.0625 & $2.3\times$  \\
HAR                  &  5 & 0.0009 & 0.0014 & 0.00114 & 0.00224 & 0.0625 & $27.9\times$ \\
Shuttle              & -- & 0.0001 & 0.0001 & --      & --      & 0.0625 & --           \\
Covertype$^\ddagger$ & 12 & 0.0248 & 0.0585 & 0.00487 & 0.01486 & 0.0625 & $4.2\times$  \\
\bottomrule
\end{tabular}}
\smallskip
\begin{minipage}{\linewidth}
\small
$^\ddagger$\,CDF err max $= 0.059$ on Covertype exceeds $1/k = 0.0156$
as a CDF \emph{value} error, but this does not violate the theoretical
bound: $4\varepsilon_{\mathrm{sketch}}$ in Theorem~\ref{thm:composed}
bounds Gini gain error via \emph{rank} error, not CDF value error.
Gini err max $= 0.015 \ll 0.0625$ confirms the bound holds.
\end{minipage}
\end{table}

\paragraph{Note on zero-split datasets.}
Letter produced no split events at any $k$ under the default
\texttt{grace\_period}$=200$, consistent with the main experiment
(Table~\ref{tab:sketch-vs-gaussian-real}).
Shuttle produced no split events under the unweighted-Gini validation
protocol: its heavy class imbalance (${\approx}80\%$ majority class)
causes the Gini gain gap to fall below the McDiarmid radius, so the
bound is vacuously satisfied.

\paragraph{Scaling with sketch capacity $k$.}

Table~\ref{tab:sketch-bound-kscale} reports safety margins across sketch
capacities for the four synthetic validation datasets, where sufficient split
events (20 per configuration) are available at all $k$ values.
The bound $4\varepsilon_{\mathrm{sketch}} \leq 4/k$ holds at every recorded
split event across all configurations. Margins degrade gracefully as $k$
decreases, consistent with the $O(1/k)$ rank-error guarantee of KLL: even at $k{=}16$ (the smallest tested), the minimum margin remains above
$2\times$. At $k{=}64$, worst-case margins range from $2.1\times$
(LogNormal) to $5.5\times$ (Pareto), confirming the bound holds with
meaningful slack across all tested configurations.

\begin{table}[h]
\centering
\caption{Ratio $4\varepsilon \,/\,$ Gini err max across $k$ values,
synthetic datasets only (20 split events each).
The $4\varepsilon_{\mathrm{sketch}}$ bound holds at every recorded event;
safety margin degrades gracefully as $k$ decreases but remains $\geq 2.1\times$
at $k{=}64$.}
\label{tab:sketch-bound-kscale}
\begin{tabular}{rrrrrrr}
\toprule
$k$ & $4\varepsilon$ & Pareto & LogNormal & Exponential & Outlier 5\% \\
\midrule
16  & 0.250  & $14\times$ & $20\times$ & $20\times$ & $9\times$  \\
32  & 0.125  & $14\times$ & $12\times$ & $25\times$ & $14\times$ \\
64  & 0.0625 & $5.5\times$ & $2.1\times$ & $5.1\times$ & $2.7\times$ \\
128 & 0.0312 & $13\times$ & $8\times$  & $16\times$ & $10\times$ \\
\bottomrule
\end{tabular}
\end{table}

\subsection{Sketch-Quantile vs.\ Truncated-Gaussian Inheritance on Non-Gaussian Streams}
\label{app:sketch-vs-gaussian}

This appendix evaluates the robustness of the truncated-Gaussian projection
used in Bayesian inheritance (Section~4.2) when the underlying class-conditional
distributions violate the Gaussianity assumption. We compare \mist{}-G against
\textsc{SketchInherit}, a variant that initialises child means directly from
parent KLL quantiles: the left-child mean for class $c$ on the split feature
is set to $\widehat{F}_{j,c}^{-1}(r_L/2)$, and the right-child mean to
$\widehat{F}_{j,c}^{-1}(r_L + r_R/2)$, where $r_L$ and $r_R$ are the
estimated mass fractions routing to each child. Quantiles are resolved by
nearest-point lookup with boundary clipping to $[0.01, 0.99]$. Non-split
feature parameters and class mass are inherited identically to \mist{}-G.
All other components are identical, including $\alpha{=}0.6$. Results
are averaged over 5 seeds $\{42, 123, 456, 789, 3330\}$ with
\texttt{grace\_period}$=200$, $\delta=0.10$, $k=64$.

\paragraph{Synthetic non-Gaussian streams.}
Four distributions isolate specific failure modes of the truncated-Gaussian
projection. \textbf{Pareto}: power-law tails ($\alpha{=}1.5$) with
\texttt{log1p} normalisation; tests whether log-transform approximately
Gaussianises the truncated region. \textbf{LogNormal}: log-normal marginals
($\sigma \in [0.8, 1.5]$); right-skewed distributions where
$\mathbb{E}[X \mid X \leq t]$ lies below the Gaussian prediction.
\textbf{Outlier 5\%}: Gaussian base with $5\%$ extreme outliers
(${\times}50\sigma$); tests robustness of Welford variance estimates under
contamination. \textbf{Exponential}: exponential marginals
($\lambda \in [0.5, 2]$); moderately right-skewed, between LogNormal and
Gaussian in severity. Each dataset contains $K{=}6$ classes, $d{=}10$
features, $n{=}12{,}000$ training samples ($2{,}000$ per class), presented
in strict Online CIL order.

Table~\ref{tab:sketch-vs-gaussian-synthetic} shows that gains are largest
where the Gaussian truncation assumption fails most severely: $+33.5$ pp on
Outlier~5\% (a single outlier corrupts the Welford variance estimate, making
the moment projection unreliable, while the sketch CDF estimate remains
robust) and $+2.3$ pp on LogNormal. The Pareto case ($-1.2$ pp) is the sole
exception: \texttt{log1p} normalisation approximately symmetrises the
truncated region, making the Gaussian approximation adequate while the
sketch introduces marginal quantile-grid variance.

\begin{table}[h]
\centering
\caption{Accuracy comparison on synthetic non-Gaussian streams
         (mean $\pm$ std over 5 seeds).
         $\Delta\mathrm{acc} = $ \textsc{SketchInherit} $-$ \mist{}-G.}
\label{tab:sketch-vs-gaussian-synthetic}
\resizebox{\linewidth}{!}{%
\begin{tabular}{lllrrr}
\toprule
Dataset & Distribution & \mist{}-G & \textsc{SketchInherit} & $\Delta$acc & Splits \\
\midrule
Pareto      & Power-law ($\alpha{=}1.5$), \texttt{log1p} transform
            & $0.9626\pm0.003$ & $0.9503\pm0.004$ & $-0.012$ & 5.0 \\
LogNormal   & Log-normal ($\sigma \in [0.8, 1.5]$)
            & $0.8049\pm0.017$ & $0.8277\pm0.011$ & $+0.023$ & 5.0 \\
Outlier 5\% & Gaussian $+5\%$ outliers (${\times}50\sigma$)
            & $0.3577\pm0.026$ & $0.6928\pm0.039$ & $\mathbf{+0.335}$ & 5.0 \\
Exponential & Exponential ($\lambda \in [0.5, 2]$)
            & $0.8676\pm0.025$ & $0.8809\pm0.031$ & $+0.013$ & 4.4 \\
\bottomrule
\end{tabular}}
\end{table}

\paragraph{Real-world benchmark datasets.}
Table~\ref{tab:sketch-vs-gaussian-real} confirms that on near-Gaussian data
(Pendigits, Split-MNIST, Covertype, HAR), both variants produce identical
accuracy: the truncated-Gaussian projection is harmless when valid.
Shuttle produced no split events under either variant (Table~\ref{tab:sketch-vs-gaussian-real}),
so the bound is vacuously satisfied and no accuracy comparison is possible.
Letter
produces no splits under the default settings and is included for completeness.

\textsc{SketchInherit} is $1.27\times$ slower on average
(Table~\ref{tab:inheritance-ablation}, Relative time column),
confirming that the cold-start benefit derives from the warm-start mechanism
rather than the projection formula. The default \mist{}-G uses
truncated-Gaussian projection on near-Gaussian benchmarks where it is
equivalent at lower computational cost.

\begin{table}[h]
\centering
\caption{Accuracy comparison on standard Online CIL benchmarks
         (mean $\pm$ std over 5 seeds).
         $K$: number of classes; $d$: feature dimension.}
\label{tab:sketch-vs-gaussian-real}
\begin{tabular}{llrrrr r}
\toprule
Dataset & $K$ & $d$ & \mist{}-G & \textsc{SketchInherit} & $\Delta$acc & Splits \\
\midrule
Pendigits   & 10 & 16  & $0.8886\pm0.023$ & $0.8886\pm0.023$ & $0.000$  & 1.4 \\
Shuttle     & 7  & 9   & $0.7331\pm0.033$ & $0.7331\pm0.054$ & $0.000$  & 0.0 \\
Letter      & 26 & 16  & $0.6775\pm0.006$ & $0.6775\pm0.006$ & $0.000$  & 0.0 \\
HAR         & 6  & 561 & $0.6074\pm0.013$ & $0.6074\pm0.013$ & $0.000$  & 1.0 \\
Covertype   & 7  & 54  & $0.5590\pm0.013$ & $0.5590\pm0.013$ & $0.000$  & 0.0 \\
Split-MNIST & 10 & 50  & $0.8629\pm0.003$ & $0.8629\pm0.002$ & $0.000$  & 1.0 \\
\bottomrule
\end{tabular}
\end{table}

\section{Stress-Test Dataset Descriptions}
\label{app:stress}

All stress-test datasets are evaluated under the same Online CIL protocol
as the main experiments (Section~\ref{sec:setup}) with a fixed random seed; generation
code is provided in the supplementary materials.
Figure~\ref{fig:stress-vis} shows 3D PCA projections of all eight
distributions.

\paragraph{Multimodal ($K{=}8$, $d{=}10$, 4 tasks).}
Each class is drawn from a mixture of $M{=}3$ well-separated Gaussian modes.
Tests whether leaf predictors can model non-unimodal class-conditional
distributions. \mist{}-K's sketch-based marginal estimator captures the
multimodal structure without raw-data retention, while GNB-leaf methods
incur systematic bias.

\paragraph{Skewed ($K{=}8$, $d{=}10$, 4 tasks).}
Each class has a dominant mode (90\%) and a spatially separated minority
mode (10\%). The skew places the class-conditional mean outside the dominant
mode's mass, exposing the truncated-Gaussian inheritance assumption to bias.

\paragraph{AngularSectors ($K{=}4$, $d{=}8$, 2 tasks).}
Samples lie on a 2D ring; class identity is determined by angular sector.
All $K$ classes share mean $\approx\mathbf{0}$ and identical covariance, so
QDA, LDA, and Naive Bayes degrade to chance ($1/K$). A tree with
axis-aligned splits on the two signal features can progressively narrow
sectors. \mist{}-K achieves the highest reported accuracy on this stream
(Table~\ref{tab:stress}); this result does not imply robustness on every
non-Gaussian geometry.

\paragraph{Antipodal ($K{=}8$, $d{=}10$, 4 tasks).}
Each class consists of two diametrically opposite clusters on a
hypersphere. \slda{} fails ($0.171$) because a shared covariance hyperplane
cannot separate antipodal classes; \sqda{} succeeds because per-class
Gaussian means differ.

\paragraph{Ring ($K{=}6$, $d{=}8$, 3 tasks).}
Classes are distributed on concentric hypersphere shells. The radial
direction is near-Gaussian, so most methods perform competitively. This
confirms that \mist{}'s advantage is regime-dependent: where the Gaussian
assumption is adequate, global parametric methods remain competitive.

\paragraph{HeavyTail ($K{=}8$, $d{=}10$, 4 tasks).}
Class-conditional distributions follow Student-$t$ ($\nu{=}2$) noise around
random centers. Fat tails inflate Welford variance estimates, biasing
\mist{}-G's truncated-Gaussian projection; \mist{}-K mitigates this via the
empirical sketch-based CDF.

\paragraph{NoisyFeatures ($K{=}8$, $d{=}20$, 4 tasks).}
Only $d_{\mathrm{info}}{=}2$ dimensions carry signal; the remaining 18 are
independent noise. Tests split-selection robustness under high irrelevant-
feature contamination. The conservative McDiarmid GapTest is particularly
relevant: uninformative features produce near-zero Gini gain, and a tight
confidence radius suppresses spurious splits on noise features.

\paragraph{ConceptDrift ($K{=}8$, $d{=}10$, 4 tasks).}
Class-conditional means shift linearly across tasks (drift $= 1.5$ per task
along feature~0). Included to verify that \mist{}'s conservative split
criterion does not over-commit to stale boundaries. High accuracy across all
methods ($> 0.988$) confirms that gradual drift is absorbed by the existing
split mechanism without explicit drift detection.

\begin{figure}[t]
  \centering
  \begin{subfigure}[t]{0.24\textwidth}
    \includegraphics[width=\textwidth]{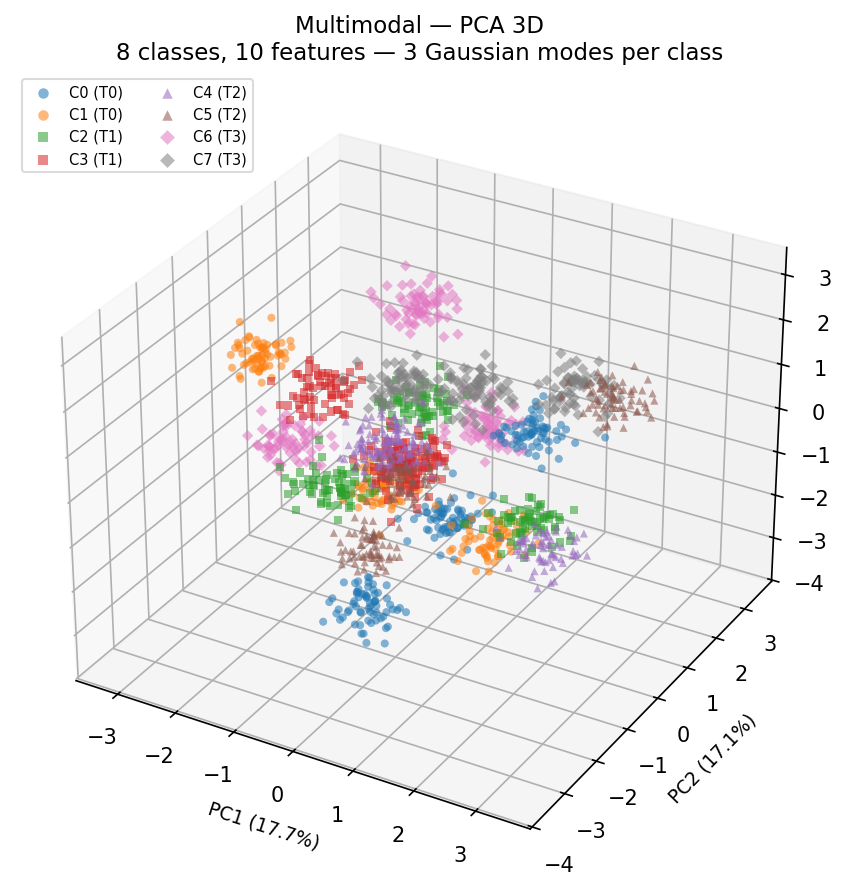}
    \caption{Multimodal}
  \end{subfigure}
  \hfill
  \begin{subfigure}[t]{0.24\textwidth}
    \includegraphics[width=\textwidth]{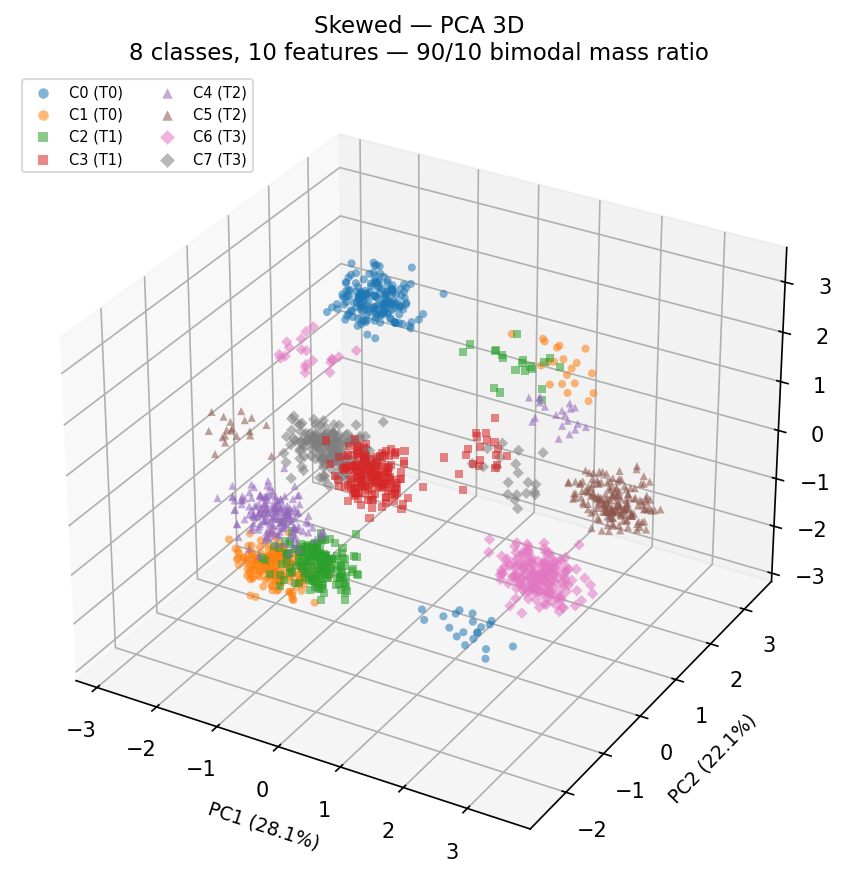}
    \caption{Skewed}
  \end{subfigure}
  \hfill
  \begin{subfigure}[t]{0.24\textwidth}
    \includegraphics[width=\textwidth]{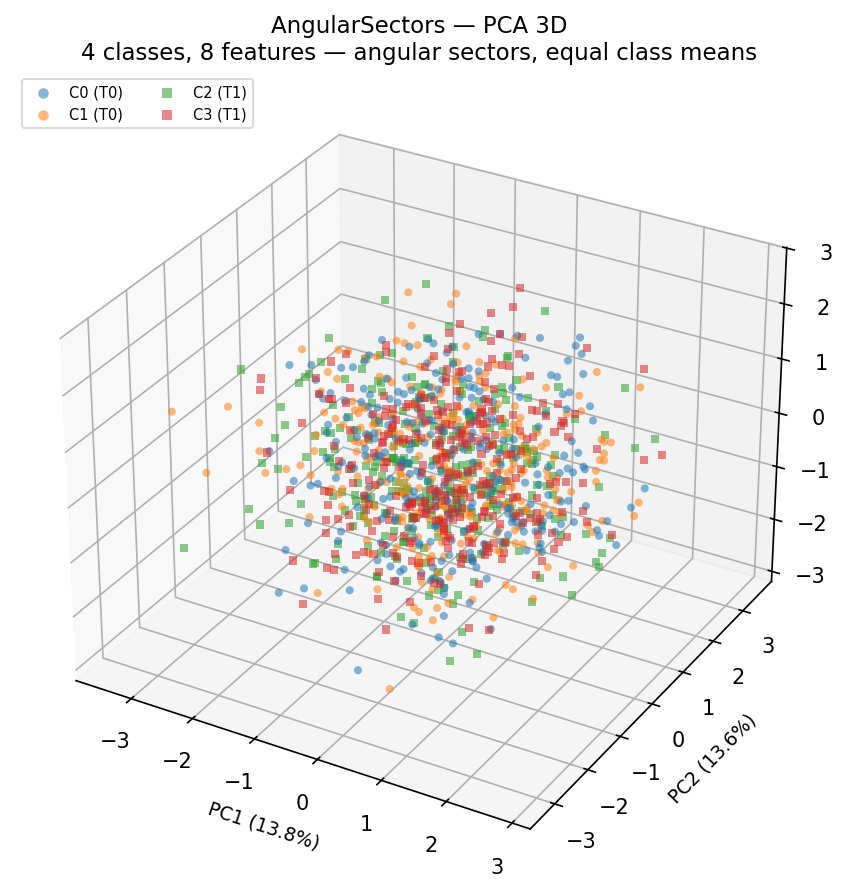}
    \caption{AngularSectors}
  \end{subfigure}
  \hfill
  \begin{subfigure}[t]{0.24\textwidth}
    \includegraphics[width=\textwidth]{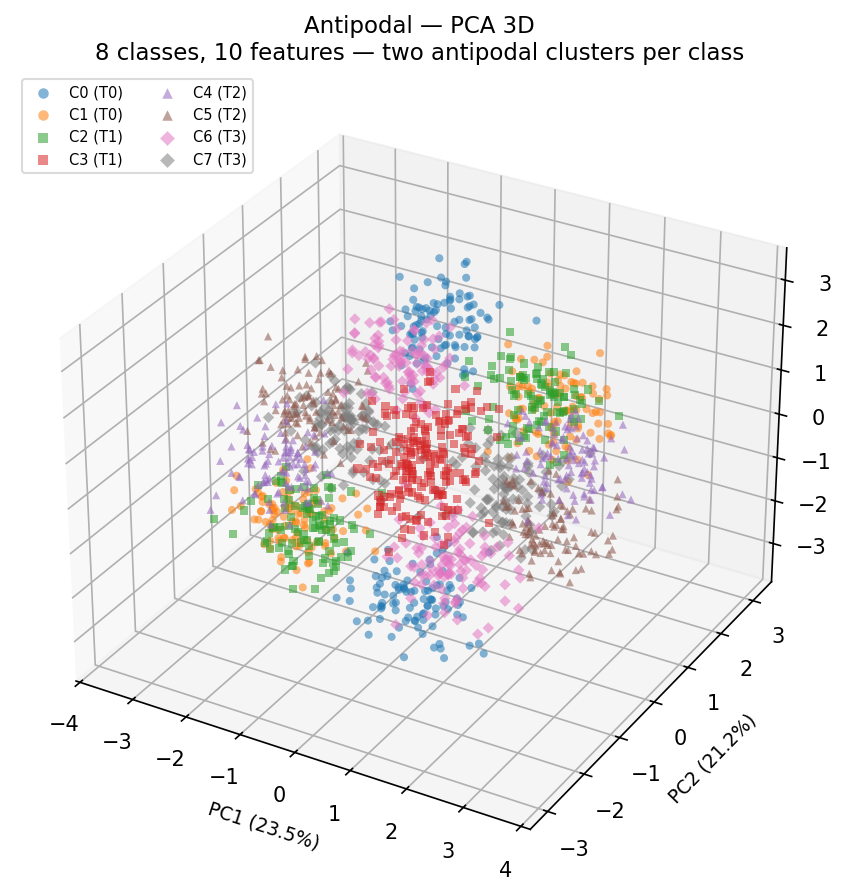}
    \caption{Antipodal}
  \end{subfigure}
  \\[6pt]
  \begin{subfigure}[t]{0.24\textwidth}
    \includegraphics[width=\textwidth]{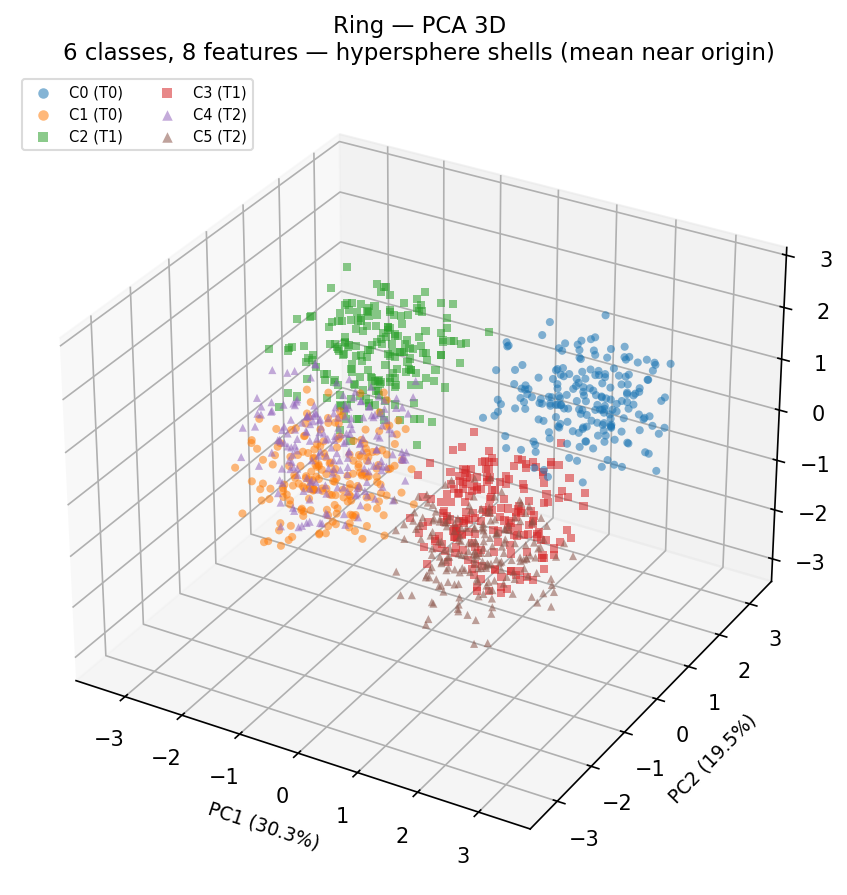}
    \caption{Ring}
  \end{subfigure}
  \hfill
  \begin{subfigure}[t]{0.24\textwidth}
    \includegraphics[width=\textwidth]{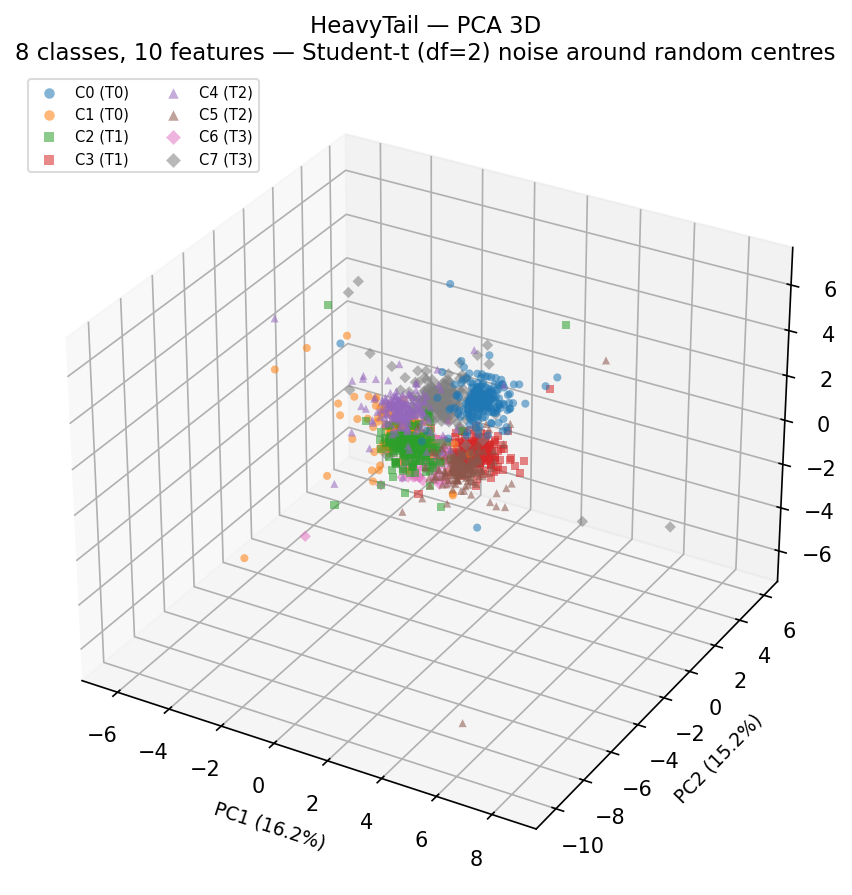}
    \caption{HeavyTail}
  \end{subfigure}
  \hfill
  \begin{subfigure}[t]{0.24\textwidth}
    \includegraphics[width=\textwidth]{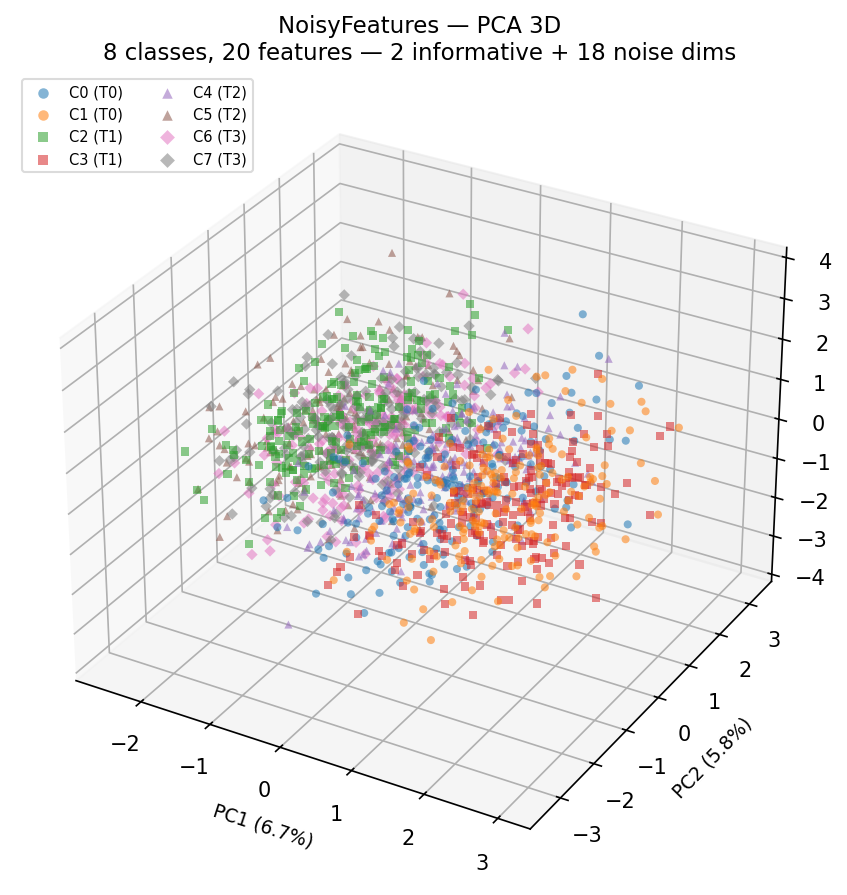}
    \caption{NoisyFeature}
  \end{subfigure}
  \hfill
  \begin{subfigure}[t]{0.24\textwidth}
    \includegraphics[width=\textwidth]{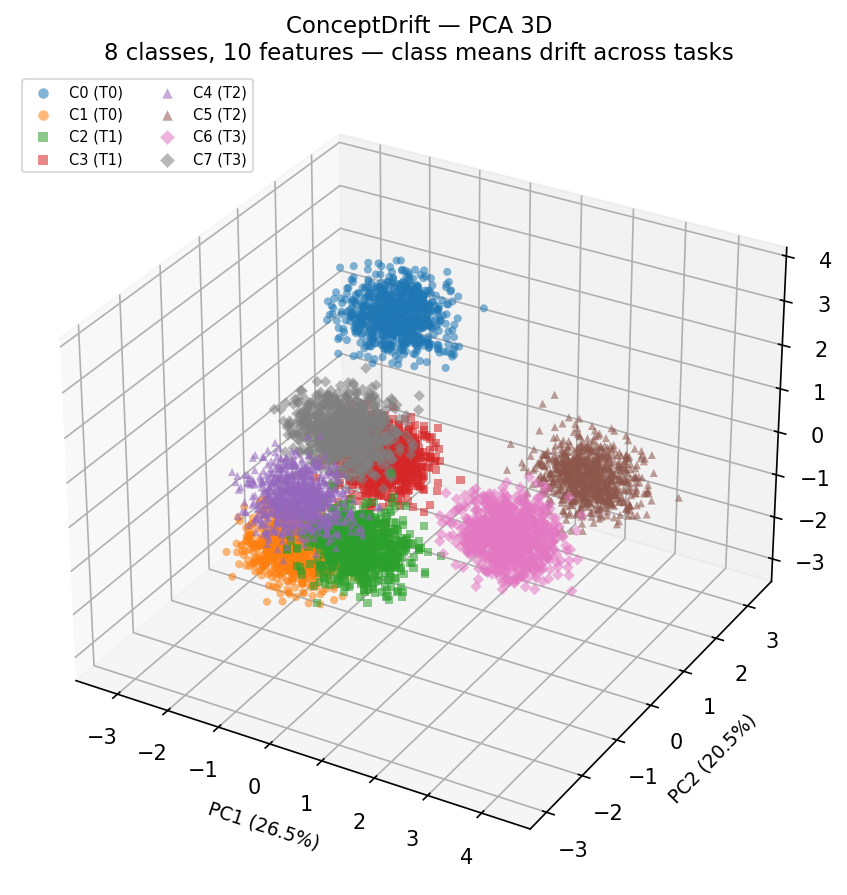}
    \caption{ConceptDrift}
  \end{subfigure}
  \caption{3D PCA visualisations of the eight stress-test streams (seed~42;
  all generated samples shown, standardised and projected onto the top three
  principal components; colour = class, marker shape = task). Each panel
  illustrates the geometric property that motivates the dataset: non-unimodal
  class structure (Multimodal, Skewed), manifold geometry (AngularSectors,
  Ring), cluster arrangement (Antipodal), and distributional properties
  (HeavyTail, NoisyFeature, ConceptDrift).}
  \label{fig:stress-vis}
\end{figure}

\section{Mathematical Proofs}
\label{app:proofs}

Table~\ref{tab:assumptions} summarises the assumptions underlying each
theoretical result. The bounded-difference property is deterministic;
applying the stated McDiarmid inequality additionally requires independent
samples. Identical distributions are not required by that inequality.
This sample-independence assumption is distinct from the Naive Bayes
feature factorisation used for leaf prediction, which does not enter the
Gini sensitivity calculation. Adaptive routing and repeated split testing
require care; no general temporal-dependence or time-uniform guarantee is
claimed here.

\begin{table}[h]
\centering
\caption{Assumptions per theoretical result.}
\label{tab:assumptions}
\renewcommand{\arraystretch}{1.25}
\begin{tabular}{>{\raggedright\arraybackslash}p{2.6cm}>{\raggedright\arraybackslash}p{4.0cm}>{\raggedright\arraybackslash}p{5.5cm}}
\toprule
Result & Assumption & Note \\
\midrule
Theorem~\ref{thm:tight}
  & Bounded differences per leaf
  & The sensitivity statement is deterministic. Its use in McDiarmid
    concentration additionally requires independent samples \\
Theorem~\ref{thm:composed} (statistical term)
  & Bounded differences; independent samples within a leaf
  & Concentrates $F$ around $\mathbb{E}[F]$ over the empirical leaf
    distribution. Independence, not identical distribution, is assumed, as in
    the per-leaf protocol of~\cite{DomingosHulten2000} \\
Theorem~\ref{thm:composed} (sketch term)
  & KLL rank guarantee
  & Holds w.p.\ $\geq 1-\delta_{\mathrm{sk}}$ \\
Corollary~\ref{cor:opt-displacement}
  & Gaussian $X_j \mid c$; finite condition number $\kappa$
  & Gaussianity only converts rank error into threshold displacement; holds
    w.p.\ $\geq 1-\delta-\delta_{\mathrm{sk}}$ \\
Proposition~\ref{prop:variance}
  & (i) Dirichlet prior; (ii) NIG prior with $\alpha n_c^{\mathrm{parent}} > 2$
  & Standard conjugate models; the variance in (ii) is finite only above this
    inherited count \\
\bottomrule
\end{tabular}
\end{table}

\subsection{Truncated-Gaussian Moment Updates}
\label{app:truncgauss}

When $X \sim \mathcal{N}(\mu, \sigma^2)$ is conditioned on $X \leq v$
(left child), let $\zeta = (v-\mu)/\sigma$. The exact conditional
moments are:
\begin{align}
  \mathbb{E}[X \mid X \leq v]
    &= \mu - \sigma\frac{\phi(\zeta)}{\Phi(\zeta)}, \label{eq:trunc-mean-left}\\
  \mathrm{Var}(X \mid X \leq v)
    &= \sigma^2\!\left[1 - \frac{\zeta\phi(\zeta)}{\Phi(\zeta)}
      - \left(\frac{\phi(\zeta)}{\Phi(\zeta)}\right)^{\!2}\right].
      \label{eq:trunc-var-left}
\end{align}

For the right child ($X \geq v$), replacing $\Phi(\zeta)$ with
$1{-}\Phi(\zeta)$ throughout gives:
\begin{align}
  \mathbb{E}[X \mid X \geq v]
    &= \mu + \sigma\frac{\phi(\zeta)}{1-\Phi(\zeta)}, \label{eq:trunc-mean-right}\\
  \mathrm{Var}(X \mid X \geq v)
    &= \sigma^2\!\left[1 \mathbf{+} \frac{\zeta\phi(\zeta)}{1-\Phi(\zeta)}
      - \left(\frac{\phi(\zeta)}{1-\Phi(\zeta)}\right)^{\!2}\right].
      \label{eq:trunc-var-right}
\end{align}

\noindent\textbf{Critical sign difference.}
The linear term $\zeta\phi(\zeta)/\tilde\Phi$ carries a \emph{negative}
sign for the left child (Eq.~\eqref{eq:trunc-var-left}) and a
\emph{positive} sign for the right child (Eq.~\eqref{eq:trunc-var-right}).
Algorithm~\ref{alg:split} Lines~7--10 apply this conditional sign.
Using the same sign for both children would yield the incorrect formula
$[1 - \zeta\phi(\zeta)/(1{-}\Phi(\zeta)) - \cdots]$ for the right child,
which can produce negative variances: at $\zeta=0$ both versions of the bracket
equal $1 - 2/\pi \approx 0.36 > 0$, but at $\zeta=1$ the wrong-sign formula
gives $\approx -2.85$ while the correct right-child bracket is $\approx 0.20$.

Both conditional variances satisfy
$\mathrm{Var}(X \mid X \gtrless v) \leq \sigma^2$,
confirming that the inheritance projection cannot inflate
feature-parameter uncertainty. For categorical features, no moment
projection is applied; see Section~\ref{sec:inheritance} for the
empirical co-occurrence formula.

\subsection{Proof of Theorem~\ref{thm:tight} (Tightness of the Gini Sensitivity)}
\label{app:proof-tight}

We first state two supporting lemmas and then give the three-step proof.

\begin{lemma}[Gini gain is Lipschitz in threshold]
\label{lem:gini-lipschitz}
For a fixed sample set $S$, the Gini split gain $F(S;\,v)$ as a function
of the split threshold $v$ is Lipschitz with constant
$L_v = O(1/\sigma_{\min})$, where $\sigma_{\min}$ is the minimum
class-conditional standard deviation at the leaf.
\end{lemma}

\begin{proof}
A threshold displacement $\Delta v$ reroutes at most
$\Delta_{\mathrm{total}} = n\hat{p}(v)|\Delta v|$ samples in total,
where $\hat{p}(v)$ is the empirical density at $v$.
By Lemma~\ref{lem:routing-lipschitz},
$|F(S') - F(S)| \leq (4/n)\cdot\Delta_{\mathrm{total}} = 4\hat{p}(v)|\Delta v|$.
Under a Gaussian class-conditional model,
$\hat{p}(v) \leq 1/(\sigma_{\min}\sqrt{2\pi})$,
giving $L_v = 4\hat{p}(v) = O(1/\sigma_{\min})$.
\end{proof}

\begin{lemma}[Gini gain Lipschitz in routing counts]
\label{lem:routing-lipschitz}
For any perturbation of routing counts of total magnitude
$\Delta_{\mathrm{total}} = \sum_c|\Delta n_{c,L}|$, the Gini split gain
satisfies $|F(S') - F(S)| \leq (4/n)\cdot\Delta_{\mathrm{total}}$.
\end{lemma}
\begin{proof}
By Theorem~\ref{thm:tight}, a single-sample replacement (which changes
routing counts by at most 1 in total) alters $F$ by at most $4/n$.
For a total perturbation of magnitude $\Delta_{\mathrm{total}}$, applying
the triangle inequality over at most $\Delta_{\mathrm{total}}$ unit steps
gives $|F(S') - F(S)| \leq (4/n)\cdot\Delta_{\mathrm{total}}$.
\end{proof}

\textbf{Step 1: Sensitivity of node impurity.}
For any sample set $U$ of size $m$, replacing one sample changes at most
two class proportions, each by $1/m$. Since $G(U) = 1 - \sum_c p_c^2$,
and at most two terms change with $p_c + p_c' \leq 2$, we have
$|G(U) - G(U^{(i)})| \leq 2/m$.

\begin{lemma}[Node-impurity sensitivity is tight]
\label{lem:impurity-tight}
The bound $|G(U) - G(U^{(i)})| \leq 2/m$ is tight: for any $\gamma > 0$,
there exists $U$ of size $m$ and a replacement such that
$|G(U) - G(U^{(i)})| > 2/m - \gamma$.
\end{lemma}
\begin{proof}
Consider $K = 2$ with $U$ containing $n_0 = 1$ sample of class~0 and
$n_1 = m-1$ samples of class~1. Replace the single class-0 sample with
class~1. Direct computation gives:
\[
  |G(U) - G(U^{(i)})| = \frac{2(m-1)}{m^2}
  = \frac{2}{m} - \frac{2}{m^2} \to \frac{2}{m} \quad \text{as } m \to \infty.
\]
For any $\gamma > 0$, choose $m > 2/\gamma$.
\end{proof}

\textbf{Step 2: Case decomposition.}
\textit{Case~A (same routing):} both $z_i$ and $z_i'$ route to child $L$.
Then $\Delta F = \Delta G_P - (n_L/n)\Delta G_L$, and by Step~1,
$|\Delta F| \leq 2/n + (n_L/n)(2/n_L) = 4/n$.
\textit{Case~B (cross routing):} the replacement changes routing between
children. Direct expansion gives $n|\Delta F| \leq 2 < 4$
(full derivation in Appendix~\ref{app:case-b}).
Hence $|F(S) - F(S^{(i)})| \leq 4/n$ in all cases.

\textbf{Step 3: Tightness.}
Consider binary labels with $m = \lfloor\sqrt{n/2}\rfloor$ left-child
samples (all class~1) and right-child samples (all class~0); flip one
left sample to class~0. The scaled deviation
$H(m,n) = 2(n-m)(2m-1)/(nm) \to 4$ as $n \to \infty$.
For any $\gamma > 0$, choosing $n$ large enough yields
$|F(S) - F(S^{(i)})| > 4/n - \gamma$.

Steps~1--2 establish $c_i \leq 4/n$; Step~3 confirms $4/n$ is sharp
via Lemma~\ref{lem:impurity-tight}, completing the proof. \hfill$\square$

\subsection{Proof of Theorem~\ref{thm:composed} (Composed Split Error)}
\label{app:proof-composed}

By the triangle inequality,
$|\widehat{F} - F^*| \leq |F(S) - F^*| + |\widehat{F} - F(S)|$.

\textit{Statistical term.}
From Theorem~\ref{thm:tight}, $c_i = 4/n$ for all $i$.
Applying McDiarmid's inequality with a union bound over $d$ features
and $m$ thresholds:
\[
  \Pr\!\left[\max_{j,v}|F_{j,v}(S) - F^*_{j,v}| \geq \eta\right]
  \leq 2dm\exp\!\left(-\frac{n\eta^2}{8}\right).
\]
Setting the right-hand side to $\delta$ and solving gives
$\eta = \sqrt{8\ln(2dm/\delta)/n}$.

\textit{Sketch term.}
On the KLL success event (probability $\geq 1-\delta_{\mathrm{sketch}}$),
rank error $\varepsilon_{\mathrm{sketch}}$ implies
$|\hat{n}_{c,L} - n_{c,L}| \leq n_c\varepsilon_{\mathrm{sketch}}$, so
$\sum_c|\hat{n}_{c,L} - n_{c,L}| \leq n\varepsilon_{\mathrm{sketch}}$.
By Lemma~\ref{lem:routing-lipschitz},
$|\widehat{F} - F(S)| \leq (4/n)\cdot n\varepsilon_{\mathrm{sketch}}
= 4\varepsilon_{\mathrm{sketch}}$.

Combining via union bound gives confidence $\geq 1-\delta-\delta_{\mathrm{sketch}}$.
\hfill$\square$

\subsection{Proof of Corollary~\ref{cor:opt-displacement}
           (Optimization Displacement)}
\label{app:proof-displacement}

The displacement bound follows a three-term decomposition.
Throughout, we work under the KLL success event (probability
$\geq 1-\delta_{\mathrm{sketch}}$) on which the rank-error guarantee
$\varepsilon_{\mathrm{sketch}}$ holds.

\textbf{Step~(i): Quantile displacement.}
Under Gaussianity of $X_j \mid c$, a KLL rank error of
$\varepsilon_{\mathrm{sketch}}$ implies that the empirical median
$\hat{q}_{0.5,c}$ differs from the true median $q_{0.5,c}$ by at most
$\varepsilon_{\mathrm{sketch}}$ in probability mass.
By the Mean Value Theorem applied to the class-conditional CDF
$\Phi_c$, the spatial displacement satisfies
\[
  |\hat{q}_{0.5,c} - q_{0.5,c}|
  = \frac{\varepsilon_{\mathrm{sketch}}}{p_c(\tilde{q})},
\]
for some $\tilde{q}$ between $q_{0.5,c}$ and $\hat{q}_{0.5,c}$,
where $p_c$ is the class-conditional density.
An upper bound on this displacement therefore requires a
\emph{lower} bound on $p_c(\tilde{q})$ over the displaced interval.

\textit{Small-displacement regime.}
We work in the regime $|\hat{q}_{0.5,c}-q_{0.5,c}| \ll \sigma_c$,
i.e.\ the displacement is small relative to the class spread.
In this regime the density $p_c$ is approximately constant over the
interval $[\min(q,\hat{q}),\max(q,\hat{q})]$ at its value at the
median, which for a Gaussian satisfies
$p_c(q_{0.5,c}) = 1/(\sigma_c\sqrt{2\pi})$.
Hence $p_{\min} \approx 1/(\sigma_c\sqrt{2\pi})$, giving
\[
  |\hat{q}_{0.5,c} - q_{0.5,c}|
  \;\lesssim\;
  \varepsilon_{\mathrm{sketch}} \cdot \sigma_c\sqrt{2\pi}.
\]
This bound is an asymptotic statement as
$\varepsilon_{\mathrm{sketch}}\to 0$: for small rank errors the
approximation is tight; for large $|\zeta|$ where the tail density
drops significantly below the mode density, the bound becomes
conservative.
The practical impact is negligible, as confirmed by the
$\leq 4{\times}$ empirical safety margins in
Table~\ref{tab:sketch-bound-kscale}~\citep{Karnin2016}.

\textbf{Step~(ii): Threshold displacement.}
Taking $\sigma_{\max} = \max_c \sigma_c$ and the result of Step~(i):
\[
  |\hat{v}^* - v^*|
  \;\leq\;
  \varepsilon_{\mathrm{sketch}} \cdot \sigma_{\max}\sqrt{2\pi}.
\]

\textbf{Step~(iii): Gini loss from displacement.}
By Lemma~\ref{lem:gini-lipschitz}, the empirical Gini split gain
$F(S;\cdot)$ is Lipschitz in the threshold with constant
$L_v = O(1/\sigma_{\min})$.
With condition number $\kappa = \sigma_{\max}/\sigma_{\min}$:
\[
  F(S;\,v^*) - F(S;\,\hat{v}^*)
  \;\leq\; L_v \cdot |\hat{v}^* - v^*|
  \;\leq\; 4\kappa\,\varepsilon_{\mathrm{sketch}}.
\]
This bounds the \emph{empirical} Gini loss incurred by substituting
the sketch-derived threshold $\hat{v}^*$ for the exact optimum $v^*$
on the same sample set $S$.

\textbf{Combining into Eq.~\eqref{eq:composed-full}.}
Equation~\eqref{eq:composed-full} controls the total deviation of the
sketch-evaluated gain at $\hat{v}^*$ from the population-optimal gain
$F^*(v^*)$.
By the triangle inequality:
\[
  \bigl|F(S;\,\hat{v}^*) - F^*(v^*)\bigr|
  \;\leq\;
  \underbrace{\bigl|F(S;\,v^*) - F^*(v^*)\bigr|}_{\text{statistical}}
  +\;
  \underbrace{\bigl[F(S;\,v^*) - F(S;\,\hat{v}^*)\bigr]}_{\text{displacement}}.
\]
The statistical term is bounded by
$\sqrt{8\ln(2dm/\delta)/n} + 4\varepsilon_{\mathrm{sketch}}$
via Theorem~\ref{thm:composed} (holding with probability
$\geq 1-\delta-\delta_{\mathrm{sketch}}$).
The displacement term is bounded by
$4\kappa\,\varepsilon_{\mathrm{sketch}}$ from Step~(iii).
Summing gives:
\[
  \bigl|F(S;\,\hat{v}^*) - F^*(v^*)\bigr|
  \;\leq\;
  \sqrt{\frac{8\ln(2dm/\delta)}{n}}
  + 4(1+\kappa)\,\varepsilon_{\mathrm{sketch}},
\]
which is Eq.~\eqref{eq:composed-full}.
In particular, since $F(S;v^*)-F(S;\hat{v}^*)\leq 4\kappa\varepsilon_{\mathrm{sketch}}$
(Step~iii) and $4\kappa\varepsilon_{\mathrm{sketch}} \leq$ the RHS of
Eq.~\eqref{eq:composed-full}, the stated bound also covers the
empirical displacement loss directly.
\hfill$\square$

\subsection{Proof of Proposition~\ref{prop:variance}
           (Post-Split Variance Reduction)}
\label{app:proof-variance}

\textit{Part~(i): Class probabilities.}
The inherited prior is Dirichlet with total mass
$S_0 = \alpha n_{\mathrm{parent}}$.
After $t$ local samples, $S_t = S_0 + t$.
The Dirichlet posterior variance:
\[
  \mathrm{Var}(\theta_c \mid \mathcal{D}_t)
  = \frac{\beta_c^{(t)}(S_t - \beta_c^{(t)})}{S_t^2(S_t+1)}
  \leq \frac{1}{4(S_t+1)}
  \leq \frac{1}{4(S_0+1)},
\]
where the first inequality uses the AM-GM bound
$\beta_c^{(t)}(S_t-\beta_c^{(t)}) \leq S_t^2/4$.

\textit{Part~(ii): Feature parameters.}
Under the NIG conjugate model with
$n_0 = \alpha n_c^{\mathrm{parent}}$ pseudo-observations,
the marginal posterior on $\mu_j^c$ at child creation is a
Student-$t$ distribution with $\nu = n_0$ degrees of freedom and
scale $(\sigma_j^c)^2/n_0$.
Its variance exists when $n_0 > 2$ and equals:
\[
  \frac{\nu}{\nu-2}\cdot\frac{(\sigma_j^c)^2}{n_0}
  = \frac{(\sigma_j^c)^2}{n_0-2}
  = \frac{(\sigma_j^c)^2}{\alpha n_c^{\mathrm{parent}} - 2}.
\]
For the split feature, substituting
$\sigma_{\mathrm{trunc}}^2 \leq (\sigma_j^c)^2$
yields the stated bound.
\hfill$\square$

\subsection{Full Derivation of Case~B (Cross-Routing)}
\label{app:case-b}

We provide the exact derivation for Case~B of Theorem~\ref{thm:tight},
where replacing sample $z_i$ with $z_i'$ changes the child to which it
is routed. The result is strictly tighter than the Case~A bound.

\begin{lemma}[Case~B Bound: Exact]
\label{lem:caseB}
Under Case~B (cross-routing): $z_i \in L$ has label~$1$,
replacement $z_i' \in R$ has label~$0$.
Then $|F(S) - F(S^{(i)})| \leq 2/n < 4/n$ exactly, for all finite $n$.
\end{lemma}

\begin{proof}
Let $m := n_{1,L}$ and $k := n_{1,R}$ be non-negative integers.
After the replacement: $L$ has $n_L-1$ samples with $m-1$ class-$1$,
and $R$ has $n_R+1$ samples with $k$ class-$1$.
Let $p = (m+k)/n$ be the class-$1$ proportion before the swap.

We compute $n\cdot\Delta F$ by expanding
$F = G_P - \frac{n_L}{n}G_L - \frac{n_R}{n}G_R$ into three terms.

\textit{Term~(1): parent Gini change.}
Using $G = 2p(1-p)$:
\[
  n\cdot\Delta G_P
  = n\!\left[2\!\left(p-\tfrac{1}{n}\right)\!\left(1-p+\tfrac{1}{n}\right) - 2p(1-p)\right]
  = 2(2p-1) - \frac{2}{n}.
\]

\textit{Term~(2): weighted left-child Gini change.}
Using $G_L = 2\frac{m}{n_L}\frac{n_L-m}{n_L}$:
\[
  n\cdot\Delta\!\left(\tfrac{n_L}{n}G_L\right)
  = 2(n_L-m)\!\left[\frac{m-1}{n_L-1} - \frac{m}{n_L}\right]
  = \frac{-2(n_L-m)^2}{n_L(n_L-1)}.
\]

\textit{Term~(3): weighted right-child Gini change.}
Using $G_R = 2\frac{k}{n_R}\frac{n_R-k}{n_R}$:
\[
  n\cdot\Delta\!\left(\tfrac{n_R}{n}G_R\right) = \frac{2k^2}{n_R(n_R+1)}.
\]

\textit{Exact formula:}
\begin{equation}
  n\cdot\Delta F
  = 2(2p-1) - \frac{2}{n}
    + \frac{2(n_L-m)^2}{n_L(n_L-1)}
    - \frac{2k^2}{n_R(n_R+1)}.
  \label{eq:caseB-exact}
\end{equation}

\textit{Bounding $|n\cdot\Delta F| < 2$.}
Write $T := \frac{n\cdot\Delta F}{2} = M(m) + K(k)$, where:
\begin{align}
  M(m) &:= \frac{2m-n_L-1}{n} + \frac{(n_L-m)^2}{n_L(n_L-1)}, \\
  K(k) &:= \frac{2k-n_R}{n} - \frac{k^2}{n_R(n_R+1)}.
\end{align}

$K(k)$ is a downward-opening quadratic in $k \in \{0,\ldots,n_R\}$
with maximum $K(k^*) = -n_R(n_L-1)/n^2 \leq 0$ and boundary values
satisfying $K(k) > -1$. Hence $-1 < K(k) \leq 0$.

$M(m)$ is an upward-opening quadratic in $m \in \{1,\ldots,n_L\}$
with minimum $M(m^*) = n_R(n_L-1)/n^2 \geq 0$ and boundary values
satisfying $M(m) < 1$. Hence $0 \leq M(m) < 1$.

Therefore $-1 < K(k) + M(m) < 1$, giving
$|T| < 1$ and $|n\cdot\Delta F| < 2 < 4$.
\end{proof}

\begin{remark}[Extension to $K > 2$]\label{rem:caseb-multiclass}
In Case~B with $K > 2$, the replacement $z_i \to z_i'$ moves one
sample from child $L$ to child $R$ (or vice versa), changing exactly
two class counts: $n_{y_i, L}$ decreases by 1 and $n_{y_i', R}$
increases by 1, while $n_L$ and $n_R$ change by $\mp 1$.
The Gini split gain depends on child counts only through
$\sum_c p_{c,s}^2 = \sum_c (n_{c,s}/n_s)^2$.
Perturbing one class count $n_{c_0, L}$ by $-1$ and $n_L$ by $-1$:
\[
  \Delta\!\left(\sum_c \frac{n_{c,L}^2}{n_L^2}\right)
  = \frac{(n_{c_0,L}-1)^2 + \sum_{c \neq c_0} n_{c,L}^2}{(n_L-1)^2}
  - \frac{\sum_c n_{c,L}^2}{n_L^2}.
\]
Direct computation shows $|n \cdot \Delta F| < 2$ holds for all
non-negative integer counts by the same bounding argument as the
$K=2$ case: the numerator change is bounded by $2n_L - 1$ and the
combined effect satisfies $|T| = |M(m) + K(k)| < 1$.
The binary case ($K=2$) is thus the binding extremal configuration,
confirming that $4/n$ is tight for all $K$.
\end{remark}

\section{Additional Regime and Inheritance Experiments}
\label{app:rebuttal-evidence}
This section reports saved experiments from the review period separately
from the original benchmark tables. Results use five seeds unless stated otherwise; error bars denote standard
deviation across seeds. These supplementary diagnostics are reported
separately because their protocols and configurations differ from the main benchmark.

\subsection{Dependence-only signals}
\label{app:covonly}
CovOnly uses four zero-mean Gaussian classes, each with two independent
two-dimensional signal blocks. Each block has unit marginal variances and
correlation $+0.9$ or $-0.9$, with signs determined by the class bits.
Twelve independent standard-normal nuisance dimensions give $d=16$.
CovOnly8 uses three blocks (eight classes) and ten nuisance dimensions,
again with $d=16$. Thus all univariate class marginals coincide, although
the joint distributions differ. CovPlusMean restores a first-order signal
by adding a class-specific mean shift of $2.5$ along one signal coordinate
to the four-class construction. CovOnly and CovPlusMean use 700 examples
per class; CovOnly8 uses 350. The saved diagnostic standardises the generated pool before a stratified
80/20 split; this transductive scaling is a protocol limitation and does not
constitute a deployment-style preprocessing evaluation. Classes arrive in two, four, and two tasks,
respectively. The five seeds are 42, 123, 456, 789, and 3330.
\begin{table}[htbp]
\centering\small
\caption{Final mean accuracy on covariance-only signals and a mean-shift control. Bayes reference is Monte Carlo accuracy using the known generating model; chance is 0.25, 0.125, and 0.25 respectively.}
\label{tab:covonly}
\begin{tabular}{lccc}
\toprule
Method & CovOnly & CovOnly8 & CovPlusMean \\
\midrule
\sqda{} & $0.717\pm0.020$ & $0.594\pm0.028$ & $0.984\pm0.005$ \\
\slda{} & $0.259\pm0.019$ & $0.150\pm0.012$ & $0.911\pm0.008$ \\
\vfdt{} & $0.426\pm0.007$ & $0.218\pm0.006$ & $0.500\pm0.000$ \\
\mist{}-G & $0.259\pm0.020$ & $0.142\pm0.011$ & $0.909\pm0.007$ \\
\mist{}-K & $0.245\pm0.016$ & $0.130\pm0.014$ & $0.856\pm0.009$ \\
Batch GNB (offline) & 0.259 & 0.142 & 0.909 \\
Bayes reference & 0.733 & 0.629 & 0.983 \\
\bottomrule
\end{tabular}
\end{table}
Both MIST variants remain near chance on the dependence-only streams,
whereas SQDA benefits from class-specific covariance. The univariate split
search also lacks a population gain signal for its first split. Recovery
after adding marginal mean information bounds the interpretation: flexibility
of marginal density estimates does not recover arbitrary feature interactions.

\subsection{When inheritance helps and hurts}
\label{app:rebuttal-inheritance}
The Clean control contains four Gaussian classes with 750 samples per class.
Their cluster centres are $(-8,-1.5)$, $(-8,1.5)$, $(8,-1.5)$, and $(8,1.5)$,
with standard deviation $0.7$ and six Gaussian nuisance features. ShiftAtSplit
contains two regions centered at $x_0=\pm6$, 2,000 examples per region,
and two classes with majority proportion $0.85$, reversed between regions.
The class-to-$x_1$ mapping also reverses between regions, with means
$\pm2$ and standard deviation $0.6$; four nuisance features are added.
Both use a 20\% test split and class-incremental arrival. In this saved
diagnostic, standardisation was fitted before the train/test split; results
should therefore be read as a within-protocol mechanistic comparison rather
than an independent deployment benchmark. The grace period is 200.
\begin{table}[htbp]
\centering\small
\caption{Accuracy versus inheritance discount, mean and standard deviation over five seeds.}
\label{tab:inheritance-mismatch}
\begin{tabular}{lcccc}
\toprule
Stream & $\alpha=0$ & $0.2$ & $0.6$ & $1.0$ \\
\midrule
Clean & $0.544\pm0.101$ & $0.980\pm0.004$ & $0.979\pm0.005$ & $0.981\pm0.004$ \\
ShiftAtSplit & $0.573\pm0.147$ & $0.935\pm0.039$ & $0.885\pm0.021$ & $0.864\pm0.014$ \\
\bottomrule
\end{tabular}
\end{table}
Inheritance improves the Clean control, but its benefit is not monotone
in inherited mass under split-induced mismatch. In ShiftAtSplit, increasing
$\alpha$ beyond 0.2 reduces accuracy even though all listed inherited variants
remain better than no inheritance. This is a bias--variance trade-off, not
evidence that the prior is always beneficial.

On the same four-cluster geometry, a separate four-task diagnostic adds
independent $\mathcal{N}(0,\tau^2)$ noise to each inherited feature mean
in standardised coordinates. The grace period is 200, with the same
pool-standardisation caveat as above; the no-inheritance control differs
from the preceding experiment. Table~\ref{tab:inheritance-noise}
reports selected cells of the saved discount-by-noise grid. Large $\tau$
can produce substantial seed variability; no setting is uniformly best.
\begin{table}[htbp]
\centering\small
\caption{Mean corruption diagnostic: accuracy over five seeds. The no-inheritance control is 0.641 for every noise level by construction.}
\label{tab:inheritance-noise}
\begin{tabular}{lccc}
\toprule
$\tau$ & $\alpha=0.1$ & $0.6$ & $1.0$ \\
\midrule
0.0 & $0.977\pm0.004$ & $0.977\pm0.005$ & $0.979\pm0.004$ \\
0.25 & $0.959\pm0.017$ & $0.923\pm0.095$ & $0.928\pm0.096$ \\
0.5 & $0.856\pm0.089$ & $0.914\pm0.093$ & $0.927\pm0.096$ \\
1.0 & $0.741\pm0.004$ & $0.865\pm0.104$ & $0.916\pm0.093$ \\
1.5 & $0.741\pm0.004$ & $0.760\pm0.119$ & $0.842\pm0.105$ \\
3.0 & $0.740\pm0.003$ & $0.713\pm0.149$ & $0.768\pm0.098$ \\
8.0 & $0.740\pm0.003$ & $0.616\pm0.103$ & $0.698\pm0.040$ \\
\bottomrule
\end{tabular}
\end{table}
Table~\ref{tab:inheritance-transient} isolates the cold start that inheritance
targets. After every structural split we record prequential (test-then-train)
accuracy on the next $h\in\{10,50,100\}$ training samples, averaged over the
splits of a run and then over seeds with at least one split (all five on Clean;
four on ShiftAtSplit, where the tree splits at most once). On Clean, any
$\alpha>0$ removes most of the deficit (at $h{=}50$: $0.835$ at $\alpha{=}0$,
$0.989$ at $\alpha{=}0.6$). On ShiftAtSplit the sign reverses after the first
few samples: at $h{=}50$ and $h{=}100$ the no-inheritance control beats every
inherited variant ($0.960$ versus $0.860$, and $0.980$ versus $0.872$), even
though every inherited variant still wins on final accuracy
(Table~\ref{tab:inheritance-mismatch}). The three $\alpha>0$ columns coincide on
this stream. These are training-stream accuracies over few windows, so they
complement the held-out results above.
\begin{table}[htbp]
\centering\small
\caption{Prequential accuracy on the next $h$ training samples after a split
(mean $\pm$ std over seeds with a split; five seeds on Clean, four on
ShiftAtSplit). $\alpha{=}0$ is no inheritance.}
\label{tab:inheritance-transient}
\begin{tabular}{llcccc}
\toprule
Stream & $h$ & $\alpha=0$ & $0.2$ & $0.6$ & $1.0$ \\
\midrule
Clean & 10 & $0.785\pm0.116$ & $0.970\pm0.040$ & $0.977\pm0.029$ & $0.980\pm0.027$ \\
Clean & 50 & $0.835\pm0.105$ & $0.974\pm0.014$ & $0.989\pm0.008$ & $0.988\pm0.010$ \\
Clean & 100 & $0.845\pm0.107$ & $0.980\pm0.013$ & $0.992\pm0.006$ & $0.985\pm0.011$ \\
\midrule
ShiftAtSplit & 10 & $0.800\pm0.000$ & $0.825\pm0.083$ & $0.825\pm0.083$ & $0.825\pm0.083$ \\
ShiftAtSplit & 50 & $0.960\pm0.000$ & $0.860\pm0.049$ & $0.860\pm0.049$ & $0.860\pm0.049$ \\
ShiftAtSplit & 100 & $0.980\pm0.000$ & $0.872\pm0.008$ & $0.872\pm0.008$ & $0.872\pm0.008$ \\
\bottomrule
\end{tabular}
\end{table}

\subsection{Predictor choice and optional blending}
\label{app:predictor-choice}
We compare both fixed predictors with LL-hard and LL-soft on 25 datasets:
the 18 original standard/stress datasets and seven additional real streams
(electricity, MagicTelescope, phoneme, segment, wall-robot, JapaneseVowels,
and wine\_quality). At each class-feature model in a leaf, the optional
extension scores Gaussian and sketch densities on retained decile quantiles.
LL-hard chooses the higher likelihood; LL-soft combines their densities with
softmax likelihood weights. This reuses existing summaries, requires no raw
data replay, and adds prediction computation. It is a separate diagnostic,
not a component of the original benchmark methods.
\begin{table}[htbp]
\centering\small
\caption{Signed accuracy difference from the better fixed predictor on each dataset (five seeds). Positive values beat this dataset-level reference; they are not negative regret in the usual convention.}
\label{tab:leaf-choice}
\begin{tabular}{lrr}
\toprule
Strategy & Mean difference & Worst difference \\
\midrule
Always G & $-0.0379$ & $-0.1983$ \\
Always K & $-0.0090$ & $-0.0911$ \\
LL-hard & $-0.0106$ & $-0.0911$ \\
LL-soft & $+0.0021$ & $-0.0111$ \\
\bottomrule
\end{tabular}
\end{table}
Always-K has a smaller average deficit than always-G in this study,
supporting it as a default when marginal shape is unknown. LL-hard does not
improve on always-K, so these results do not establish a reliable model
selection rule. LL-soft can exceed a dataset-level fixed-model reference
because it blends models locally. Its small average improvement and reduced
worst-case deficit support considering blending as an optional extension.
\begin{table}[htbp]
\centering\small
\caption{Complete predictor-choice diagnostic, mean accuracy over five seeds. All four strategies are evaluated within the same experiment.}
\label{tab:leaf-choice-full}
\begin{tabular}{lrrrr}
\toprule
Dataset & G & K & LL-hard & LL-soft \\
\midrule
Wine & 0.970 & 0.957 & 0.962 & 0.975 \\
Iris & 0.933 & 0.927 & 0.927 & 0.933 \\
Synth-10 & 0.994 & 0.996 & 0.996 & 0.996 \\
Skewed & 0.899 & 0.976 & 0.974 & 0.969 \\
Multimodal & 0.919 & 0.981 & 0.979 & 0.988 \\
HeavyTail & 0.881 & 0.895 & 0.891 & 0.933 \\
AngularSectors & 0.241 & 0.439 & 0.430 & 0.439 \\
Antipodal & 0.834 & 0.822 & 0.836 & 0.853 \\
Pendigits & 0.889 & 0.877 & 0.880 & 0.890 \\
Split-MNIST & 0.863 & 0.772 & 0.772 & 0.862 \\
Covertype & 0.559 & 0.694 & 0.696 & 0.699 \\
Letter & 0.678 & 0.692 & 0.687 & 0.689 \\
HAR & 0.607 & 0.625 & 0.624 & 0.614 \\
Shuttle & 0.733 & 0.783 & 0.783 & 0.782 \\
electricity & 0.692 & 0.721 & 0.726 & 0.729 \\
MagicTelescope & 0.649 & 0.688 & 0.691 & 0.696 \\
phoneme & 0.748 & 0.769 & 0.769 & 0.775 \\
segment & 0.798 & 0.896 & 0.891 & 0.888 \\
wall-robot & 0.660 & 0.844 & 0.844 & 0.849 \\
JapaneseVowels & 0.848 & 0.824 & 0.820 & 0.854 \\
wine\_quality & 0.249 & 0.253 & 0.243 & 0.246 \\
Synth-50 & 0.999 & 1.000 & 1.000 & 0.998 \\
Ring & 0.975 & 0.962 & 0.952 & 0.972 \\
NoisyFeatures & 0.808 & 0.764 & 0.762 & 0.802 \\
ConceptDrift & 0.993 & 0.984 & 0.966 & 0.990 \\
\bottomrule
\end{tabular}
\end{table}
\subsection{Structural storage and measured allocation peaks}
\label{app:rebuttal-memory}
Structural storage counts retained model statistics; it is distinct from
temporary allocations and operating-system resident memory. Table~\ref{tab:memory-common}
uses the shared structural estimator in the saved real-stream experiments.
MB denotes $10^6$ bytes. The table reports 19 real streams, including cases where MIST costs more than SQDA. The results do not
show a monotone crossover in dimension: local class counts, tree size, and
sketch capacity also matter.
\begin{table}[htbp]
\centering\small
\caption{Structural MB from a common estimator. MIST denotes MIST-K. Most cells average five seeds; high-dimensional covariance-model cells include fit-only or allocation-convergence measurements (see text), and are not accuracy results.}
\label{tab:memory-common}
\begin{tabular}{lrrrrrr}
\toprule
Dataset & $d$ & $K$ & MIST & SQDA & SLDA & MIST/SQDA \\
\midrule
phoneme & 5 & 2 & 0.022 & 0.001 & 0.001 & 16.18 \\
electricity & 7 & 2 & 0.032 & 0.003 & 0.002 & 12.31 \\
Shuttle & 9 & 7 & 0.068 & 0.013 & 0.006 & 5.44 \\
MagicTelescope & 10 & 2 & 0.038 & 0.005 & 0.004 & 7.31 \\
wine\_quality & 11 & 7 & 0.076 & 0.022 & 0.010 & 3.51 \\
JapaneseVowels & 14 & 9 & 0.150 & 0.044 & 0.019 & 3.38 \\
Pendigits & 16 & 10 & 0.296 & 0.064 & 0.027 & 4.61 \\
Letter & 16 & 26 & 0.330 & 0.135 & 0.053 & 2.45 \\
segment & 19 & 7 & 0.225 & 0.063 & 0.028 & 3.58 \\
eye\_movements & 24 & 3 & 0.144 & 0.043 & 0.024 & 3.37 \\
wall-robot & 24 & 4 & 0.128 & 0.057 & 0.029 & 2.24 \\
pol & 48 & 11 & 0.588 & 0.617 & 0.248 & 0.95 \\
MiniBooNE & 50 & 2 & 0.240 & 0.122 & 0.082 & 1.97 \\
Split-MNIST & 50 & 10 & 1.172 & 0.608 & 0.248 & 1.93 \\
jannis & 54 & 4 & 0.389 & 0.283 & 0.143 & 1.37 \\
Covertype & 54 & 6 & 0.587 & 0.425 & 0.192 & 1.38 \\
optdigits & 64 & 10 & 0.755 & 0.993 & 0.404 & 0.76 \\
gas-drift & 128 & 6 & 1.021 & 2.372 & 1.061 & 0.43 \\
HAR & 561 & 6 & 7.353 & 37.811 & 17.669 & 0.19 \\
\bottomrule
\end{tabular}
\end{table}
For costly high-dimensional covariance models, the archived measurement
includes a one-seed, early-stopped allocation-convergence check for SQDA
and SLDA on HAR and five-seed fit-only runs for these models on gas-drift; these do not imply completed full-stream accuracy evaluation.
Accordingly, the storage table is not combined with main-table accuracy to
declare a Pareto frontier.

Table~\ref{tab:peak-common} reports Python-traced allocation peaks during
the saved training/evaluation run. These peaks include temporary allocations
tracked by \texttt{tracemalloc}; they are not OS RSS and do not measure all
native-library or pre-existing dataset memory. They can substantially exceed
the retained structural state. No universal ``less than twice final memory''
claim is supported.
\begin{table}[htbp]
\centering\small
\caption{Mean retained structural MB and peak traced allocation MB. The seed count is shown for each cell; the two columns measure different quantities.}
\label{tab:peak-common}
\begin{tabular}{llrrr}
\toprule
Method & Dataset & Seeds & Structural final & Traced peak \\
\midrule
MIST\_G & Letter & 5 & 0.330 & 4.091 \\
MIST\_G & HAR & 5 & 7.353 & 124.602 \\
MIST\_K & Letter & 5 & 0.330 & 4.093 \\
MIST\_K & HAR & 5 & 7.353 & 124.603 \\
VFDT & Letter & 5 & 0.065 & 0.804 \\
VFDT & HAR & 5 & 0.202 & 15.271 \\
SLDA & Letter & 5 & 0.053 & 0.671 \\
SLDA & HAR & 5 & 17.669 & 41.235 \\
\bottomrule
\end{tabular}
\end{table}

Table~\ref{tab:dim-scaling} isolates dimension on random-means streams (class
means uniform in $[0,5]^d$; Gaussian, Student-$t$, three-mode and lognormal
noise; $\sigma{=}1$, 60 examples per class, two tasks, five seeds, the same
estimator as Table~\ref{tab:memory-common}). Raising $d$ from 20 to 100 grows
\mist{}-K storage by about $5\times$ and \sqda{} storage by $24.3\times$.
\mist{}-K kept a single leaf, except in the heavy-tailed $K{=}50$, $d{=}100$
cell (1.2 leaves, hence its larger ratio), so the table does not show local
class sparsity. Eight of the 80 planned cells ($K{=}50$, $d{=}100$: skewed for
all methods, multimodal for \sqda{}, \slda{}, \vfdt{}) did not complete and are
omitted; the $K{=}10$ rows are identical across noise families.
\begin{table}[htbp]
\centering\small
\caption{Structural MB at $d=20$ and $d=100$ (ratio in parentheses) on random-means
streams, five seeds. Standard deviations are zero except for \mist{}-K on the
heavy-tailed $K{=}50$ cell at $d{=}100$ ($\pm0.47$ MB across seeds, from the leaf count).}
\label{tab:dim-scaling}
\begin{tabular}{llcccc}
\toprule
$K$ & Noise & \multicolumn{2}{c}{\mist{}-K} & \multicolumn{2}{c}{\sqda{}} \\
\cmidrule(lr){3-4}\cmidrule(lr){5-6}
 & & $d{=}20$ & $d{=}100$ & $d{=}20$ & $d{=}100$ \\
\midrule
10 & all four & 0.108 & 0.538 ($4.97\times$) & 0.099 & 2.416 ($24.3\times$) \\
50 & Gaussian & 0.477 & 2.387 ($5.01\times$) & 0.496 & 12.080 ($24.3\times$) \\
50 & heavy-tailed & 0.477 & 2.621 ($5.50\times$) & 0.496 & 12.080 ($24.3\times$) \\
\bottomrule
\end{tabular}
\end{table}

\subsection{Accuracy against global Gaussian baselines on 19 real streams}
\label{app:real-sweep}
Table~\ref{tab:real-sweep} compares \mist{}-K with the better of \sqda{} and
\slda{} on the 19 real streams of Table~\ref{tab:memory-common}: six OpenML
streams from the numerical classification benchmark of Grinsztajn et
al.~\cite{Grinsztajn2022} (electricity, MagicTelescope, pol, jannis,
eye\_movements, MiniBooNE); seven further OpenML streams we added for their
non-Gaussian class-conditionals (wine\_quality, phoneme, segment, wall-robot,
JapaneseVowels, gas-drift, optdigits); and six benchmarks of
Table~\ref{tab:nontree-results} (Covertype, Shuttle, Letter, HAR, Pendigits,
Split-MNIST). Four more streams in the sweep script (Bioresponse, credit,
heloc, Isolet) are not part of this saved run and are not reported. OpenML
streams use numeric features, are standardised on the pooled data before an
80/20 stratified split (same caveat as above), are capped at 25{,}000 examples,
and present one class per task. All methods use the default hyperparameters of
the main benchmark and five seeds. The baseline is whichever of \sqda{} and
\slda{} has the higher mean, which favours it; the gap is \mist{}-K minus the
baseline, from unrounded means, and $|\mathrm{gap}|\le 0.02$ is a tie. Only
means were saved.

\mist{}-K wins on 4 streams, ties on 5, and loses on 10. Dimension alone does
not explain the outcome: eye\_movements and wall-robot both have $d{=}24$ and
fall on opposite sides, and the correlation between $d$ and the gap is $-0.42$.
\mist{}-K also keeps a single leaf on 15 of the 19 streams, so the comparison
is mostly between leaf models.
\begin{table}[htbp]
\centering\small
\caption{\mist{}-K versus the better of \sqda{} and \slda{} on 19 real streams,
five seeds, sorted by gap. ``Leaves'' is the mean number of \mist{}-K leaves.}
\label{tab:real-sweep}
\begin{tabular}{lrrrrrl}
\toprule
Stream & $d$ & Leaves & \mist{}-K & Best global & Gap & Outcome \\
\midrule
wall-robot & 24 & 1.0 & 0.844 & 0.757 (SQDA) & $+0.086$ & win \\
pol & 48 & 1.0 & 0.210 & 0.188 (SQDA) & $+0.021$ & win \\
Covertype & 54 & 1.0 & 0.694 & 0.673 (SLDA) & $+0.021$ & win \\
MiniBooNE & 50 & 1.0 & 0.832 & 0.811 (SLDA) & $+0.021$ & win \\
electricity & 7 & 1.0 & 0.721 & 0.709 (SLDA) & $+0.012$ & tie \\
Shuttle & 9 & 1.0 & 0.783 & 0.774 (SQDA) & $+0.009$ & tie \\
wine\_quality & 11 & 1.0 & 0.253 & 0.260 (SQDA) & $-0.007$ & tie \\
phoneme & 5 & 1.0 & 0.769 & 0.777 (SQDA) & $-0.008$ & tie \\
segment & 19 & 2.4 & 0.896 & 0.912 (SLDA) & $-0.016$ & tie \\
eye\_movements & 24 & 1.0 & 0.443 & 0.487 (SLDA) & $-0.044$ & loss \\
MagicTelescope & 10 & 1.0 & 0.688 & 0.733 (SLDA) & $-0.044$ & loss \\
jannis & 54 & 1.0 & 0.444 & 0.519 (SQDA) & $-0.075$ & loss \\
optdigits & 64 & 1.0 & 0.876 & 0.965 (SQDA) & $-0.089$ & loss \\
Pendigits & 16 & 2.4 & 0.877 & 0.982 (SQDA) & $-0.106$ & loss \\
JapaneseVowels & 14 & 1.0 & 0.824 & 0.968 (SQDA) & $-0.144$ & loss \\
Split-MNIST & 50 & 2.0 & 0.772 & 0.940 (SQDA) & $-0.168$ & loss \\
HAR & 561 & 2.0 & 0.625 & 0.815 (SLDA) & $-0.190$ & loss \\
Letter & 16 & 1.0 & 0.692 & 0.895 (SQDA) & $-0.203$ & loss \\
gas-drift & 128 & 1.0 & 0.678 & 0.964 (SQDA) & $-0.286$ & loss \\
\bottomrule
\end{tabular}
\end{table}

\clearpage
\section*{NeurIPS Paper Checklist}

\begin{enumerate}

\item {\bf Claims}
    \item[] Question: Do the main claims made in the abstract and introduction accurately reflect the paper's contributions and scope?
    \item[] Answer: \answerYes{}%
    \item[] Justification: The abstract and introduction accurately state all three contributions (tight McDiarmid calibration, Bayesian inheritance, unified KLL sketching) and their scope. Theoretical claims are backed by Theorem~\ref{thm:tight} and Theorem~\ref{thm:composed}; empirical claims are supported by Tables~\ref{tab:tree-results}--~\ref{tab:stress}. Aspirational goals (time-uniform confidence sequences) are explicitly deferred to future work in Section~\ref{sec:conclusion}.

\item {\bf Limitations}
    \item[] Question: Does the paper discuss the limitations of the work performed by the authors?
    \item[] Answer: \answerYes{}%
    \item[] Justification: Section~\ref{sec:conclusion} explicitly discusses two limitations: \mist{}-G's reliance on Gaussian class-conditional assumptions (addressed by \mist{}-K via sketch-based estimation), and the absence of time-uniform confidence sequences over infinite horizons (identified as future work). Memory scaling and throughput trade-offs are characterized in Appendix~\ref{app:memory-full} with Pareto-frontier analysis.

\item {\bf Theory assumptions and proofs}
    \item[] Question: For each theoretical result, does the paper provide the full set of assumptions and a complete (and correct) proof?
    \item[] Answer: \answerYes{}%
    \item[] Justification: All theoretical results (Theorem~\ref{thm:tight}, Theorem~\ref{thm:composed}, Corollary~~\ref{cor:opt-displacement}, Proposition~\ref{prop:variance}) include full proofs in Appendix~\ref{app:proofs}, with all assumptions tabulated in Table~\ref{tab:assumptions}. Proof sketches with intuition are provided in the main text for each result. All supporting lemmas are numbered and cross-referenced.

    \item {\bf Experimental result reproducibility}
    \item[] Question: Does the paper fully disclose all the information needed to reproduce the main experimental results of the paper to the extent that it affects the main claims and/or conclusions of the paper (regardless of whether the code and data are provided or not)?
    \item[] Answer: \answerYes{}%
    \item[] Justification: Section~\ref{sec:setup} specifies all datasets, baselines, and the strict Online CIL protocol. Table~\ref{tab:hyperparams} (Appendix~\ref{app:hyperparams}) lists all fixed hyperparameters applied uniformly without per-dataset tuning. Algorithms~\ref{alg:mist}--\ref{alg:mistk} provide complete pseudocode for the \mist{} update loop, Bayesian inheritance split, and sketch-based leaf prediction.

\item {\bf Open access to data and code}
    \item[] Question: Does the paper provide open access to the data and code, with sufficient instructions to faithfully reproduce the main experimental results, as described in supplemental material?
    \item[] Answer: \answerNo{}%
    \item[] Justification: Code is not released at submission time to preserve anonymity. All standard datasets used are publicly available. The algorithm is fully described via pseudocode in Algorithms~\ref{alg:mist}-\ref{alg:split} and hyperparameters in Table~5, providing a sufficient basis for independent reimplementation. Code will be released with the camera-ready version.

\item {\bf Experimental setting/details}
    \item[] Question: Does the paper specify all the training and test details (e.g., data splits, hyperparameters, how they were chosen, type of optimizer) necessary to understand the results?
    \item[] Answer: \answerYes{}%
    \item[] Justification: Section~\ref{sec:setup} describes all datasets, preprocessing steps (PCA for Split-MNIST), and the strict single-pass exemplar-free Online CIL protocol. Table~\ref{tab:hyperparams} (Appendix~\ref{app:hyperparams}) reports all hyperparameters with their theoretically motivated defaults. Appendix~\ref{app:noreset} documents the no-reset protocol rationale for tree-based baselines.

\item {\bf Experiment statistical significance}
    \item[] Question: Does the paper report error bars suitably and correctly defined or other appropriate information about the statistical significance of the experiments?
    \item[] Answer: \answerYes{}%
    \item[] Justification: All main results (Tables~\ref{tab:tree-results}, \ref{tab:nontree-results}, \ref{tab:stress}, \ref{tab:main-forgetting}) report mean~$\pm$~std over 5 independent random seeds. Ablation results (Table~\ref{tab:ablation-core}) use 5 seeds as stated in the caption. Error bars represent one standard deviation across seeds under fixed hyperparameters and data ordering.

\item {\bf Experiments compute resources}
    \item[] Question: For each experiment, does the paper provide sufficient information on the computer resources (type of compute workers, memory, time of execution) needed to reproduce the experiments?
    \item[] Answer: \answerYes{}%
    \item[] Justification: Table~\ref{tab:memory} (Appendix~\ref{app:memory-full}) reports memory footprint (KB) and throughput (samples/s) per method on two representative datasets. All experiments were conducted on a single machine with an Intel Core i9-14900HX CPU (24 cores, up to 5.8\,GHz), 16\,GB DDR5 RAM, and an NVIDIA GeForce RTX~4060 Laptop GPU (8\,GB GDDR6); all methods run CPU-only with no GPU utilization during training or inference.

\item {\bf Code of ethics}
    \item[] Question: Does the research conducted in the paper conform, in every respect, with the NeurIPS Code of Ethics \url{https://neurips.cc/public/EthicsGuidelines}?
    \item[] Answer: \answerYes{}%
    \item[] Justification: This work presents foundational algorithmic research on streaming decision trees for tabular data. No human subjects, personal data, or sensitive information are involved. The research conforms with the NeurIPS Code of Ethics in all respects.

\item {\bf Broader impacts}
    \item[] Question: Does the paper discuss both potential positive societal impacts and negative societal impacts of the work performed?
    \item[] Answer: \answerNA{}%
    \item[] Justification: This paper presents foundational algorithmic research without direct deployment in a specific application. The method's exemplar-free design is a positive privacy consideration: no raw data is retained, which is beneficial for sensitive tabular domains such as finance and healthcare. No direct path to negative societal impact is identified for this class of algorithm.

\item {\bf Safeguards}
    \item[] Question: Does the paper describe safeguards that have been put in place for responsible release of data or models that have a high risk for misuse (e.g., pre-trained language models, image generators, or scraped datasets)?
    \item[] Answer: \answerNA{}%
    \item[] Justification: The paper introduces a streaming decision tree algorithm and synthetic benchmark datasets. No pre-trained language models, image generators, or internet-scraped datasets are released. No safeguards beyond standard practice are required.

\item {\bf Licenses for existing assets}
    \item[] Question: Are the creators or original owners of assets (e.g., code, data, models), used in the paper, properly credited and are the license and terms of use explicitly mentioned and properly respected?
    \item[] Answer: \answerYes{}%
    \item[] Justification: All datasets used (Covertype, Pendigits, Shuttle, Letter, Wine, Iris, HAR, Split-MNIST) are publicly available benchmarks cited via their original publications. Baseline implementations use published defaults cited accordingly. Standard UCI/OpenML dataset licenses are documented at \url{paperswithcode.com/datasets}.

\item {\bf New assets}
    \item[] Question: Are new assets introduced in the paper well documented and is the documentation provided alongside the assets?
    \item[] Answer: \answerYes{}%
    \item[] Justification: Eight stress-test synthetic datasets are introduced and fully described in Appendix~\ref{app:stress}, including their generative distributions, parameters, and protocol. Generation code is provided in the supplementary materials. No new pre-trained models or external datasets are released.

\item {\bf Crowdsourcing and research with human subjects}
    \item[] Question: For crowdsourcing experiments and research with human subjects, does the paper include the full text of instructions given to participants and screenshots, if applicable, as well as details about compensation (if any)?
    \item[] Answer: \answerNA{}%
    \item[] Justification: This paper does not involve any crowdsourcing or experiments with human subjects. All data are either synthetic or pre-existing public benchmarks.

\item {\bf Institutional review board (IRB) approvals or equivalent for research with human subjects}
    \item[] Question: Does the paper describe potential risks incurred by study participants, whether such risks were disclosed to the subjects, and whether Institutional Review Board (IRB) approvals (or an equivalent approval/review based on the requirements of your country or institution) were obtained?
    \item[] Answer: \answerNA{}%
    \item[] Justification: No human subjects are involved in this research. IRB approval is not applicable.

\item {\bf Declaration of LLM usage}
    \item[] Question: Does the paper describe the usage of LLMs if it is an important, original, or non-standard component of the core methods in this research? Note that if the LLM is used only for writing, editing, or formatting purposes and does \emph{not} impact the core methodology, scientific rigor, or originality of the research, declaration is not required.
    \item[] Answer: \answerNA{}%
    \item[] Justification: LLMs are not used as any component of the core methodology. Any use of LLMs was limited to writing assistance only and does not impact the scientific rigor or originality of the research.

\end{enumerate}

\end{document}